\documentclass[twoside]{article}

\usepackage[accepted]{arxiv_style}
\usepackage{natbib}

\usepackage[utf8]{inputenc} 
\usepackage[T1]{fontenc}    
\usepackage{hyperref}       
\usepackage{url}            
\usepackage{booktabs}       
\usepackage{nicefrac}       
\usepackage{microtype}      
\usepackage{xcolor}         
\usepackage{amsmath}
\usepackage{graphicx}
\usepackage{bbold}
\usepackage{multirow}
\usepackage{booktabs}
\usepackage{amsfonts}
\usepackage{mathrsfs}
\usepackage{subcaption}
\usepackage{caption}
\usepackage{amsthm}
\usepackage[noend,algo2e,noline,ruled,linesnumbered]{algorithm2e}

\newtheorem{theorem}{Theorem}
\newtheorem{propo}[theorem]{Proposition}
\newcommand{\algorithmfootnote}[2][\footnotesize]{%
  \let\old@algocf@finish\@algocf@finish
  \def\@algocf@finish{\old@algocf@finish
    \leavevmode\rlap{\begin{minipage}{\linewidth}
    #1#2
    \end{minipage}}%
  }%
}
\usepackage{tabularx}
\usepackage[export]{adjustbox}
\usepackage{bbm}

\usepackage[compact]{titlesec}
\titlespacing{\section}{0pt}{.5ex}{.5ex}

\begin{document}

%

%

\twocolumn[

\aistatstitle{AUC-based Selective Classification}

\aistatsauthor{Andrea Pugnana
\And 
Salvatore Ruggieri 
}

\aistatsaddress{Scuola Normale Superiore
\And  
University of Pisa
} ]

\begin{abstract}
Selective classification (or classification with a reject option) pairs a classifier with a selection function to determine whether or not a  prediction should be accepted. This framework trades off coverage (probability of accepting a prediction) with predictive performance, typically measured by distributive loss functions.
In many application scenarios, such as credit scoring, performance is instead measured by ranking metrics, such as the Area Under the ROC Curve (AUC). We propose a model-agnostic approach to associate a selection function to a given probabilistic binary classifier. The approach is specifically targeted at optimizing the AUC. We provide both theoretical justifications and a novel algorithm, called AUCROSS, to achieve such a goal. 
Experiments show that our method succeeds in trading-off coverage for AUC, improving over existing selective classification methods targeted at optimizing accuracy.
\end{abstract}
\section{INTRODUCTION}

The predictive performance of a classifier is typically not homogeneous over the data distribution. Identifying sub-populations with low performance is helpful, e.g. for debugging and monitoring purposes. In many socially sensitive application scenarios, we can consider not predicting at all for these sub-populations as an alternative to poor or even harmful predictions. Such applications include credit scoring, curriculum screening, access to public benefits, and medical diagnoses.  
\textit{Selective classification} (or \textit{classification with a reject option}) \citep{DBLP:journals/tit/Chow70,DBLP:conf/icml/Pietraszek05} pairs a classifier with a selection function to determine whether a prediction should be accepted or the classifier should abstain. The selection function assesses the trustworthiness of a prediction.
The literature on selective classification has mainly considered distributive loss functions. However, in many real-world scenarios, such as credit scoring \citep{Engelmann2011}, the quality of a binary probabilistic classifier is concerned with the discriminative power of the ranking induced by its scores \citep{DBLP:journals/tkde/HuangL05}. This is measured by the Area Under the ROC Curve (AUC) \citep{DBLP:journals/prl/Fawcett06} or, equivalently, by the Gini coefficient \citep{Wu14}. Algorithms designed to minimize, e.g., the error rate may not lead to the best possible AUC values \citep{DBLP:conf/nips/CortesM03}. 

In this paper, we first introduce the problem of \textit{AUC-based selective classification}, which consists of inducing a selective classifier which optimizes the AUC over the sub-population of accepted predictions while guaranteeing a minimum probability mass (called \textit{coverage}) of such sub-population. 
We then focus on an instance of the AUC-selective classification problem, where we assume that: \textit{(i)} the classification algorithm is given, including its hyper-parameters, and \textit{(ii)} the selection function abstains on a range of the scores (\textit{score-based selective functions}). 
Assumption \textit{(i)} allows lifting an existing classifier to a selective classifier. In this respect, our approach is model-agnostic, as we do not make assumptions about the classification algorithm. Assumption \textit{(ii)} is widespread in selective classification, supported by theoretical results  \citep{DBLP:journals/jmlr/El-YanivW10}. Such an assumption allows for a theoretical analysis of the AUC-based selective classification problem. 
We devise \textsc{AUCross}, a model-agnostic algorithm for estimating the bounds of score-based selective functions that maximize the AUC for target coverage. The approach is based on a cross-fitting strategy over the training set. Bound estimates are supported by a theoretical analysis of (sufficient) conditions for abstaining that lead to an increase in the AUC. 
Experiments show that the approach performs very close to an oracle score-based selective function (which has access to the true class of instances), and it outperforms existing selective classification approaches targeted at optimizing accuracy.

The paper is organized as follows. Section \ref{sec:rel_work} briefly surveys related work. Section \ref{sec:sel_class}  recalls basic concepts of selective classification and it introduces the AUC-based selective classification problem. Our theoretical approach is presented in Section \ref{sec:method}, while we introduce \textsc{AUCross} algorithm in Section \ref{sec:alg}. Section~\ref{sec:exp} reports experimental results. Finally, we draw conclusions and outline possible extensions. 

\section{RELATED WORK}
\label{sec:rel_work}

The two main models of selective classification are the cost model \citep{DBLP:journals/tit/Chow70} and the the bounded-improvement model \citep{DBLP:conf/icml/Pietraszek05}. In the former, the goal is to minimize the expected cost, assuming a cost for misclassification and a cost for abstention, or a more refined cost based on the confusion matrix \citep{DBLP:journals/prl/Tortorella05}.
In the Bayes optimal selective classifier, the selection function abstains when the posterior probability of the predicted class is below some threshold. Posterior probabilities can be estimated on a validation set by the \textit{plug-in} rule \citep{Herbei06}. 
In the bounded-improvement model, the selection function is evaluated based on the probability mass of the accepted region (\textit{coverage}) and the expected loss over such a region (\textit{selective risk}). One can maximize coverage for a maximum target risk or minimize risk for a minimum target coverage \citep{DBLP:conf/nips/GeifmanE17}. \cite{DBLP:conf/icml/FrancP19} establish the equivalence of cost-based and\,bounded-improvement\,models.

Most approaches for selective classification are model-specific in that they build the classifier and the selection function concurrently. Examples include methods for Support Vector Machines (SVM) \citep{DBLP:conf/svm/FumeraR02}, boosted decision trees \citep{DBLP:conf/nips/CortesDM16}, Deep Neural Networks (DNNs) \citep{DBLP:conf/nips/GeifmanE17,DBLP:conf/icml/GeifmanE19,DBLP:conf/nips/LiuWLSMU19,DBLP:conf/nips/Huang0020}.
See \cite{Hendrickx2021} for a complete overview of existing methods.
Moreover, the definitions of cost or risk in selective classification have been provided using distributive loss functions, where the loss is defined for every prediction in isolation. AUC is a metric about the ranking induced by a classifier, for which the loss is determined for pairs of instances.
To the best of our knowledge, the only work directly addressing AUC 
selective classification is \cite{DBLP:conf/aaai/ShenYG20}. However, 
the selection function is used here to accept or to abstain from ranking \textit{pairs} of instances. The AUC to be optimized is defined as the mean correct order of accepted pairs (an extension of the Mann-Whitney U-statistics): a same instance may appear in both an accepted and a rejected pair. 
We target, instead, either accepting or rejecting predictions for each single instance. The AUC we optimize is the mean correct order of any pair of accepted cases (i.e., the Mann-Whitney U-statistics over the accepted region). 

In this paper, we adhere to the bounded-improvement model, with the AUC as the metric to optimize by abstaining on \textit{single instances}, and we take a model-agnostic view of the problem. Notice that  AUC-based selective classification is orthogonal to the many approaches for optimizing AUC in supervised learning \citep{DBLP:journals/corr/YangYing2022}.

\section{BACKGROUND}
\label{sec:sel_class}

Consider random variables $(\mathbf{X},Y) \in \mathcal{X}\times\mathcal{Y}$, where  $\mathcal{X} \subseteq \mathbb{R}^d$ is a feature space and $\mathcal{Y} = \{ 0, 1, \ldots, n_{\mathcal{Y}}\}$ a finite label space. The joint distribution of $(\mathbf{X},Y) \sim \mathcal{D}$ is unknown, but we can observe one or more datasets of i.i.d.~realizations. 
A \textit{classifier} is a function $h:\mathcal{X}\to\mathcal{Y}$ that maps features to classes, computed from an hypothesis space 
and a dataset (training set). The expected loss over the distribution $R(h)=\int_{\mathcal{X}\times\mathcal{Y}}l(h(\mathbf{x}),y)d\mathcal{D}(\mathbf{x},y) = \mathbb{E}_{\mathcal{D}}[l(h(\mathbf{X}),Y)]$ is called the \textit{risk}, where $l:\mathcal{Y}\times \mathcal{Y} \to \mathbb{R}$ is a loss function. The risk can be estimated starting from a test set $S_n = \{(\mathbf{x}_i,y_i)\}_{i=1}^{n}$ through the \textit{empirical risk} 
$\hat{r}(h|S_n) =
\frac{1}{n}\sum_{i=1}^{n}l(h(\mathbf{x}_i),y_i)$. 
Such a canonical setting is extended to model situations where predictions of classifiers are not sufficiently reliable, and the option of abstaining is preferable. A \textit{selective classifier} is a pair $(h,g)$, where $h$ is a classifier and $g:\mathcal{X}\to\{0,1\}$ is a \textit{selection function}, which determines when to accept/abstain from using $h$:
\begin{equation*}
(h,g)(\mathbf{x})=
        \begin{cases}
      h(\mathbf{x}) & \text{if}\ g(\mathbf{x})=1 \\
      \text{abstain} & \text{otherwise}
    \end{cases}
\end{equation*}
A soft selection approach  \citep{DBLP:conf/nips/GeifmanE17} consists of defining $g$ in terms of a \textit{confidence function} $k_h : \mathcal{X}\to [0,1]$ (the subscript highlights that $k_h$ depends on $h$) and a threshold $\theta$  of minimum confidence for accepting: 
\begin{equation}
\label{eq:selfunction_conf}
g(\mathbf{x}) = \mathbb{1}(k_h(\mathbf{x})>\theta)
\end{equation}
A good confidence function should order instances based on descending loss, i.e., if $k_{h}(\mathbf{x}_i) \leq k_{h}(\mathbf{x}_j)$ then $l(h(\mathbf{x}_i),y_i) \geq l(h(\mathbf{x}_j),y_j)$.
The \textit{coverage} of a selective classifier is $\phi(g)=E_{\mathcal{D}}[g(\mathbf{X})]$, i.e., the expected probability mass of the accepted region. The \textit{selective risk} is risk normalized by coverage:
\begin{equation*}
    R(h,g) = \frac{\mathbb{E}_{\mathcal{D}}[l(h(\mathbf{X}),Y)g(\mathbf{X})]}{\phi(g)} = \mathbb{E}_{\mathcal{D}}[l(h(\mathbf{X}),Y) | g(\mathbf{X})]
\end{equation*}
Empirical coverage and empirical selective risk are respectively defined as follows:\\[-2ex]
\[ \hat{\phi}(g|S_n) = \frac{\sum_{i=1}^n g(\mathbf{x}_i)} {n} \] 
\[ \hat{r}(h,g|S_n) = \frac{\frac{1}{n} \sum_{i=1}^n l(h(\mathbf{x}_i),y_i)g(\mathbf{x}_i)}{ \hat{\phi}(g|S_n)} \]
By defining $S^g_m = \{ (\mathbf{x}_i, y_i) \in S_n \ |\ g(\mathbf{x}_i )=1\}$, the empirical coverage is $\hat{\phi}(g|S_n) = |S^g_m|/n = m/n$ and the empirical selective risk reduces to $\hat{r}(h,g|S_n) = \hat{r}(h|S^g_m)$, namely to the empirical risk over the accepted instances.
The inherent trade-off between risk and coverage is summarized by the \textit{risk-coverage curve}  \citep{DBLP:journals/jmlr/El-YanivW10}.
The selective classification problem can be framed by fixing an upper bound to the selective risk and looking for a selective classifier that maximizes coverage. \cite{DBLP:conf/nips/GeifmanE17} 
show how to convert this framing into an alternative one, where a lower bound $c$ for coverage is fixed (\textit{target coverage}), and then we look for a selective classifier that minimizes selective risk. We adhere to such a formulation. Called $\theta$ the parameter(s) defining $h$ and $g$ (e.g., as in the soft selection approach), the \textit{selective classification problem} is stated as:
\begin{equation*} \label{scp}
 \min_{\theta} R(h,g) \hspace{2ex} 
 \text{s.t.} \quad \phi(g)\geq c
\end{equation*}
Let us now extend the framework to the AUC metric.
We consider binary classes (0/1, negatives/positives) and probabilistic classifiers, where $h(\mathbf{x}) \in [0, 1]$ is an estimate of the probability that $\mathbf{x}$ is positive. Let $\mathcal{D}_1$ and $\mathcal{D}_0$ be the conditional distributions of positives and negatives, respectively.
The AUC can be defined as the probability that a randomly drawn positive receives a higher score than a randomly drawn negative, conditioned to the fact that both are selected according to $g$, i.e.
\begin{align*}
\label{eq:selauc}
    AUC(h,g) = \mathbb{E}_{\mathbf{X}_0 \sim \mathcal{D}_0, \mathbf{X}_1 \sim  \mathcal{D}_1}[\mathbb{1}(h(\mathbf{X}_1) > h(\mathbf{X}_0)) |\\ g(\mathbf{X}_0)=1, g(\mathbf{X}_1)=1]
\end{align*}
The \textit{AUC-selective classification problem} 
can be stated as:
\begin{equation} 
\label{eq:aucsc}
 \max_{\theta} AUC(h,g) \hspace{2ex} 
 \text{s.t.} \quad \phi(g)\geq c
\end{equation}
A good confidence function should order instances based on ascending contribution to the AUC while allowing for controlling the target coverage. Empirical AUC given $S_n$ can be defined by restricting to the set $S^g_m = \{ (\mathbf{x}_i, y_i) \in S_n \ |\ g(\mathbf{x}_i )=1\}$ of accepted instances directly from the definition above by resorting to the Mann-Whitney U-statistic: $\widehat{AUC}(h,g|S_n) = \widehat{AUC}(h|S^g_m) = \frac{1}{m^+}\frac{1}{m^-} \sum_{ (\mathbf{x}_i,1) \in S^g_m} \sum_{(\mathbf{x}_i,0) \in S^g_m} \mathbb{1}(h(\mathbf{x}_i) > h(\mathbf{x}_j))$, where $m^+$ is the number of positives in $S^g_m$, and $m^-$ is the number of negatives in $S^g_m$. An alternative calculation of $\widehat{AUC}$ is the Area Under the Receiver Operating Characteristic (ROC) Curve \citep{DBLP:journals/prl/Fawcett06}.

\section{AUC SELECTIVE CLASSIFICATION}\label{sec:method}

We consider the AUC-selective classification problem for binary probabilistic classifiers. We make the further assumption that the parameters of $h$ are fixed, i.e., $\theta$ in (\ref{eq:aucsc}) includes only the parameters of the selection function $g$.
Let us denote by $H$ a binary probabilistic classifier induction algorithm whose hyper-parameters are fixed. The algorithm induces the classifier $h$ starting from a training set. $h(\mathbf{x})$ is the score assigned by $h$ to an instance $\mathbf{x}$.
We aim to lift $h$ to a selective classifier $(h, g)$ by calculating a selection function $g$ from the following hypothesis space (\textit{score-bounded selection functions}):
\begin{equation}\label{eq:gk}
    g(\mathbf{x}) = \begin{cases}
    0 \quad \text{if } \theta_{l} \leq h(\mathbf{x}) \leq \theta_{u}\\ 
    1 \quad \text{otherwise}
    \end{cases}
\end{equation}
Intuitively, the selective classifier abstains if the score is within the bounds $\theta_l$ and $\theta_u$. 

To identify the bounds of $g$, we consider the following scenario, where we have a sample $S_n=\{\mathbf{x}_{i}, y_{i}\}_{i=1}^{n}$ and a classifier $h$. 
Our objective is to estimate a score-bounded selection function $g$ such that: \textit{(i)} $\hat{\phi}(g|S_{n}) \geq c$; and \textit{(ii)}: $\widehat{AUC}(h, g | S_{n})$ is as large as possible.
Let us define $n = n^+ + n^-$, with $n^+$ (resp., $n^-$) the number of positive (resp., negative) instances in $S_n$. Let $p^+=n^+/n$ be the fraction of positives and let $p^-= 1- p^+$ be the fraction of  negatives. Without loss of generality, we can assume that $\mathbf{x}_1, \ldots, \mathbf{x}_n$ is ordered by non-ascending scores, i.e., $h(\mathbf{x}_i) \geq h(\mathbf{x}_{i+1})$ for $i \in [1, n-1]$.
The rank of an instance is $r(\mathbf{x}_{i})=i$, and its relative rank is $r(\mathbf{x}_{i})/n$. We write $t({\mathbf{x}_{i}})$ to indicate the number of positive instances occurring in ranks from $1$ to $i$. The true positive rate (TPR) at $\mathbf{x}_i$ is the ratio $tpr(\mathbf{x}_{i})=t(\mathbf{x}_{i})/n^+$. We depart from reasoning directly on the definition of the AUC. 
Instead, we consider the linearly related definition of the Gini coefficient \citep{Engelmann2011}:
\begin{equation}\label{eq:g2auc}
    G = 2\cdot AUC - 1
\end{equation}
for which the maximization problem is equivalent. The Gini coefficient (also known as Accuracy Ratio) is defined starting from the Cumulative Accuracy Profile (CAP), which maps relative rank to TPR. The Gini coefficient can be written as $G = A/(A+B)$ where $A$ is the area of the CAP between the diagonal and the line of the classifier $h$, and $B$ is the area between the line of $h$ and the line of a perfect classifier $h^{\star}$. A perfect classifier assigns a score of $1$ to positives and a score of $0$ to negatives; hence it ranks all positives first, then all negatives afterwards. The diagonal line represents the performance of a random classifier. In summary, the Gini coefficient measures the discriminative power of a probabilistic classifier as a fraction of the difference in power between a perfect classifier and a random classifier.
Using the CAP plot and the Gini coefficient will be crucial in proving our results.
Figure \ref{fig:CAP_modified} shows a sample CAP plot where the axes are ranks and true positives instead of their relative counterparts. The Gini coefficient can equivalently be computed as ratio $A n^+ n /( A n^+ n + B n^+ n)$. We observe that $(A+B) n^+ n$ is $n^+ n^+ /2$ (the area of $h^{\star}$ for positives) plus $n^- n^+$ (the area of $h^{\star}$ for negatives) minus $n^+ n/2$ (the area under the diagonal). Hence $A+B = n^-/(2n) = p^-/2$.

Let us define $r'(\mathbf{x}_i)=n-r(\mathbf{x}_i)+1$, i.e. $r'(\mathbf{x}_i)$ is the rank over the non-descending scores. We show next a (sufficient) condition to improve $\widehat{AUC}(h, g | S_{n})$ by abstaining on positive instances.
\begin{figure*}[t!]
\begin{minipage}{.31\textwidth}
    \centering
    \includegraphics[scale=0.18]{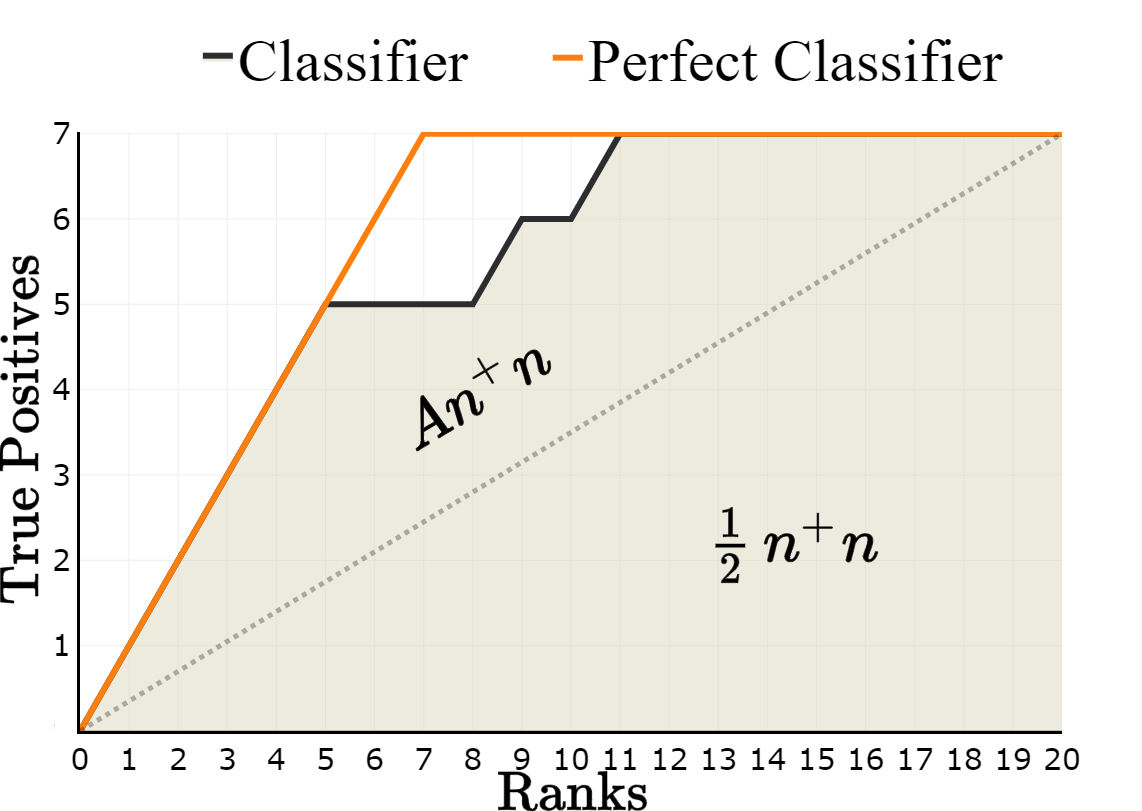}\\[1ex]
    \captionof{figure}{\small{Sample CAP plot with $n=20$ and $n^+=7$.}}
    \label{fig:CAP_modified}
\end{minipage}\hfill
\begin{minipage}{.31\textwidth}
    \centering
    \includegraphics[scale=0.18]{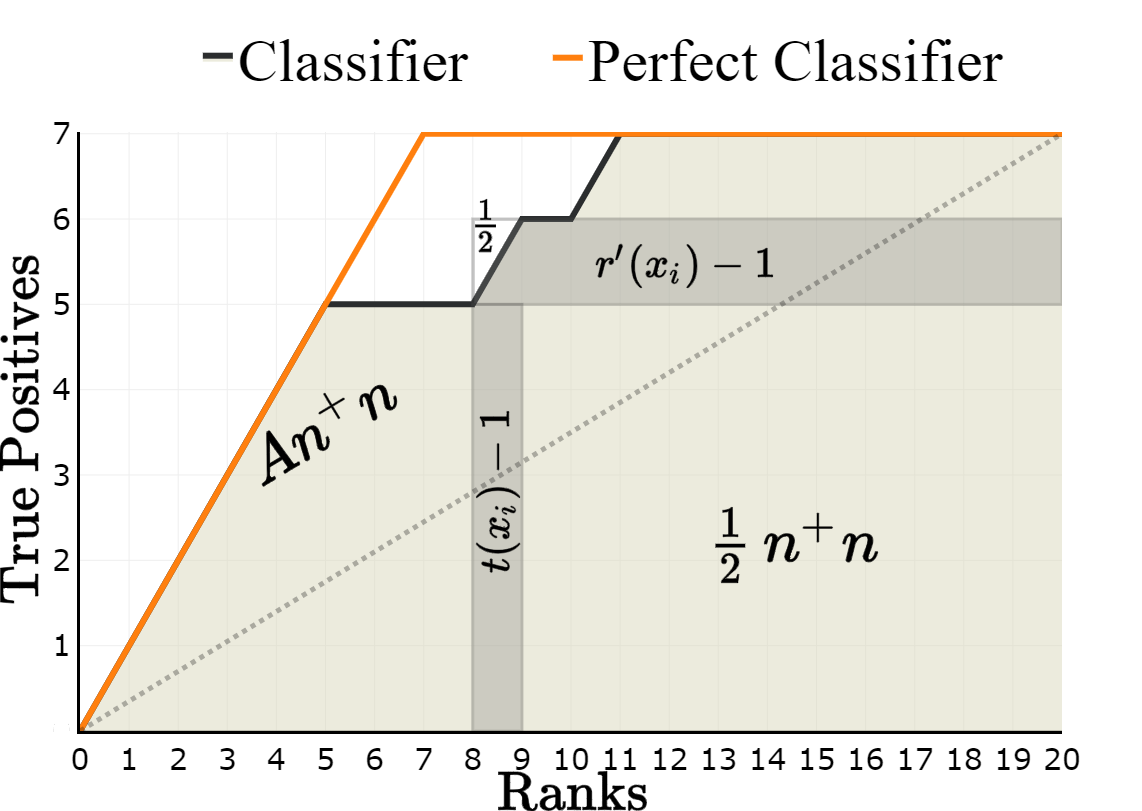}\\[1ex]
    \captionof{figure}{\small{CAP plot before removing a positive instance.}}
    \label{fig:CAP_remove_pos}
\end{minipage}\hfill
\begin{minipage}{.31\textwidth}    \centering
    \includegraphics[scale=0.18]{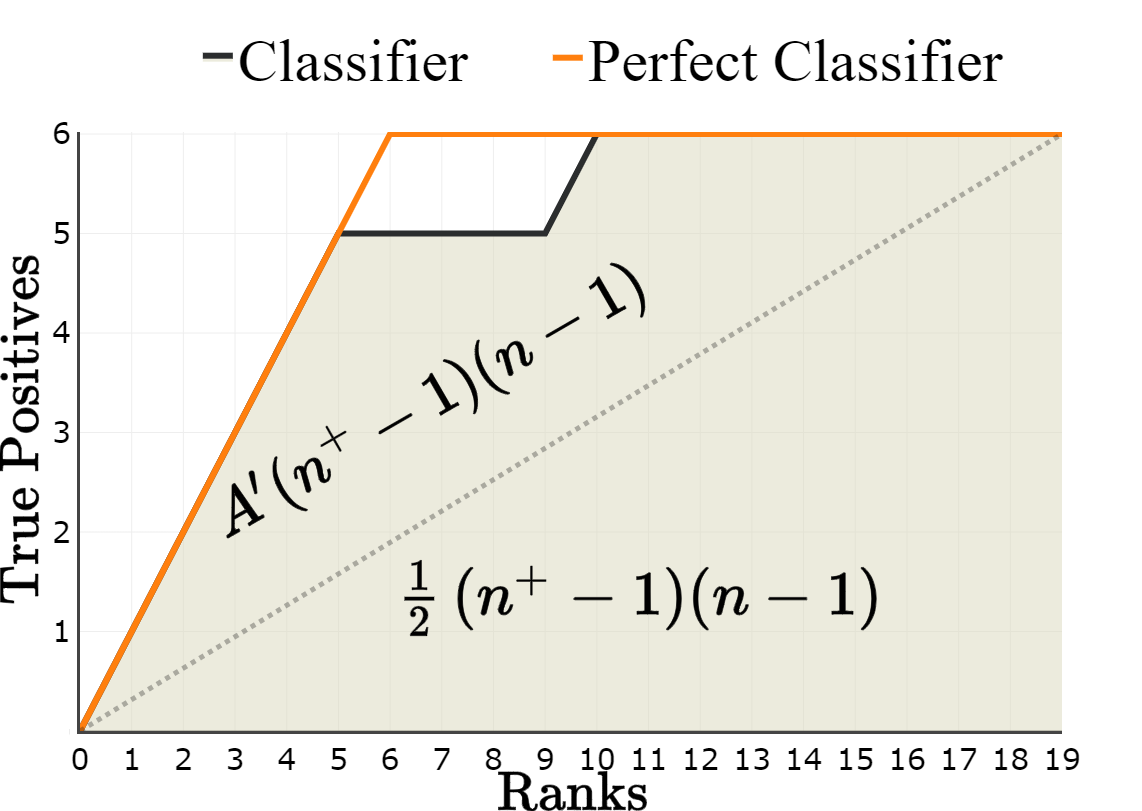}\\[1ex]
    \captionof{figure}{\small{CAP plot after removing a positive instance.}}
    \label{fig:CAP_after_remove_pos}
\end{minipage}
\end{figure*}
\begin{propo}
For any (number of) positive instance $\mathbf{x}_i$ such that $g(\mathbf{x}_i)=1$ and:
\begin{equation}
    \frac{r'(\mathbf{x}_i)}{n} \leq \widehat{AUC}(h, g | S_{n}) \cdot p^-
    \label{eq:pro1cond}
\end{equation}
we have: $\widehat{AUC}(h, g | S_{n}) \leq \widehat{AUC}(h, g' | S_{n})$, where $g'(\mathbf{x}_i)=0$ and $g'(\mathbf{x})=g(\mathbf{x})$ otherwise.
The inequality is strict if at least one such $\mathbf{x}_i$ exists in $S_n$.
\label{pro:positive1}
\end{propo}
\begin{proof}[Proof]
By (\ref{eq:g2auc}), we can equivalently show the result for the Gini coefficient. Let $G = A/(A+B)$, and let $\bar{G} = \bar{A}/(\bar{A}+\bar{B})$ be the Gini coefficient after removing (a.k.a., abstaining on) one positive instance $\mathbf{x}_i$. We have $A+B = n^-/(2n)$ and $\bar{A}+\bar{B} = n^-/(2(n-1))$.
Then $\bar{G} > G$ iff:
\begin{equation}
    \label{eq:gg}
    \bar{A} (n-1) > A n
\end{equation}
As $r'(\mathbf{x}_i) = n - r(\mathbf{x}_i) + 1$, we have that $\bar{A}$ is related to $A$ as follows:
\begin{equation}
\resizebox{0.48\textwidth}{!}{$
\bar{A} + \frac{1}{2} = ( A  + \frac{1}{2} - \left[\frac{(r'(\mathbf{x}_i)-1) + (t(\mathbf{x}_i)-1) + \nicefrac{1}{2}}{n^+ n} \right]) \frac{n^+ n}{(n^+-1) (n-1)}$}
\label{eq:gini_eq}
\end{equation}
Using such equality, the condition in (\ref{eq:gg}) can be simplified to (see Appendix for full derivation):
\begin{equation*}
    \label{eq:ub}
    \frac{r'(\mathbf{x}_i)}{n} \leq A + \frac{ p^-}{2} + \frac{(n^+ - t(\mathbf{x}_i))}{n}
\end{equation*}
%
Since $n^+ \geq t(\mathbf{x}_i)$, the inequality above is satisfied if:
\begin{equation}\label{eq:suff}
    \frac{r'(\mathbf{x}_i)}{n}  \leq
    A + \frac{p^-}{2}
\end{equation}
By (\ref{eq:g2auc}) and $G = A/(A+B) = 2 A/p^-$, we have $A + \nicefrac{p^-}{2} = \widehat{AUC}(h, g | S_{n}) \cdot p^-$, and thus the conclusion (\ref{eq:pro1cond}) holds after removing \textit{one} positive instance.
%
Let us then consider the case when we remove $(1-\alpha)n^+$ positive instances, with $\alpha \in [0,1[$. Inequality (\ref{eq:suff}) becomes:
\begin{equation}
    \label{eq:al}
    \frac{r'(\mathbf{x}_i)}{n} \leq  A_{\alpha} \frac{(n^-+\alpha n^+)}{n} + \frac{p^-}{2} 
\end{equation}  
where $A_{\alpha}$ is the numerator of the Gini fraction after removing $(1-\alpha)n^+$ positives.
By repeatedly removing one positive instance at a time, (\ref{eq:gg}) implies that 
$A_{\alpha} (n^-+\alpha n^+) > A n$. Therefore, if (\ref{eq:pro1cond}) holds, then (\ref{eq:al})  holds for any $\alpha$. Therefore, by removing as many as possible positive instances $\mathbf{x}_i$ such that (\ref{eq:pro1cond}) holds, we increase the Gini coefficient, and, a fortiori, the $\widehat{AUC}(h, g' | S_{n})$. \end{proof}

A key step in the proof is equation (\ref{eq:gini_eq}). Let us give some intuitions. Abstaining on a positive instance means removing some areas in the CAP plot, as shown in grey in Figure \ref{fig:CAP_remove_pos}. The vertical grey band consists of $t(\mathbf{x}_i)-1$ cells in the $n  n^+$ grid. The horizontal grey band consists of $r'(\mathbf{x}_i)-1$ cells. In addition, an area of $1/2$ is removed from the cell with the increase of the classifier line.
Finally, after removing those areas, the grid is rescaled from $n  n^+$ to $(n^+ -1)(n-1)$, which provides the rescaling factor in (\ref{eq:gini_eq}). Figure \ref{fig:CAP_after_remove_pos} shows the CAP plot after removing a positive instance.

A similar result holds when removing negative examples. The condition is based on the TPR of the removed instance.
\begin{propo}
For any (number of) negative instance $\mathbf{x}_i$ such that $g(\mathbf{x}_i)=1$ and:
\begin{equation}
    \frac{t(\mathbf{x}_i)}{n^+} \leq \widehat{AUC}(h, g | S_{n}) - \frac{1}{n^+}
    \label{eq:pro2cond}
\end{equation}
we have: $\widehat{AUC}(h, g | S_{n}) \leq \widehat{AUC}(h, g' | S_{n})$, where $g'(\mathbf{x}_i)=0$ and $g'(\mathbf{x})=g(\mathbf{x})$ otherwise.
The inequality is strict if at least one such $\mathbf{x}_i$ exists in $S_n$.
\label{pro:negative1}
\end{propo}

In general, the selection functions $g'$ in Prop. \ref{pro:positive1} and Prop. \ref{pro:negative1} are not score-bounded. Let us assume that $g(\mathbf{x})=1$ for all $\mathbf{x}$, hence $\widehat{AUC}(h, g | S_{n}) = \widehat{AUC}(h | S_{n})$.
We can easily lift condition (\ref{eq:pro1cond}) to scores as follows:
\[ h(\mathbf{x}_i) \leq  \hat{q}_n(\widehat{AUC}(h | S_{n}) \cdot p^-) = \theta_u
\] 
where $\hat{q}_n$ is the empirical quantile function over the scores $\{ h(\mathbf{x}_i) \ |\ (\mathbf{x}_i, y_i) \in S_n\}$. As a consequence, $g'$ is score bounded as it boils down to $g'(\mathbf{x}) = 0$ iff $0 \leq h(\mathbf{x}) \leq \theta_u$.
Regarding (\ref{eq:pro2cond}), since TPR is anti-monotonic with the scores, we can restate the condition as follows: 
\[ \theta_l = \hat{q}_{n^+}(\widehat{AUC}(h | S_{n})-\nicefrac{1}{n^+}) \leq h(\mathbf{x}_i) \]
where $\hat{q}_{n^+}$ is the empirical quantile function over the scores of the positives $\{ h(\mathbf{x}_i) \ |\ (\mathbf{x}_i, 1) \in S_n\}$. 
In this case, $g'$ is score bounded as $g'(\mathbf{x}) = 0$ iff $\theta_l \leq h(\mathbf{x}) \leq 1$.

In summary, if $h(\mathbf{x}_i)$ is within the  bounds $\theta_l$ and $\theta_u$, then, irrespective whether $\mathbf{x}_i$ is positive or negative, by abstaining on $\mathbf{x}_i$ we obtain an increase in $\widehat{AUC}$. In summary, we have the following result.

\begin{propo}
Called $\theta_{l} = \hat{q}_{n^+}(\widehat{AUC}(h | S_{n})-\nicefrac{1}{n^+})$ and $\theta_u = \hat{q}_n(\widehat{AUC}(h | S_{n}) \cdot p^-)$, we define $g(x)$ as:
    \[ g(\mathbf{x}) = \begin{cases}
    0 \quad \text{if } \theta_{l} 
    \leq 
    h(\mathbf{x}) 
    \leq 
    \theta_u
    \label{eq:pro3cond}\\ 
    1 \quad \text{otherwise}. \end{cases}\]
Then we have: $\widehat{AUC}(h | S_{n}) \leq \widehat{AUC}(h, g | S_{n})$.
The inequality is strict if at least one  $\mathbf{x}_i$ exists in $S_n$ such that $g(\mathbf{x}_i)=0$.
\label{pro:final}
\end{propo}

\section{THE AUCROSS ALGORITHM}
\label{sec:alg}
Alg. \ref{alg:est} shows a procedure to calculate the bounds stated in Prop. \ref{pro:final} starting from empirical scores and true class labels.

\newcommand\mycommfont[1]{\scriptsize\ttfamily\scriptsize{#1}}
\SetCommentSty{mycommfont}

\IncMargin{1.5em}
\begin{algorithm2e}[h!]
\small 
    \caption{EstimateThetasAUC()}
    \label{alg:est}
	\SetKwInOut{Input}{Input}
	\SetKwInOut{Output}{Output}
	\Input{$(\mathbf{s}, \mathbf{y})$ - scores and true class labels
    }
	\Output{$(\theta_l, \theta_u)$ - bounds for AUC-based selective classification\\}
	\BlankLine
        $n, n^+, p^- \leftarrow |\mathbf{y}|, |\mathbf{y} == 1|, 1-{n^+}/{n}$\\
        $\widehat{AUC} \leftarrow AUC.ROC(\mathbf{s}, \mathbf{y})$\tcp*[f]{compute empirical AUC}\\
        $u \leftarrow \lfloor \widehat{AUC} \cdot p^- \cdot n \rfloor $\tcp*[f]{$\theta_{u}$ position}\\
        $\mathbf{tpr} \leftarrow 1 - cum.sum(\mathbf{y}[order(\mathbf{s})])/n^+$\tcp*[f]{compute true positive rates}\\
        $l\leftarrow  {search.sorted}{(\mathbf{tpr}, \widehat{AUC}-1/n^+)}$ \tcp*[f]{$\theta_{l}$ position}\\
        $\mathbf{ss} \leftarrow {sort}(\mathbf{s})$\tcp*[f]{sort scores ascending}\\
        $\theta_{l}, \theta_{u} \leftarrow \mathbf{ss}[l], \mathbf{ss}[u]$\tcp*[f]{bounds of $g$}\\
    \Return{$(\theta_{l}, \theta_{u})$}
\end{algorithm2e}
We devise an approach for lifting $H$ to an AUC-selective classifier from this procedure.
We call this induction algorithm \textsc{AUCross}.
The approach differs from existing methods in several aspects.
First, it aims to determine thresholds $\theta_l$ and $\theta_u$ specifically designed for the AUC selective classification problem  (\ref{eq:aucsc}).
Second, we prevent setting apart a validation set from the training set (as done in state-of-the-art methods) to be able to fit the classifier $h$ on the whole available training set. Since the goal is to estimate quantiles over the (unknown) population of scores of the classifier $h$, we approximate sampling from such a population by using a cross-fitting strategy, as in \cite{PugnanaRuggieri2023}\footnote{Exploiting cross-fitting is a common practice also in other fields, such as the double ML approach \cite{chernozukov2018} for causal inference and the cross-conformal prediction algorithm \cite{DBLP:journals/ml/Vovk13}.}.
Third, existing state-of-the-art methods calibrate the selection function by computing the empirical quantile of scores over a validation set. We adopt a quantile estimator based on subsamples, which improves over the variance of a full-sample quantile estimator. The theoretical backbone of our approach is based on the results by \cite{Knight2003}, reported next for completeness.
\begin{theorem}[\citep{Knight2003}, Theorem 3]\label{teo1}
Given a random sample distributed according to $F$, satisfying some regularity conditions, and $K$ non-overlapping subsamples of it, let  $\hat{q}(\alpha)$ be the empirical $\alpha$-quantile estimator of $F$ over the whole sample, and $\bar{q}(\alpha)$ a weighted average of the empirical quantile estimators of $F$ over the subsamples. 
For $t\in[0,1]$, let us define the linear combination 
$\tilde{q}(\alpha)=t\hat{q}(\alpha)+(1-t)\bar{q}(\alpha)$. The variance of $\tilde{q}(\alpha)$ is minimized for $t=\nicefrac{1}{\sqrt{2}}$, $K=2$ and equally sized subsamples.
\end{theorem}
The sample quantile $\hat{q}(\alpha)$ is known to be asymptotically normal. The weighted mean of subsample quantiles $\bar{q}(\alpha)$ is first-order equivalent to $\hat{q}(\alpha)$. 
The above theorem states the conditions for minimizing the variance of (the second-order term of) any linear combination of the two estimators.

We report the pseudo-code of \textsc{AUCross} in Alg. \ref{alg:fit}.

\begin{algorithm2e}[h!]
\small 
    \caption{\textsc{AUCross}.fit()}
    \label{alg:fit}
	\SetKwInOut{Input}{Input}
	\SetKwInOut{Output}{Output}
	\Input{$(\mathbf{X}, \mathbf{y})$ - training set,\\
    $H$ - binary probabilistic classifier,\\
    $c$ - target coverage,\\ $K$ - number of folds\\
    }
	\Output{$(h, g)$ - selective classifier}
	\BlankLine
	$\mathit{Lbs}, \mathit{Ubs} \leftarrow [], []$\tcp*[f]{empty lists of  bounds} \\
	$S \leftarrow \text{StratifiedKFold}((\mathbf{X}, \mathbf{y}),K)$\tcp*[f]{stratified $K$-fold partition}\\
	\For(\tcp*[f]{for each fold}){$\mathbf{X}_k, \mathbf{y}_k \in S$}{
	    $(\mathbf{X}'_k, \mathbf{y}'_k) = (\mathbf{X} - \mathbf{X}_k, \mathbf{y}-\mathbf{y}_k)$\tcp*[f]{training data}\\
	    $h_k \leftarrow H.fit(\mathbf{X}'_k, \mathbf{y}'_k)$\tcp*[f]{train $k^{th}$ classifier}\\
	    $\mathbf{s}_k \leftarrow h_k.score(\mathbf{X}_k)$\tcp*[f]{score test data}\\}
	$\mathbf{s} \leftarrow \cup_{k=1}^K \mathbf{s}_k$
	\tcp*[f]{store all the scores}\\
    $n \leftarrow |\mathbf{y}|$  \tcp*[f]{number of instances}\\ 
    $\theta_l, \theta_u \leftarrow \textit{EstimateThetasAUC}(\mathbf{s},\mathbf{y})$\tcp*[f]{bounds for scores over whole training set}\\ 
    $J \leftarrow KFold((\mathbf{s},\mathbf{y}), 2)$ \tcp*[f]{split scores and actual values in two sub-samples} \\
    \For(\tcp*[f]{for each sub-sample}){$(\mathbf{s}_j,\mathbf{y}_j)\in J$}{
        $\theta_{l_j}, \theta_{u_j} \leftarrow \textit{EstimateThetasAUC}(\mathbf{s}_j,\mathbf{y}_j)$
    }
    $\theta_{u^*} \leftarrow \frac{1}{\sqrt{2}}\theta_u +(1-\frac{1}{\sqrt{2}})\sum_{i=1}^2\theta_{u_i}$\tcp*[f]{estimate of $\theta_u$}\\
    $\theta_{l^*}\leftarrow \frac{1}{\sqrt{2}}\theta_l +(1-\frac{1}{\sqrt{2}})\sum_{i=1}^2\theta_{l_i}$\tcp*[f]{estimate of $\theta_l$}\\
    $mid \leftarrow \lfloor n \cdot (\theta_{u^*}+\theta_{l^*})/2 \rfloor$ \tcp*[f]{compute mid point position} \\
    $\delta \leftarrow n \cdot (1-c)/{2} $ \tcp*[f]{half-width of rejection area}\\
    $u'\leftarrow \min\{\lfloor mid+\delta \rfloor, n\}$ \tcp*[f]{upper bound position}\\
    $l'\leftarrow \max\{1,\lfloor mid-\delta\rfloor\}$\tcp*[f]{lower bound position}\\
    $\mathbf{ss} \leftarrow {sort}(\mathbf{s})$\tcp*[f]{sort scores ascending}\\
    $\theta_{l'}, \theta_{u'} \leftarrow \mathbf{ss}[l'], \mathbf{ss}[u']$\tcp*[f]{bounds of $g$}\\
    $h \leftarrow H.fit(\mathbf{X}, \mathbf{y})$\tcp*[f]{train  classifier}\\ 
    $g \leftarrow \text{\textbf{lambda}}\;\mathbf{x}: 1-\mathbb{1}(\theta_{l'} \leq h.score(\mathbf{x}) \leq \theta_{u'})$\\ \tcp*[f]{selection function}\\
    \Return{$(h, g)$}
\end{algorithm2e}

\textsc{AUCross} iterates over a stratified $K$-fold partitioning of the training set (lines 3-6 of Alg. \ref{alg:fit}), where, for each of the folds $k$, we use $K-1$ folds to train a classifier $h_k$ (using $H$) and to predict scores $h_k(\mathbf{X}_k)$ over the  fold $k$ (lines 4-5). 
We store all these predictions in $\mathbf{s}$ (line 7), and we use them to compute quantiles according to Prop. \ref{pro:final} (line 9 and Alg. \ref{alg:est}).
Since $\theta_l$ and $\theta_u$ are quantiles estimates, we can improve their second-order behaviour by randomly splitting the obtained scores and the target variable $(\mathbf{s},\mathbf{y})$ into two parts (line 10) and repeat the estimation of quantiles over subsamples (lines 11-12). We combine the estimates according to Thm. \ref{teo1} to obtain the final quantiles $\theta_{l^*}$ and $\theta_{u^*}$ (lines 13-14).

The quantile estimation methods considered so far, and in particular $\theta_l$ and $\theta_u$ in Prop. \ref{pro:final}, are intended to maximize the AUC independently of the desired coverage $c$. To address the minimum coverage constraint $c$, we centre the rejection area at the midpoint between the instances with score $\theta_{l^*}$ and $\theta_{u^*}$ (line 15). 
Such a mid point $mid$ is the median of the distribution of scores from $\theta_{l^*}$ up to $\theta_{u^*}$. We consider a rejection area $[l', u']$ centered in $mid$ and with width $n \cdot (1-c)$, also checking not to exceed the range $[1, n]$ (lines 16-18).
The final selective classifier $(h,g)$ is then obtained by: \textit{(i)} fitting $H$ on the whole sample to get $h$ (line 21); \textit{(ii)} setting the bounds in the selection function $g$ to the scores $\theta_{l'}$ and $\theta_{u'}$ at the boundaries $l'$ and $u'$ of the rejection area (lines 20 and 22).

We point out that the bounds computed by \textsc{AUCross} may be sup-optimal, either because they do not achieve the minimum target coverage $c$ or do not maximize the AUC. Even when restricting to the class of score-bounded selection functions, the reached AUC may be sub-optimal: for instance, we used (\ref{eq:suff}) as a sufficient but not necessary condition during the proof of Prop. \ref{pro:positive1}. We investigate the experimental performance of \textsc{AUCross} in the next section, also concerning such theoretical caveats.

\section{EXPERIMENTS}

We run experiments on nine real-world datasets (seven tabular datasets and two image datasets). We split available data instances into 75\% for training selective classifiers and 25\% for performance testing. When possible, the split was based on timestamps to allow for out-of-time validation. Otherwise, a stratified random split is performed. A summary of the experimental datasets is reported in Table \ref{tab:data}, including the number of features, size for training and test, and the rate of positives in the test set. For the sake of space, we detail the pre-processing procedures in the Appendix. We only mention here that, since we are interested in a binary classification task, we labelled all the cat images of \textit{CIFAR-10-cat} as the positive class, while images of other classes are considered as negatives. 
\begin{table}[t!]

    \centering
        \caption{Experimental datasets.}
    \resizebox{.48\textwidth}{!}{
\begin{tabular}{c|c|c|c|c}
\textbf{Dataset} & \textbf{\# Features} & \textbf{Training Size} & \textbf{Test Size} & \textbf{Positive Rate} \\
\midrule
\textit{\href{https://archive.ics.uci.edu/ml/datasets/adult}{Adult}} & 55    & 30,162 & 15,060 & .246 \\
\textit{\href{https://www.kaggle.com/wordsforthewise/lending-club}{LendingClub}} & 65    & 1,364,697 & 445,912 & .225 \\
\textit{\href{https://www.kaggle.com/c/GiveMeSomeCredit}{GiveMe}} & 12    & 112,500 & 37,500 & .067 \\
\textit{\href{https://archive.ics.uci.edu/ml/datasets/default+of+credit+card+clients}{UCICredit}} & 23    & 22,500 & 7,500 & .221 \\
\textit{CSDS1} \cite{DBLP:journals/eswa/BarddalLEL20} & 155   & 230,409 & 76,939 & .144 \\
\textit{CSDS2} \cite{DBLP:journals/eswa/BarddalLEL20} & 35    & 37,100 & 12,533 & .018 \\
\textit{CSDS3} \cite{DBLP:journals/eswa/BarddalLEL20} & 144   & 71,177 & 23,288 & .253 \\
\textit{\href{https://www.kaggle.com/competitions/dogs-vs-cats}{CatsVsDogs}} & 64x64 & 20,000 & 5,000 & .500 \\
\textit{\href{https://www.cs.toronto.edu/~kriz/cifar.html}{CIFAR-10-cat}} & 32x32 & 50,000 & 10,000 & .100 \\
\bottomrule
\end{tabular}
}
\label{tab:data}
\end{table}
Experiments were run on a machine with 96 cores equipped with Intel(R) Xeon(R) Gold 6342 CPU @ 2.80GHz and two NVIDIA RTX A6000, OS Ubuntu 20.04.4, programming\,language\,Python\,3.8.12.

\label{sec:exp}


\begin{table*}
\begin{center}
\caption{\small{Absolute deviation in empirical AUC of \textsc{AUCross} w.r.t. an \textsc{Oracle}.\label{tab:auc_dev}}}
\resizebox{.8\textwidth}{!}{
\begin{tabular}{c|ccccccccc}
\textbf{c} & \textbf{Adult} & \textbf{Lending} & \textbf{GiveMe} & \textbf{UCICredit} & \textbf{CSDS1} & \textbf{CSDS2} & \textbf{CSDS3} & \textbf{CatsVsDogs} & \textbf{CIFAR-10-cat} \\
\midrule
.99   & .0003 & .0000 & .0004 & .0008 & .0003 & .0057 & .0001 & .0000 & .0001 \\
.95   & .0011 & .0000 & .0009 & .0008 & .0012 & .0050  & .0005 & .0000 & .0002 \\
.90   & .0022 & .0000 & .0004 & .0007 & .0024 & .0080  & .0011 & .0000 & .0006 \\
.85   & .0033 & .0000 & .0003 & .0013 & .0033 & .0027 & .0017 & .0004 & .0016 \\
.80   & .0033 & .0000 & .0006 & .0007 & .0039 & .0006 & .0020  & .0004 & .0049 \\
.75   & .0037 & .0001 & .0005 & .0001 & .0053 & .0027 & .0023 & .0000 & .0193 \\
\end{tabular}%

}
\end{center}
\end{table*}
\begin{figure}
\centering
\includegraphics[scale=0.3]{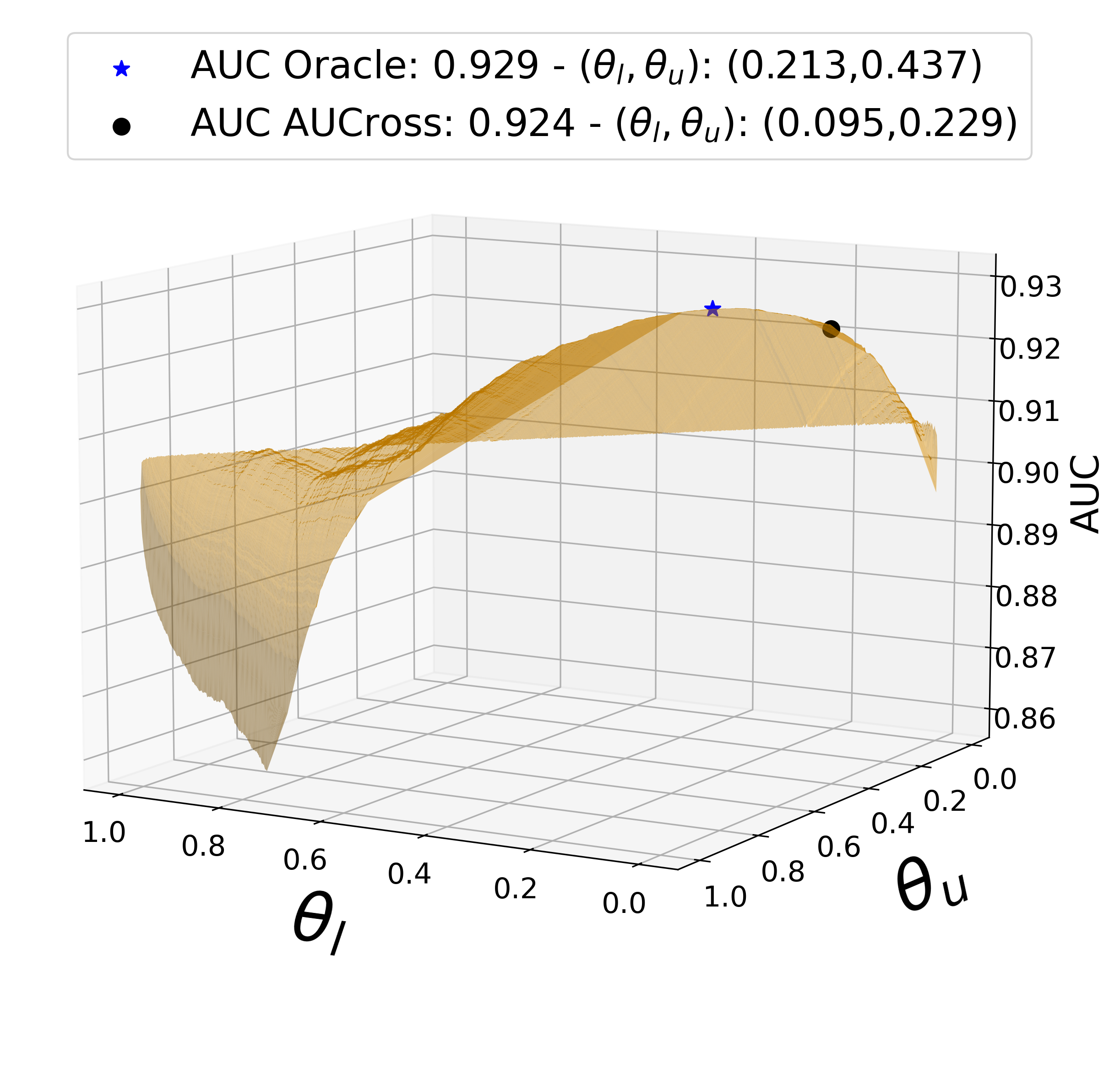}\\[-2ex]
\caption{\small{AUC by varying bounds: Adult and $c=.90$.}}
\label{fig:grid}
\end{figure}

\paragraph{\textsc{AUCross} vs \textsc{Oracle}}The first experiment evaluates how well the \textsc{AUCross} algorithm approximates the optimal solution to the problem of maximizing  $\widehat{AUC}(h, g | S_n)$ s.t. $\hat{\phi}(g|S_n) \geq c$ over the test set $S_n$, considering the family of score-bounded selection functions. For tabular data, $h$ is fixed to a LightGBM classifier \citep{DBLP:conf/nips/KeMFWCMYL17} with default hyperparameters. For image data, we used a VGG architecture as a base classifier \citep{DBLP:journals/corr/SimonyanZ14a}, trained with cross-entropy loss. 
We enumerate all the score-bounded selection functions $g(\mathbf{x}) = 0$ iff $\theta_l \leq h(\mathbf{x}) \leq \theta_u  $ by varying the bounds $\theta_l, \theta_u \in \{ h(\mathbf{x}_i) \ |\ (\mathbf{x}_i, y) \in S_n \}$ such that $\hat{\phi}(g|S_n) \geq c$. The \textsc{Oracle} approach chooses the function $g$ which maximizes $\widehat{AUC}(h, g | S_n)$. It requires knowing the true class of instances in $S_n$; hence it is not feasible in practice. Table~\ref{tab:auc_dev} shows the difference between the empirical AUC ($\widehat{AUC}$) achieved by the \textsc{Oracle} and the one of \textsc{AUCross} on the test sets of the experimental datasets (the value of the empirical AUC for \textsc{AUCross} is reported in Table~\ref{tab:metrics} and discussed later on). 
We notice that \textsc{AUCross} reaches the \textsc{Oracle} performance for the \textit{LendingClub} and \textit{CatsVsDogs} datasets.
For the datasets \textit{Adult}, \textit{CSDS1}, \textit{CSDS3} and \textit{CIFAR-10-cat}, there are  larger violations for smaller $c$'s, while we do not see such a trend for \textit{GiveMe}, \textit{UCICredit} and \textit{CSDS2}. This suggests that the strategy of centring the rejection area at a midpoint (line 15 of Alg. \ref{alg:fit}) might be sub-optimal. Figure \ref{fig:grid} shows the performance over the space of bounds in a specific case, highlighting the bounds of \textsc{Oracle} and \textsc{AUCross}. 

\paragraph{\textsc{AUCross} vs baselines} We compare the performance of  \textsc{AUCross} to a few baselines. The first one is a bounded-improvement version of the plug-in rule \citep{Herbei06} (\textsc{PlugIn}). \textsc{PlugIn} uses softmax  as confidence function, i.e.,~in the binary case, $g(x)=\mathbb{1}(\max\{h(x),1-h(x)\}>\theta)$, and it estimates the $\theta$ parameter as the $(1-c)$ quantile on a validation set. Such an approach requires splitting the training set into a dataset for building the classifier and a validation set for estimating $\theta$ (10\% of the original dataset in our experiments). In contrast, \textsc{AUCross} allows for building the classifier using the whole training set.
Second, we consider two hybrid versions of \textsc{AUCross} and \textsc{PlugIn}: (1) a plug-in rule specialized for AUC (\textsc{PlugInAUC}), i.e., that uses a score-bounded selection function by estimating bounds as in Alg.~\ref{alg:est} on a validation set; (2)  a cross-fitting version of \textsc{PlugIn} (called \textsc{SCross} by \cite{PugnanaRuggieri2023}), where the selection function is softmax and the estimation of the bounds exploits the second-order improvement of
of Thm. \ref{teo1}.
\begin{table*}[ht!]
\centering
\caption{\small{Performance metrics (1,000 bootstrap runs over the test set, results as mean $\pm$ stdev). $V$ (for violation) is the absolute difference between the empirical coverage and the target coverage $c$.}}
\resizebox{1.01\textwidth}{!}{%
\begin{tabular}{c|c|>{\centering\arraybackslash}p{1.75cm}>{\centering\arraybackslash}p{1.71cm}>{\centering\arraybackslash}p{1.92cm}ccc|>{\centering\arraybackslash}p{1.75cm}>{\centering\arraybackslash}p{1.71cm}>{\centering\arraybackslash}p{1.92cm}ccc}
\multicolumn{1}{c}{} & \multicolumn{1}{c}{} & \multicolumn{6}{c}{\textbf{Empirical Coverage }} & \multicolumn{6}{c}{\textbf{Selective AUC}} \\
      & \boldmath{}\textbf{$c$}\unboldmath{} & \textbf{\textsc{AUCross}} & \textbf{\textsc{PlugIn}} & \textbf{\textsc{PlugInAUC}} & \textbf{\textsc{SCross}} & \textbf{\textsc{SelNet}} & \textbf{\textsc{SAT}} & \textbf{\textsc{AUCross}} & \textbf{\textsc{PlugIn}} & \textbf{\textsc{PlugInAUC}} & \textbf{\textsc{SCross}} & \textbf{\textsc{SelNet}} & \textbf{\textsc{SAT}} \\
\midrule
\multirow{6}{*}{\rotatebox[origin=c]{90}{\textbf{Adult}}}  & .99   & \boldmath{}\textbf{.989 $\pm$ .001}\unboldmath{} & .992 $\pm$ .001 & .991 $\pm$ .001 & .993 $\pm$ .001 & .985 $\pm$ .001 & .998 $\pm$ .001 & \boldmath{}\textbf{.929 $\pm$ .003}\unboldmath{} & .928 $\pm$ .003 & .928 $\pm$ .003 & .928 $\pm$ .003 & .901 $\pm$ .003 & .902 $\pm$ .003 \\
      & .95   & \boldmath{}\textbf{.950 $\pm$ .002}\unboldmath{} & .947 $\pm$ .002 & .950 $\pm$ .002 & \boldmath{}\textbf{.950 $\pm$ .002}\unboldmath{} & .954 $\pm$ .002 & .984 $\pm$ .002 & \boldmath{}\textbf{.935 $\pm$ .003}\unboldmath{} & \boldmath{}\textbf{.935 $\pm$ .003}\unboldmath{} & .934 $\pm$ .003 & .935 $\pm$ .003 & .907 $\pm$ .003 & .903 $\pm$ .003 \\
      & .90   & .896 $\pm$ .003 & \boldmath{}\textbf{.900 $\pm$ .003}\unboldmath{} & .898 $\pm$ .003 & .904 $\pm$ .003 & .897 $\pm$ .003 & .966 $\pm$ .002 & \boldmath{}\textbf{.943 $\pm$ .002}\unboldmath{} & .941 $\pm$ .003 & .942 $\pm$ .003 & .941 $\pm$ .003 & .917 $\pm$ .003 & .904 $\pm$ .003 \\
      & .85   & .849 $\pm$ .003 & .846 $\pm$ .003 & .850 $\pm$ .003 & \boldmath{}\textbf{.850 $\pm$ .003}\unboldmath{} & .772 $\pm$ .004 & .952 $\pm$ .002 & \boldmath{}\textbf{.950 $\pm$ .002}\unboldmath{} & .947 $\pm$ .003 & .949 $\pm$ .002 & .948 $\pm$ .003 & .848 $\pm$ .005 & .904 $\pm$ .003 \\
      & .80   & .797 $\pm$ .004 & .794 $\pm$ .004 & .793 $\pm$ .004 & \boldmath{}\textbf{.802 $\pm$ .004}\unboldmath{} & .793 $\pm$ .004 & .934 $\pm$ .003 & \boldmath{}\textbf{.958 $\pm$ .002}\unboldmath{} & .952 $\pm$ .003 & .957 $\pm$ .002 & .953 $\pm$ .003 & .924 $\pm$ .004 & .904 $\pm$ .003 \\
      & .75   & \boldmath{}\textbf{.750 $\pm$ .004}\unboldmath{} & .743 $\pm$ .004 & .744 $\pm$ .004 & .746 $\pm$ .004 & .736 $\pm$ .004 & .919 $\pm$ .003 & \boldmath{}\textbf{.963 $\pm$ .002}\unboldmath{} & .958 $\pm$ .003 & .962 $\pm$ .002 & .958 $\pm$ .003 & .917 $\pm$ .004 & .904 $\pm$ .003 \\
\midrule
\midrule
\multirow{6}{*}{\rotatebox[origin=c]{90}{\textbf{Lending}}}  & .99   & .996 $\pm$ .001 & .995 $\pm$ .001 & .997 $\pm$ .001 & \boldmath{}\textbf{.994 $\pm$ .001}\unboldmath{} & .997 $\pm$ .001 & .073 $\pm$ .001 & .983 $\pm$ .001 & \boldmath{}\textbf{.984 $\pm$ .001}\unboldmath{} & .984 $\pm$ .001 & .983 $\pm$ .001 & .973 $\pm$ .001 & .664 $\pm$ .004 \\
      & .95   & .980 $\pm$ .001 & .970 $\pm$ .001 & .980 $\pm$ .001 & \boldmath{}\textbf{.969 $\pm$ .001}\unboldmath{} & .985 $\pm$ .001 & .073 $\pm$ .001 & \boldmath{}\textbf{.985 $\pm$ .001}\unboldmath{} & \boldmath{}\textbf{.985 $\pm$ .001}\unboldmath{} & \boldmath{}\textbf{.985 $\pm$ .001}\unboldmath{} & .984 $\pm$ .001 & .973 $\pm$ .001 & .664 $\pm$ .004 \\
      & .90   & .959 $\pm$ .001 & \boldmath{}\textbf{.939 $\pm$ .001}\unboldmath{} & .958 $\pm$ .001 & \boldmath{}\textbf{.938 $\pm$ .001}\unboldmath{} & .928 $\pm$ .001 & .073 $\pm$ .001 & \boldmath{}\textbf{.987 $\pm$ .001}\unboldmath{} & .986 $\pm$ .001 & \boldmath{}\textbf{.987 $\pm$ .001}\unboldmath{} & .986 $\pm$ .001 & .981 $\pm$ .001 & .664 $\pm$ .004 \\
      & .85   & .936 $\pm$ .001 & .909 $\pm$ .001 & .935 $\pm$ .001 & \boldmath{}\textbf{.908 $\pm$ .001}\unboldmath{} & .866 $\pm$ .001 & .073 $\pm$ .001 & \boldmath{}\textbf{.989 $\pm$ .001}\unboldmath{} & .988 $\pm$ .001 & \boldmath{}\textbf{.989 $\pm$ .001}\unboldmath{} & .987 $\pm$ .001 & .977 $\pm$ .001 & .664 $\pm$ .004 \\
      & .80   & .912 $\pm$ .001 & .880 $\pm$ .001 & .911 $\pm$ .001 & .879 $\pm$ .001 & \boldmath{}\textbf{.850 $\pm$ .001}\unboldmath{} & .073 $\pm$ .001 & .990 $\pm$ .001 & .989 $\pm$ .001 & \boldmath{}\textbf{.991 $\pm$ .001}\unboldmath{} & .989 $\pm$ .001 & .991 $\pm$ .001 & .664 $\pm$ .004 \\
      & .75   & .886 $\pm$ .001 & .850 $\pm$ .001 & .884 $\pm$ .001 & \boldmath{}\textbf{.849 $\pm$ .001}\unboldmath{} & .880 $\pm$ .001 & .073 $\pm$ .001 & \boldmath{}\textbf{.992 $\pm$ .001}\unboldmath{} & .990 $\pm$ .001 & \boldmath{}\textbf{.992 $\pm$ .001}\unboldmath{} & .990 $\pm$ .001 & .976 $\pm$ .001 & .664 $\pm$ .004 \\
\midrule
\midrule
\multirow{6}{*}{\rotatebox[origin=c]{90}{\textbf{GiveMe}}}  & .99   & \boldmath{}\textbf{.990 $\pm$ .001}\unboldmath{} & \boldmath{}\textbf{.990 $\pm$ .001}\unboldmath{} & \boldmath{}\textbf{.990 $\pm$ .001}\unboldmath{} & \boldmath{}\textbf{.990 $\pm$ .001}\unboldmath{} & \boldmath{}\textbf{.944 $\pm$ .002}\unboldmath{} & .746 $\pm$ .003 & .868 $\pm$ .004 & .861 $\pm$ .004 & \boldmath{}\textbf{.869 $\pm$ .004}\unboldmath{} & .859 $\pm$ .004 & .747 $\pm$ .007 & .730 $\pm$ .006 \\
      & .95   & \boldmath{}\textbf{.950 $\pm$ .002}\unboldmath{} & .949 $\pm$ .002 & .947 $\pm$ .002 & \boldmath{}\textbf{.950 $\pm$ .002}\unboldmath{} & .944 $\pm$ .002 & .727 $\pm$ .003 & .874 $\pm$ .004 & .824 $\pm$ .006 & \boldmath{}\textbf{.875 $\pm$ .004}\unboldmath{} & .826 $\pm$ .006 & .748 $\pm$ .007 & .729 $\pm$ .006 \\
      & .90   & .898 $\pm$ .002 & .903 $\pm$ .002 & .897 $\pm$ .002 & .902 $\pm$ .002 & \boldmath{}\textbf{.900 $\pm$ .002}\unboldmath{} & .705 $\pm$ .003 & \boldmath{}\textbf{.883 $\pm$ .004}\unboldmath{} & .784 $\pm$ .007 & \boldmath{}\textbf{.883 $\pm$ .004}\unboldmath{} & .784 $\pm$ .007 & .708 $\pm$ .008 & .728 $\pm$ .007 \\
      & .85   & .848 $\pm$ .002 & .856 $\pm$ .002 & .849 $\pm$ .002 & .852 $\pm$ .002 & \boldmath{}\textbf{.848 $\pm$ .002}\unboldmath{} & .699 $\pm$ .003 & \boldmath{}\textbf{.890 $\pm$ .004}\unboldmath{} & .759 $\pm$ .008 & \boldmath{}\textbf{.890 $\pm$ .004}\unboldmath{} & .757 $\pm$ .008 & .658 $\pm$ .009 & .728 $\pm$ .007 \\
      & .80   & .795 $\pm$ .003 & .807 $\pm$ .003 & .801 $\pm$ .003 & \boldmath{}\textbf{.800 $\pm$ .003}\unboldmath{} & .802 $\pm$ .003 & .694 $\pm$ .003 & \boldmath{}\textbf{.898 $\pm$ .004}\unboldmath{} & .744 $\pm$ .009 & .897 $\pm$ .004 & .743 $\pm$ .009 & .677 $\pm$ .009 & .727 $\pm$ .007 \\
      & .75   & .748 $\pm$ .003 & .756 $\pm$ .003 & .750 $\pm$ .003 & .748 $\pm$ .003 & \boldmath{}\textbf{.754 $\pm$ .003}\unboldmath{} & .689 $\pm$ .003 & \boldmath{}\textbf{.904 $\pm$ .004}\unboldmath{} & .738 $\pm$ .010 & \boldmath{}\textbf{.904 $\pm$ .004}\unboldmath{} & .734 $\pm$ .010 & .676 $\pm$ .010 & .727 $\pm$ .007 \\
\midrule
\midrule
\multirow{6}{*}{\rotatebox[origin=c]{90}{\textbf{UCICredit}}}  & .99   & .988 $\pm$ .002 & .993 $\pm$ .002 & .993 $\pm$ .001 & .993 $\pm$ .001 & \boldmath{}\textbf{.989 $\pm$ .002}\unboldmath{} & 1.000$\pm$ .000 & .772 $\pm$ .008 & .772 $\pm$ .007 & \boldmath{}\textbf{.774 $\pm$ .007}\unboldmath{} & .770 $\pm$ .008 & .765 $\pm$ .007 & .768 $\pm$ .007 \\
      & .95   & .944 $\pm$ .003 & .948 $\pm$ .003 & .946 $\pm$ .003 & \boldmath{}\textbf{.951 $\pm$ .003}\unboldmath{} & .909 $\pm$ .004 & 1.000$\pm$ .000 & .778 $\pm$ .008 & .769 $\pm$ .008 & \boldmath{}\textbf{.780 $\pm$ .007}\unboldmath{} & .765 $\pm$ .008 & .718 $\pm$ .009 & .768 $\pm$ .007 \\
      & .90   & .896 $\pm$ .004 & \boldmath{}\textbf{.897 $\pm$ .004}\unboldmath{} & .900 $\pm$ .004 & .906 $\pm$ .004 & .904 $\pm$ .004 & .994 $\pm$ .001 & .785 $\pm$ .008 & .758 $\pm$ .008 & \boldmath{}\textbf{.787 $\pm$ .008}\unboldmath{} & .754 $\pm$ .009 & .707 $\pm$ .009 & .769 $\pm$ .007 \\
      & .85   & .840 $\pm$ .005 & \boldmath{}\textbf{.851 $\pm$ .005}\unboldmath{} & .846 $\pm$ .005 & .848 $\pm$ .005 & .855 $\pm$ .005 & .994 $\pm$ .001 & .792 $\pm$ .008 & .746 $\pm$ .009 & \boldmath{}\textbf{.795 $\pm$ .008}\unboldmath{} & .734 $\pm$ .009 & .668 $\pm$ .009 & .769 $\pm$ .007 \\
      & .80   & .789 $\pm$ .005 & \boldmath{}\textbf{.803 $\pm$ .005}\unboldmath{} & \boldmath{}\textbf{.797 $\pm$ .005}\unboldmath{} & \boldmath{}\textbf{.797 $\pm$ .005}\unboldmath{} & .810 $\pm$ .005 & .994 $\pm$ .001 & .800 $\pm$ .008 & .722 $\pm$ .010 & \boldmath{}\textbf{.802 $\pm$ .008}\unboldmath{} & .710 $\pm$ .010 & .644 $\pm$ .010 & .769 $\pm$ .007 \\
      & .75   & .737 $\pm$ .006 & .751 $\pm$ .006 & .739 $\pm$ .006 & \boldmath{}\textbf{.750 $\pm$ .005}\unboldmath{} & .759 $\pm$ .005 & .994 $\pm$ .001 & .809 $\pm$ .008 & .695 $\pm$ .011 & \boldmath{}\textbf{.810 $\pm$ .008}\unboldmath{} & .679 $\pm$ .011 & .620 $\pm$ .011 & .769 $\pm$ .007 \\
\midrule
\midrule
\multirow{6}{*}{\rotatebox[origin=c]{90}{\textbf{CSDS1}}}  & .99   & \boldmath{}\textbf{.990 $\pm$ .001}\unboldmath{} & .980 $\pm$ .001 & .991 $\pm$ .001 & .980 $\pm$ .001 & .985 $\pm$ .001 & .984 $\pm$ .001 & \boldmath{}\textbf{.686 $\pm$ .003}\unboldmath{} & .675 $\pm$ .003 & \boldmath{}\textbf{.686 $\pm$ .003}\unboldmath{} & .675 $\pm$ .003 & .671 $\pm$ .003 & .669 $\pm$ .003 \\
      & .95   & \boldmath{}\textbf{.950 $\pm$ .001}\unboldmath{} & .918 $\pm$ .001 & .949 $\pm$ .001 & .917 $\pm$ .001 & .924 $\pm$ .001 & .937 $\pm$ .001 & \boldmath{}\textbf{.689 $\pm$ .003}\unboldmath{} & .654 $\pm$ .004 & \boldmath{}\textbf{.689 $\pm$ .003}\unboldmath{} & .655 $\pm$ .004 & .652 $\pm$ .004 & .654 $\pm$ .004 \\
      & .90   & \boldmath{}\textbf{.898 $\pm$ .002}\unboldmath{} & .853 $\pm$ .002 & .898 $\pm$ .002 & .849 $\pm$ .002 & .863 $\pm$ .002 & .861 $\pm$ .002 & \boldmath{}\textbf{.693 $\pm$ .003}\unboldmath{} & .640 $\pm$ .004 & \boldmath{}\textbf{.694 $\pm$ .003}\unboldmath{} & .641 $\pm$ .004 & .638 $\pm$ .004 & .635 $\pm$ .004 \\
      & .85   & \boldmath{}\textbf{.849 $\pm$ .002}\unboldmath{} & .792 $\pm$ .002 & .845 $\pm$ .002 & .790 $\pm$ .002 & .800 $\pm$ .002 & .803 $\pm$ .002 & .697 $\pm$ .003 & .631 $\pm$ .004 & \boldmath{}\textbf{.698 $\pm$ .003}\unboldmath{} & .631 $\pm$ .004 & .630 $\pm$ .004 & .625 $\pm$ .004 \\
      & .80   & \boldmath{}\textbf{.799 $\pm$ .002}\unboldmath{} & .737 $\pm$ .002 & .795 $\pm$ .002 & .736 $\pm$ .002 & .740 $\pm$ .002 & .739 $\pm$ .002 & .702 $\pm$ .004 & .624 $\pm$ .004 & \boldmath{}\textbf{.703 $\pm$ .003}\unboldmath{} & .620 $\pm$ .004 & .618 $\pm$ .004 & .616 $\pm$ .004 \\
      & .75   & \boldmath{}\textbf{.748 $\pm$ .002}\unboldmath{} & .681 $\pm$ .002 & .747 $\pm$ .002 & .682 $\pm$ .002 & .686 $\pm$ .002 & .687 $\pm$ .002 & .706 $\pm$ .004 & .613 $\pm$ .005 & \boldmath{}\textbf{.707 $\pm$ .004}\unboldmath{} & .616 $\pm$ .005 & .609 $\pm$ .005 & .608 $\pm$ .005 \\
\midrule
\midrule
\multirow{6}{*}{\rotatebox[origin=c]{90}{\textbf{CSDS2}}}  & .99   & \boldmath{}\textbf{.990 $\pm$ .001}\unboldmath{} & .984 $\pm$ .002 & .991 $\pm$ .001 & .983 $\pm$ .002 & .980 $\pm$ .002 & 1.000$\pm$ .000 & .616 $\pm$ .020 & .585 $\pm$ .020 & .587 $\pm$ .020 & .616 $\pm$ .020 & .607 $\pm$ .020 & \boldmath{}\textbf{.623 $\pm$ .019}\unboldmath{} \\
      & .95   & \boldmath{}\textbf{.952 $\pm$ .002}\unboldmath{} & .927 $\pm$ .003 & .954 $\pm$ .002 & .916 $\pm$ .003 & .894 $\pm$ .003 & 1.000$\pm$ .000 & \boldmath{}\textbf{.619 $\pm$ .020}\unboldmath{} & .572 $\pm$ .021 & .589 $\pm$ .020 & .581 $\pm$ .022 & .587 $\pm$ .022 & .623 $\pm$ .019 \\
      & .90   & \boldmath{}\textbf{.904 $\pm$ .003}\unboldmath{} & .851 $\pm$ .004 & .907 $\pm$ .003 & .834 $\pm$ .004 & .805 $\pm$ .004 & .880 $\pm$ .003 & \boldmath{}\textbf{.620 $\pm$ .020}\unboldmath{} & .566 $\pm$ .023 & .590 $\pm$ .021 & .570 $\pm$ .023 & .560 $\pm$ .024 & .588 $\pm$ .022 \\
      & .85   & \boldmath{}\textbf{.860 $\pm$ .004}\unboldmath{} & .773 $\pm$ .004 & .861 $\pm$ .004 & .761 $\pm$ .004 & .712 $\pm$ .005 & .817 $\pm$ .004 & \boldmath{}\textbf{.631 $\pm$ .021}\unboldmath{} & .552 $\pm$ .024 & .590 $\pm$ .021 & .572 $\pm$ .025 & .562 $\pm$ .026 & .589 $\pm$ .023 \\
      & .80   & \boldmath{}\textbf{.811 $\pm$ .004}\unboldmath{} & .715 $\pm$ .005 & .818 $\pm$ .004 & .687 $\pm$ .005 & .663 $\pm$ .005 & .756 $\pm$ .004 & \boldmath{}\textbf{.631 $\pm$ .021}\unboldmath{} & .532 $\pm$ .025 & .592 $\pm$ .021 & .555 $\pm$ .026 & .547 $\pm$ .027 & .582 $\pm$ .024 \\
      & .75   & \boldmath{}\textbf{.760 $\pm$ .004}\unboldmath{} & .651 $\pm$ .005 & .776 $\pm$ .004 & .620 $\pm$ .005 & .564 $\pm$ .005 & .697 $\pm$ .005 & \boldmath{}\textbf{.635 $\pm$ .021}\unboldmath{} & .536 $\pm$ .027 & .592 $\pm$ .021 & .541 $\pm$ .028 & .579 $\pm$ .029 & .575 $\pm$ .026 \\
\midrule
\midrule
\multirow{6}{*}{\rotatebox[origin=c]{90}{\textbf{CSDS3}}}  & .99   & \boldmath{}\textbf{.991 $\pm$ .001}\unboldmath{} & .992 $\pm$ .001 & .992 $\pm$ .001 & \boldmath{}\textbf{.991 $\pm$ .001}\unboldmath{} & .988 $\pm$ .001 & .992 $\pm$ .001 & \boldmath{}\textbf{.851 $\pm$ .003}\unboldmath{} & .850 $\pm$ .003 & .850 $\pm$ .003 & .850 $\pm$ .003 & .844 $\pm$ .003 & .844 $\pm$ .003 \\
      & .95   & \boldmath{}\textbf{.948 $\pm$ .002}\unboldmath{} & .957 $\pm$ .002 & .960 $\pm$ .002 & .955 $\pm$ .002 & .942 $\pm$ .002 & .951 $\pm$ .002 & \boldmath{}\textbf{.858 $\pm$ .003}\unboldmath{} & .852 $\pm$ .003 & .855 $\pm$ .003 & .853 $\pm$ .003 & .851 $\pm$ .003 & .849 $\pm$ .003 \\
      & .90   & \boldmath{}\textbf{.901 $\pm$ .002}\unboldmath{} & .909 $\pm$ .002 & .909 $\pm$ .002 & .911 $\pm$ .002 & .913 $\pm$ .002 & .902 $\pm$ .002 & \boldmath{}\textbf{.865 $\pm$ .003}\unboldmath{} & .855 $\pm$ .003 & .863 $\pm$ .003 & .856 $\pm$ .004 & .850 $\pm$ .004 & .855 $\pm$ .004 \\
      & .85   & \boldmath{}\textbf{.849 $\pm$ .003}\unboldmath{} & .859 $\pm$ .003 & .856 $\pm$ .003 & .865 $\pm$ .003 & .867 $\pm$ .003 & .856 $\pm$ .003 & \boldmath{}\textbf{.873 $\pm$ .003}\unboldmath{} & .857 $\pm$ .004 & .871 $\pm$ .003 & .858 $\pm$ .004 & .851 $\pm$ .004 & .858 $\pm$ .004 \\
      & .80   & \boldmath{}\textbf{.805 $\pm$ .003}\unboldmath{} & .808 $\pm$ .003 & .811 $\pm$ .003 & .818 $\pm$ .003 & .796 $\pm$ .003 & .819 $\pm$ .003 & \boldmath{}\textbf{.880 $\pm$ .003}\unboldmath{} & .859 $\pm$ .004 & .878 $\pm$ .003 & .859 $\pm$ .004 & .859 $\pm$ .004 & .858 $\pm$ .004 \\
      & .75   & .759 $\pm$ .003 & \boldmath{}\textbf{.758 $\pm$ .003}\unboldmath{} & .765 $\pm$ .003 & .769 $\pm$ .003 & .765 $\pm$ .003 & .776 $\pm$ .003 & \boldmath{}\textbf{.887 $\pm$ .003}\unboldmath{} & .858 $\pm$ .004 & .884 $\pm$ .003 & .859 $\pm$ .004 & .843 $\pm$ .004 & .853 $\pm$ .004 \\
\midrule
\midrule
\multirow{6}{*}{\rotatebox[origin=c]{90}{\textbf{CatsVsDogs}}}  & .99   & .992 $\pm$ .002 & \boldmath{}\textbf{.990 $\pm$ .002}\unboldmath{} & .991 $\pm$ .002 & .992 $\pm$ .002 & .991 $\pm$ .002 & .993 $\pm$ .002 & .984 $\pm$ .003 & \boldmath{}\textbf{.985 $\pm$ .002}\unboldmath{} & \boldmath{}\textbf{.985 $\pm$ .002}\unboldmath{} & .985 $\pm$ .002 & .979 $\pm$ .003 & \boldmath{}\textbf{.985 $\pm$ .002}\unboldmath{} \\
      & .95   & .962 $\pm$ .003 & .965 $\pm$ .003 & \boldmath{}\textbf{.954 $\pm$ .004}\unboldmath{} & .959 $\pm$ .003 & \boldmath{}\textbf{.954 $\pm$ .003}\unboldmath{} & .959 $\pm$ .003 & .986 $\pm$ .003 & .985 $\pm$ .002 & \boldmath{}\textbf{.987 $\pm$ .002}\unboldmath{} & .987 $\pm$ .003 & \boldmath{}\textbf{.987 $\pm$ .002}\unboldmath{} & .987 $\pm$ .002 \\
      & .90   & .920 $\pm$ .004 & .910 $\pm$ .004 & \boldmath{}\textbf{.907 $\pm$ .005}\unboldmath{} & .917 $\pm$ .004 & .891 $\pm$ .005 & .910 $\pm$ .005 & .987 $\pm$ .003 & .984 $\pm$ .002 & .988 $\pm$ .002 & .988 $\pm$ .003 & \boldmath{}\textbf{.990 $\pm$ .002}\unboldmath{} & .989 $\pm$ .002 \\
      & .85   & .883 $\pm$ .005 & .854 $\pm$ .005 & .858 $\pm$ .005 & .871 $\pm$ .005 & \boldmath{}\textbf{.852 $\pm$ .006}\unboldmath{} & .862 $\pm$ .005 & .987 $\pm$ .003 & .984 $\pm$ .002 & .989 $\pm$ .002 & .988 $\pm$ .003 & \boldmath{}\textbf{.990 $\pm$ .002}\unboldmath{} & .990 $\pm$ .002 \\
      & .80   & .841 $\pm$ .006 & \boldmath{}\textbf{.797 $\pm$ .006}\unboldmath{} & .805 $\pm$ .006 & .807 $\pm$ .006 & .809 $\pm$ .006 & .818 $\pm$ .006 & .987 $\pm$ .003 & .985 $\pm$ .002 & .990 $\pm$ .002 & .989 $\pm$ .003 & \boldmath{}\textbf{.991 $\pm$ .002}\unboldmath{} & .990 $\pm$ .002 \\
      & .75   & .781 $\pm$ .006 & .740 $\pm$ .007 & \boldmath{}\textbf{.749 $\pm$ .007}\unboldmath{} & .744 $\pm$ .007 & .765 $\pm$ .007 & .768 $\pm$ .007 & .987 $\pm$ .003 & .983 $\pm$ .002 & .990 $\pm$ .002 & .988 $\pm$ .003 & \boldmath{}\textbf{.992 $\pm$ .002}\unboldmath{} & .990 $\pm$ .002 \\
\midrule
\midrule
\multirow{6}{*}{\rotatebox[origin=c]{90}{\textbf{CIFAR10-CAT}}}  & .99   & 1.000$\pm$ .001 & .991 $\pm$ .001 & .989 $\pm$ .002 & .992 $\pm$ .001 & .992 $\pm$ .001 & \boldmath{}\textbf{.990 $\pm$ .002}\unboldmath{} & .890 $\pm$ .009 & .915 $\pm$ .007 & .909 $\pm$ .008 & \boldmath{}\textbf{.922 $\pm$ .007}\unboldmath{} & .896 $\pm$ .008 & .893 $\pm$ .008 \\
      & .95   & .998 $\pm$ .001 & \boldmath{}\textbf{.950 $\pm$ .003}\unboldmath{} & .948 $\pm$ .003 & .953 $\pm$ .003 & .956 $\pm$ .003 & .952 $\pm$ .003 & .889 $\pm$ .009 & .895 $\pm$ .009 & \boldmath{}\textbf{.911 $\pm$ .007}\unboldmath{} & .906 $\pm$ .009 & .902 $\pm$ .008 & .863 $\pm$ .011 \\
      & .90   & .994 $\pm$ .001 & \boldmath{}\textbf{.899 $\pm$ .004}\unboldmath{} & .898 $\pm$ .003 & .903 $\pm$ .003 & .904 $\pm$ .003 & .913 $\pm$ .003 & .889 $\pm$ .009 & .896 $\pm$ .009 & \boldmath{}\textbf{.913 $\pm$ .007}\unboldmath{} & .851 $\pm$ .013 & .899 $\pm$ .010 & .760 $\pm$ .017 \\
      & .85   & .987 $\pm$ .002 & .849 $\pm$ .004 & \boldmath{}\textbf{.851 $\pm$ .004}\unboldmath{} & .886 $\pm$ .004 & .863 $\pm$ .004 & .854 $\pm$ .004 & .888 $\pm$ .009 & .898 $\pm$ .009 & \boldmath{}\textbf{.916 $\pm$ .007}\unboldmath{} & .815 $\pm$ .016 & .676 $\pm$ .022 & .574 $\pm$ .024 \\
      & .80   & .969 $\pm$ .002 & .796 $\pm$ .005 & \boldmath{}\textbf{.799 $\pm$ .004}\unboldmath{} & .882 $\pm$ .004 & .806 $\pm$ .005 & .803 $\pm$ .005 & .885 $\pm$ .009 & .906 $\pm$ .008 & \boldmath{}\textbf{.919 $\pm$ .007}\unboldmath{} & .804 $\pm$ .017 & .359 $\pm$ .026 & .572 $\pm$ .024 \\
      & .75   & .945 $\pm$ .003 & \boldmath{}\textbf{.749 $\pm$ .005}\unboldmath{} & \boldmath{}\textbf{.749 $\pm$ .005}\unboldmath{} & .879 $\pm$ .004 & .763 $\pm$ .005 & .759 $\pm$ .005 & .870 $\pm$ .011 & .905 $\pm$ .008 & \boldmath{}\textbf{.921 $\pm$ .007}\unboldmath{} & .797 $\pm$ .018 & .408 $\pm$ .028 & .575 $\pm$ .025 \\
\midrule
\midrule
      & \#    & 21/54 & 12/54 & 8/54  & 15/54 & 8/54  & 1/54  & 28/54 & 4/54  & 29/54 & 6/54  & 5/54  & 2/54 \\
\midrule
      & $V$     & .026  $\pm$ .046 & .021  $\pm$ .029 & .013  $\pm$ .026 & .028  $\pm$ .036 & .029  $\pm$ .041 & .171  $\pm$ .251 &       &       &       &       &       &  \\
\end{tabular}%

}
    \label{tab:metrics}

\end{table*}
%
These two variants allow us to evaluate the relative contributions of the two key elements of our approach: the estimation of AUC-specific bounds (Prop. \ref{pro:final}), and the cross fitting approach paired with the second-order quantile estimation strategy (Thm. \ref{teo1}).
We also consider two state-of-the-art methods for selective classification, namely the SelectiveNet  \citep{DBLP:conf/icml/GeifmanE19} (\textsc{SelNet}), and the Self Adaptive Training  \citep{DBLP:conf/nips/Huang0020} (\textsc{SAT}) methods.
The approach from \cite{DBLP:conf/aaai/ShenYG20} is not included in the baselines since it rejects the comparison of \textit{pairs} of instances, while we focus on the rejection of predictions on instances in isolation.

For tabular data, \textsc{SelNet} and \textsc{SAT} are built using a ResNet structure \citep{DBLP:conf/nips/GorishniyRKB21}, while \textsc{AUCross}, \textsc{PlugIn}, \textsc{PlugInAUC} and \textsc{SCross} are based on a LightGBM classifier. We report in the Appendix results for other base classifiers.
For \textit{CatsVsDogs}, a VGG architecture \citep{DBLP:journals/corr/SimonyanZ14a} is used for all the methods. 
For all the DNN approaches, we set 300 epochs in training, Stochastic Gradient Descent as an optimizer, a learning rate of $.1$ decreased by a factor $.5$ every $25$ epochs, as in the original papers. All the parameters of base classifiers are left as the default ones.
See the Appendix for details.
%

%

Table \ref{tab:metrics} shows performance metrics (mean $\pm$ stdev) evaluated on 1,000 bootstrap runs of the test set: the empirical coverage 
$\hat{\phi}(g|S_n)$ and the empirical AUC $\widehat{AUC}(h,g|S_n)$. 
Resorting to bootstrap allows for calculating standard errors of the metrics\footnote{Other approaches include confidence intervals for AUC  \citep{10.2307/2531595,DBLP:conf/nips/CortesM04,Goncalves2014} or direct estimation of its variance \citep{Cleves2002}.} \citep{rajkomar2018scalable}. 
Moreover, it allows for quantifying the generalizability of the methods to perturbations of  the test set.
Additional results on the selective accuracy and the positive rate\footnote{Positive rate in the accepted region, compared to the positive rate in the overall test set, measures fairness of the selection function w.r.t. the class labels.} metrics are reported in the Appendix. 
Results in Table \ref{tab:metrics} show that \textsc{AUCross} achieves an empirical coverage close to the target $c$ in most of the cases for tabular data. The most significant violation occurs for \textit{LendingClub}, where also all the other methods fail to reach the target coverage. In particular, \textsc{SAT} performs poorly on such a dataset. The other largest violation occurs for \textit{CIFAR-10-cat}. Here, also the other cross-fitting-based algorithm (\textsc{SCross}) fails to reach the target one. Interestingly, the runner-up method w.r.t. coverage is \textsc{PlugInAUC}, which has the smallest average violation. A striking case is the \textit{CSDS2} dataset, which is the most imbalanced one with a positive rate of 1.8\%. Here, both \textsc{AUCross} and \textsc{PlugInAUC} have small violations, while all other approaches have larger ones. We argue that this is due to the bounds determined by 
Prop. \ref{pro:positive1} and Prop. \ref{pro:negative1}, which are specific for positives and negatives, respectively, while the other methods calibrate the selection function independently of the class labels.
Regarding the empirical AUC, \textsc{AUCross} and \textsc{PlugInAUC} outperform the other methods in most cases. For unbalanced datasets such as for \textit{GiveMe}, \textit{CSDS1} and \textit{CSDS2}, the empirical AUC drops for smaller coverages, contrarily to what is expected. Such behaviour does not occur for those methods that estimate bounds specifically for the AUC, i.e. \textsc{AUCross} and \textsc{PlugInAUC}.
We show in the Appendix that a trade-off exists in terms of selective accuracy, as both \textsc{AUCross} and \textsc{PlugInAUC} are outperformed concerning the accuracy over the accepted region. This is not surprising since they are optimizing the AUC metric, and it is well-known that accuracy and AUC cannot be jointly optimized \citep{DBLP:conf/nips/CortesM03}. The results in the Appendix also show that optimizing accuracy compromises the positive rate, while our approach is fairer in this sense. 

Finally, consider the running time performances (see the Appendix for details). \textsc{AUCross} and \textsc{SCross} require $K$ executions of the base classifier (loop at lines 3-6 of Alg. \ref{alg:fit}), while all other baselines train a single (selective) classifier. On tabular data, using efficient base classifiers results in a reasonably low total running time. Those two strategies are instead computationally expensive for image data, where DNN base classifiers are adopted. However, we observe that \textsc{PlugInAUC} exhibits performances comparable to  \textsc{SelNet} and \textsc{SAT} for \textit{CatsVsDogs}, and it outperforms them in the case of \textit{CIFAR-10-cat}. In summary, we recommend using \textsc{AUCross} for tabular datasets and \textsc{PlugInAUC} for image datasets.


\section{CONCLUSIONS}
\label{sec:end}

Selective classification 
can help prevent poor or even harmful decisions by abstaining from making an output, possibly demanding the decision to a human. We have extended the selective classification framework to a widely used classifier evaluation metric, the Area Under the ROC Curve. Through an analytic characterization, we devised methods for computing an estimator of the best score-bounded selection function. These methods are effective and outperform existing approaches designed for optimizing accuracy.
%

Limitations that require future work include the following. 
First, the approach should be extended to multi-class classification, for which the notion of AUC has been considered, e.g., the Volume Under the Surface \citep{DBLP:conf/ecml/FerriHS03}, the average of pairwise binary AUCs (the M metric) \citep{DBLP:journals/ml/HandT01}, or the AUC-$\mu$ \citep{DBLP:conf/icml/KleimanP19}. 
Second, since selective classification might amplify unfair decisions \citep{DBLP:conf/iclr/JonesSKKL21}, we intend to study how to account for fairness metrics in the context of AUC-based selective classification.
Third, as suggested by our experimental results, we could better determine the bounds for a target coverage $c$, reconsidering the choice of centring the rejection area in the midpoint of the bounds of Prop. \ref{pro:final}. As shown in Table \ref{tab:auc_dev}, this is especially relevant for large $c$'s. Lastly, 
AUC-based selective classification can be extended to the metric of weighted AUC \citep{DBLP:journals/prl/Fawcett06a,Keilwagen14}, where instances are weighted by importance.

\subsubsection*{Reproducibility} Data and source code can be downloaded from \href{https://github.com/andrepugni/AUCbasedSelectiveClassification}{https://\-github.com/andrepugni/AUCbasedSelectiveClassification}.

\subsubsection*{Acknowledgements}

Work  supported by the XAI project, funded by the European Union's Horizon 2020 Excellent Science European Research Council (ERC) programme under g.a. No. 834756. 
Views and opinions expressed are however those of the authors only and do not necessarily reflect those of the EU. Neither the EU nor the granting authority can be held responsible for them.

\bibliography{biblio.bib}

\begin{thebibliography}{}

\bibitem[Barddal et~al., 2020]{DBLP:journals/eswa/BarddalLEL20}
Barddal, J.~P., Loezer, L., Enembreck, F., and Lanzuolo, R. (2020).
\newblock Lessons learned from data stream classification applied to credit
  scoring.
\newblock {\em Expert Syst. Appl.}, 162:113899.

\bibitem[Chernozhukov et~al., 2018]{chernozukov2018}
Chernozhukov, V., Chetverikov, D., Demirer, M., Duflo, E., Hansen, C., Newey,
  W., and Robins, J. (2018).
\newblock {Double/debiased machine learning for treatment and structural
  parameters}.
\newblock {\em The Econometrics Journal}, 21(1):C1--C68.

\bibitem[Chow, 1970]{DBLP:journals/tit/Chow70}
Chow, C.~K. (1970).
\newblock On optimum recognition error and reject tradeoff.
\newblock {\em {IEEE} Trans. Inf. Theory}, 16(1):41--46.

\bibitem[Cleves, 2002]{Cleves2002}
Cleves, M.~A. (2002).
\newblock Comparative assessment of three common algorithms for estimating the
  variance of the area under the nonparametric receiver operating
  characteristic curve.
\newblock {\em The Stata Journal}, 2:280–289.

\bibitem[Cortes et~al., 2016]{DBLP:conf/nips/CortesDM16}
Cortes, C., DeSalvo, G., and Mohri, M. (2016).
\newblock Boosting with abstention.
\newblock In {\em {NIPS}}, pages 1660--1668.

\bibitem[Cortes and Mohri, 2003]{DBLP:conf/nips/CortesM03}
Cortes, C. and Mohri, M. (2003).
\newblock {AUC} optimization vs. error rate minimization.
\newblock In {\em {NIPS}}, pages 313--320. {MIT} Press.

\bibitem[Cortes and Mohri, 2004]{DBLP:conf/nips/CortesM04}
Cortes, C. and Mohri, M. (2004).
\newblock Confidence intervals for the area under the {ROC} curve.
\newblock In {\em {NIPS}}, pages 305--312.

\bibitem[DeLong et~al., 1988]{10.2307/2531595}
DeLong, E.~R., DeLong, D.~M., and Clarke-Pearson, D.~L. (1988).
\newblock Comparing the areas under two or more correlated receiver operating
  characteristic curves: A nonparametric approach.
\newblock {\em Biometrics}, 44(3):837--845.

\bibitem[Dua and Graff, 2017]{Dua:2019}
Dua, D. and Graff, C. (2017).
\newblock {UCI} machine learning repository.

\bibitem[El{-}Yaniv and Wiener, 2010]{DBLP:journals/jmlr/El-YanivW10}
El{-}Yaniv, R. and Wiener, Y. (2010).
\newblock On the foundations of noise-free selective classification.
\newblock {\em J. Mach. Learn. Res.}, 11:1605--1641.

\bibitem[Engelmann, 2011]{Engelmann2011}
Engelmann, B. (2011).
\newblock Measures of a rating’s discriminative power: {A}pplications and
  limitations.
\newblock In Engelmann, B. and Rauhmeier, R., editors, {\em The {B}asel {II}
  Risk Parameters}, pages 269--291. Springer, 2 edition.

\bibitem[Fawcett, 2006a]{DBLP:journals/prl/Fawcett06}
Fawcett, T. (2006a).
\newblock An introduction to {ROC} analysis.
\newblock {\em Pattern Recognit. Lett.}, 27(8):861--874.

\bibitem[Fawcett, 2006b]{DBLP:journals/prl/Fawcett06a}
Fawcett, T. (2006b).
\newblock {ROC} graphs with instance-varying costs.
\newblock {\em Pattern Recognit. Lett.}, 27(8):882--891.

\bibitem[Ferri et~al., 2003]{DBLP:conf/ecml/FerriHS03}
Ferri, C., Hern{\'{a}}ndez{-}Orallo, J., and Salido, M.~A. (2003).
\newblock Volume under the {ROC} surface for multi-class problems.
\newblock In {\em {ECML}}, volume 2837 of {\em Lecture Notes in Computer
  Science}, pages 108--120. Springer.

\bibitem[Franc and Pr\r{u}\v{s}a, 2019]{DBLP:conf/icml/FrancP19}
Franc, V. and Pr\r{u}\v{s}a, D. (2019).
\newblock On discriminative learning of prediction uncertainty.
\newblock In {\em {ICML}}, volume~97 of {\em Proceedings of Machine Learning
  Research}, pages 1963--1971. {PMLR}.

\bibitem[Fumera and Roli, 2002]{DBLP:conf/svm/FumeraR02}
Fumera, G. and Roli, F. (2002).
\newblock Support vector machines with embedded reject option.
\newblock In {\em {SVM}}, volume 2388 of {\em Lecture Notes in Computer
  Science}, pages 68--82. Springer.

\bibitem[Geifman and El{-}Yaniv, 2017]{DBLP:conf/nips/GeifmanE17}
Geifman, Y. and El{-}Yaniv, R. (2017).
\newblock Selective classification for deep neural networks.
\newblock In {\em {NIPS}}, pages 4878--4887.

\bibitem[Geifman and El{-}Yaniv, 2019]{DBLP:conf/icml/GeifmanE19}
Geifman, Y. and El{-}Yaniv, R. (2019).
\newblock Selectivenet: {A} deep neural network with an integrated reject
  option.
\newblock In {\em {ICML}}, volume~97 of {\em Proceedings of Machine Learning
  Research}, pages 2151--2159. {PMLR}.

\bibitem[Goncalves et~al., 2014]{Goncalves2014}
Goncalves, L., Subtil, A., Oliveira, M.~R., and {de Zea Bermudez}, P. (2014).
\newblock {ROC} curve estimation: {A}n overview.
\newblock {\em Revstat - Statistical Journal}, 12:1--20.

\bibitem[Gorishniy et~al., 2021]{DBLP:conf/nips/GorishniyRKB21}
Gorishniy, Y., Rubachev, I., Khrulkov, V., and Babenko, A. (2021).
\newblock Revisiting deep learning models for tabular data.
\newblock In {\em NeurIPS}, pages 18932--18943.

\bibitem[Hand and Till, 2001]{DBLP:journals/ml/HandT01}
Hand, D.~J. and Till, R.~J. (2001).
\newblock A simple generalisation of the area under the {ROC} curve for
  multiple class classification problems.
\newblock {\em Mach. Learn.}, 45(2):171--186.

\bibitem[Hendrickx et~al., 2021]{Hendrickx2021}
Hendrickx, K., Perini, L., der Plas, D.~V., Meert, W., and Davis, J. (2021).
\newblock Machine learning with a reject option: {A} survey.
\newblock {\em CoRR}, abs/2107.11277.

\bibitem[Herbei and Wegkamp, 2006]{Herbei06}
Herbei, R. and Wegkamp, M.~H. (2006).
\newblock Classification with reject option.
\newblock {\em Can. J. Stat.}, 34(4):709--721.

\bibitem[Huang and Ling, 2005]{DBLP:journals/tkde/HuangL05}
Huang, J. and Ling, C.~X. (2005).
\newblock Using {AUC} and accuracy in evaluating learning algorithms.
\newblock {\em {IEEE} Trans. Knowl. Data Eng.}, 17(3):299--310.

\bibitem[Huang et~al., 2020]{DBLP:conf/nips/Huang0020}
Huang, L., Zhang, C., and Zhang, H. (2020).
\newblock Self-adaptive training: beyond empirical risk minimization.
\newblock In {\em NeurIPS}.

\bibitem[Jones et~al., 2021]{DBLP:conf/iclr/JonesSKKL21}
Jones, E., Sagawa, S., Koh, P.~W., Kumar, A., and Liang, P. (2021).
\newblock Selective classification can magnify disparities across groups.
\newblock In {\em {ICLR}}. OpenReview.net.

\bibitem[Ke et~al., 2017]{DBLP:conf/nips/KeMFWCMYL17}
Ke, G., Meng, Q., Finley, T., Wang, T., Chen, W., Ma, W., Ye, Q., and Liu, T.
  (2017).
\newblock Light{GBM}: {A} highly efficient gradient boosting decision tree.
\newblock In {\em {NIPS}}, pages 3146--3154.

\bibitem[Keilwagen et~al., 2014]{Keilwagen14}
Keilwagen, J., Grosse, I., and Grau, J. (2014).
\newblock Area under precision-recall curves for weighted and unweighted data.
\newblock {\em PLoS ONE}, 9(3):e92209.

\bibitem[Kleiman and Page, 2019]{DBLP:conf/icml/KleimanP19}
Kleiman, R. and Page, D. (2019).
\newblock {AUC}{\(\mu\)}: {A} performance metric for multi-class machine
  learning models.
\newblock In {\em {ICML}}, volume~97 of {\em Proceedings of Machine Learning
  Research}, pages 3439--3447. {PMLR}.

\bibitem[Knight and Bassett, 2003]{Knight2003}
Knight, K. and Bassett, G.~W. (2003).
\newblock Second order improvements of sample quantiles using subsamples.
\newblock \textit{Unpublished manuscript},
  \href{http://www.utstat.toronto.edu/keith/papers/subsample.pdf}{http://www.utstat.toronto.edu/keith/papers/subsample.pdf}.

\bibitem[Krizhevsky, 2009]{Krizhevsky09learningmultiple}
Krizhevsky, A. (2009).
\newblock Learning multiple layers of features from tiny images.
\newblock Technical report, University of Toronto.

\bibitem[Liu et~al., 2019]{DBLP:conf/nips/LiuWLSMU19}
Liu, Z., Wang, Z., Liang, P.~P., Salakhutdinov, R., Morency, L., and Ueda, M.
  (2019).
\newblock Deep gamblers: Learning to abstain with portfolio theory.
\newblock In {\em NeurIPS}, pages 10622--10632.

\bibitem[Pietraszek, 2005]{DBLP:conf/icml/Pietraszek05}
Pietraszek, T. (2005).
\newblock Optimizing abstaining classifiers using {ROC} analysis.
\newblock In {\em {ICML}}, volume 119 of {\em {ACM} International Conference
  Proceeding Series}, pages 665--672. {ACM}.

\bibitem[Pugnana and Ruggieri, 2023]{PugnanaRuggieri2023}
Pugnana, A. and Ruggieri, S. (2023).
\newblock A model-agnostic heuristics for selective classification.
\newblock In {\em {AAAI}}. {AAAI} Press.

\bibitem[Rajkomar et~al., 2018]{rajkomar2018scalable}
Rajkomar, A., Oren, E., Chen, K., Dai, A.~M., Hajaj, N., Hardt, M., Liu, P.~J.,
  Liu, X., Marcus, J., Sun, M., et~al. (2018).
\newblock Scalable and accurate deep learning with electronic health records.
\newblock {\em NPJ Digital Medicine}, 1(1):1--10.

\bibitem[Shen et~al., 2020]{DBLP:conf/aaai/ShenYG20}
Shen, S., Yang, B., and Gao, W. (2020).
\newblock {AUC} optimization with a reject option.
\newblock In {\em {AAAI}}, pages 5684--5691. {AAAI} Press.

\bibitem[Simonyan and Zisserman, 2015]{DBLP:journals/corr/SimonyanZ14a}
Simonyan, K. and Zisserman, A. (2015).
\newblock Very deep convolutional networks for large-scale image recognition.
\newblock In {\em {ICLR}}.

\bibitem[Tortorella, 2005]{DBLP:journals/prl/Tortorella05}
Tortorella, F. (2005).
\newblock A {ROC}-based reject rule for dichotomizers.
\newblock {\em Pattern Recognit. Lett.}, 26(2):167--180.

\bibitem[Vovk, 2013]{DBLP:journals/ml/Vovk13}
Vovk, V. (2013).
\newblock Conditional validity of inductive conformal predictors.
\newblock {\em Mach. Learn.}, 92(2-3):349--376.

\bibitem[Wu and Lee, 2014]{Wu14}
Wu, Y.-C. and Lee, W.-C. (2014).
\newblock Alternative performance measures for prediction models.
\newblock {\em PLoS ONE}, 9(3):e91249.

\bibitem[Yang and Ying, 2023]{DBLP:journals/corr/YangYing2022}
Yang, T. and Ying, Y. (2023).
\newblock {AUC} maximization in the era of big data and {AI:} {A} survey.
\newblock {\em {ACM} Comput. Surv.}, 55(8):172:1--172:37.

\bibitem[Yeh and Lien, 2009]{DBLP:journals/eswa/YehL09a}
Yeh, I. and Lien, C. (2009).
\newblock The comparisons of data mining techniques for the predictive accuracy
  of probability of default of credit card clients.
\newblock {\em Expert Syst. Appl.}, 36(2):2473--2480.

\end{thebibliography}

\onecolumn
\aistatstitle{Supplementary Materials for AUC-based selective classification}
\appendix
\section{APPENDIX}
\subsection{Proofs}

\textbf{Proposition 1.}\\[-5ex]
\begin{proof} The missing part of the proof consists of the following equivalence:
\begin{eqnarray*}
\bar{G} > G & \mbox{iff} & \bar{A} (n-1) > A n \\
& \mbox{iff} & \{ \mbox{\ by}\ (\ref{eq:gini_eq})\ \}\\
& \mbox{iff} & (( A  + \frac{1}{2} - \left[\frac{(r'(\mathbf{x}_i)-1) + (t(\mathbf{x}_i)-1) + \nicefrac{1}{2}}{n^+ n} \right]) \frac{n^+ n}{(n^+-1) (n-1)} - \frac{1}{2})(n-1)  > A n \\
& \mbox{iff} &  ( A  + \frac{1}{2} - \frac{(r'(\mathbf{x}_i)+t(\mathbf{x}_i)-1.5)}{n^+ n} ) \frac{n^+ n}{(n^+-1)} - \frac{(n-1)}{2} > A n \\
& \mbox{iff} &  A \frac{n^+ n}{(n^+-1)} + \frac{n^+ n}{2(n^+-1)} - \frac{(r'(\mathbf{x}_i)+t(\mathbf{x}_i)-1.5)}{(n^+-1)}  - \frac{(n-1)}{2} > A n  \\
& \mbox{iff} &  A \frac{n}{(n^+-1)} + \frac{n^+ n}{2(n^+-1)} - \frac{(r'(\mathbf{x}_i)+t(\mathbf{x}_i)-1.5)}{(n^+-1)}  - \frac{(n-1)}{2} > 0  \\
& \mbox{iff} &  A \frac{n}{(n^+-1)} + \frac{n^+ n - (n-1)(n^+-1)}{2(n^+-1)} - \frac{(r'(\mathbf{x}_i)+t(\mathbf{x}_i)-1.5)}{(n^+-1)} > 0 \\
& \mbox{iff} & A n + \frac{ n + n^+ -1}{2} - (r'(\mathbf{x}_i)+t(\mathbf{x}_i)-1.5) > 0 \\
& \mbox{iff} &  A n + \frac{ n + n^+}{2} -t(\mathbf{x}_i) > r'(\mathbf{x}_i)-1 \\
& \mbox{iff} &  A n + \frac{ n + n^+}{2} - t(\mathbf{x}_i) \geq r'(\mathbf{x}_i) \\
& \mbox{iff} &  A n + \frac{ n - n^+}{2} + (n^+ - t(\mathbf{x}_i)) \geq r'(\mathbf{x}_i)\\
& \mbox{iff} &  \frac{r'(\mathbf{x}_i)}{n} \leq A + \frac{ p^-}{2} + \frac{(n^+ - t(\mathbf{x}_i))}{n}\\
\end{eqnarray*}
\end{proof}

\newpage
\textbf{Proposition 2.}\\[-5ex]
\begin{proof}
We can equivalently show the result for the Gini coefficient. Let $G = A/(A+B)$, and let $\bar{G} = \bar{A}/(\bar{A}+\bar{B})$ be the Gini coefficient after removing (a.k.a., abstaining on) one negative instance $\mathbf{x}_i$.
We have $\bar{A}+\bar{B} = (n^--1)/(2(n-1))$. 
Then $\bar{G} > G$ iff 
\begin{equation}
    \label{eq:gg2}
    \bar{A} (n-1) n^- > A n (n^--1)
\end{equation}
As in Proposition 1, we can link the area $\bar{A}$ after the removal of the negative instance $\mathbf{x}_i$ to the original area $A$ as follows:
\[ \bar{A} + \frac{1}{2} =  ( A  + \frac{1}{2} - \frac{t(\mathbf{x}_i)}{n^+ n} ) \frac{n}{n-1} \]
As for the positive case, this can be intuitively understood by looking at Figure \ref{fig:CAP_remove_neg}, where the grey area highlights the loss from the original CAP plot after we remove the negative instance. Then we rescale to account for the new number of instances, as in Figure \ref{fig:CAP_after_remove_neg}.
Then:
\begin{eqnarray*}
\bar{G} > G & \mbox{iff} &(( A  + \frac{1}{2} - \frac{t(\mathbf{x}_i)}{n^+ n} ) \frac{n}{n-1} - \frac{1}{2})(n-1) n^- > A n (n^--1)\\
& \mbox{iff} & (A  + \frac{1}{2} - \frac{t(\mathbf{x}_i)}{n^+ n} ) n n^- - \frac{(n-1) n^-}{2} > A n (n^--1) \\
& \mbox{iff} & An n^-  + \frac{n n^-}{2} - \frac{t(\mathbf{x}_i)}{n^+} n^- - \frac{(n-1) n^-}{2} > A n n^- -An \\
& \mbox{iff} & An + \frac{n n^-}{2} - \frac{(n-1) n^-}{2} > \frac{t(\mathbf{x}_i)}{n^+} n^- \\
& \mbox{iff} & A\frac{n}{n^-} + \frac{1}{2} > \frac{t(\mathbf{x}_i)}{n^+}  \\
& \mbox{iff} & A\frac{n}{n^-} + \frac{1}{2} - \frac{1}{n^+} \geq \frac{t(\mathbf{x}_i)}{n^+}  
\end{eqnarray*}
Since $\widehat{AUC}(h, g' | S_{n}) = (G+1)/2 = (2 A/p^- + 1)/2  = A \nicefrac{n}{n^-} + \nicefrac{1}{2}$, and 
\begin{equation}
    \frac{t(\mathbf{x}_i)}{n^+} \leq \widehat{AUC}(h, g | S_{n}) - \frac{1}{n^+}
    \label{eq:pro2cond2}
\end{equation}
is assumed to hold, 
we have the $\bar{G} > G$ (a.k.a., the conclusion of Proposition 2) after removing \textit{one} negative instance. Moreover, if (\ref{eq:pro2cond2}) holds for a second negative instance $\mathbf{x}_i$, then:
\[     \frac{t(\mathbf{x}_i)}{n^+} \leq \widehat{AUC}(h, g | S_{n}) - \frac{1}{n^+} \leq \widehat{AUC}(h, g' | S_{n}) - \frac{1}{n^+} \]
where $g'$ abstains on the first negative instance. In fact, $\bar{G} > G$ implies that $\widehat{AUC}(h, g | S_{n}) < \widehat{AUC}(h, g' | S_{n}) $. Therefore, we can iterate the conclusion that the Gini coefficient (or, equivalently, empirical AUC) increases by abstaining on any number of negative instances that satisfy the assumption of Proposition 2.
\end{proof}

\begin{figure*}
\begin{minipage}{.47\textwidth}
    \centering
    \includegraphics[scale=.2]{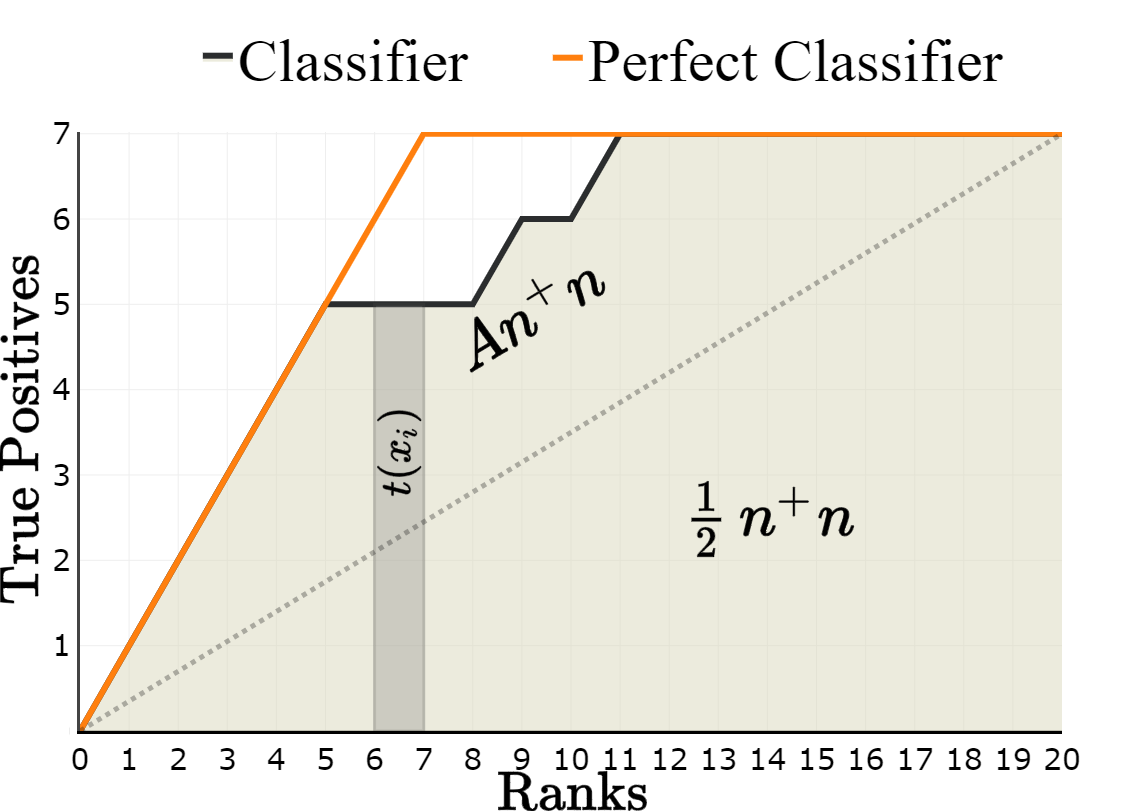}\\[1ex]
    \captionof{figure}{\small{CAP plot before removing a negative instance.}}
    \label{fig:CAP_remove_neg}
\end{minipage}\hfill
\begin{minipage}{.47\textwidth}    \centering
    \includegraphics[scale=.2]{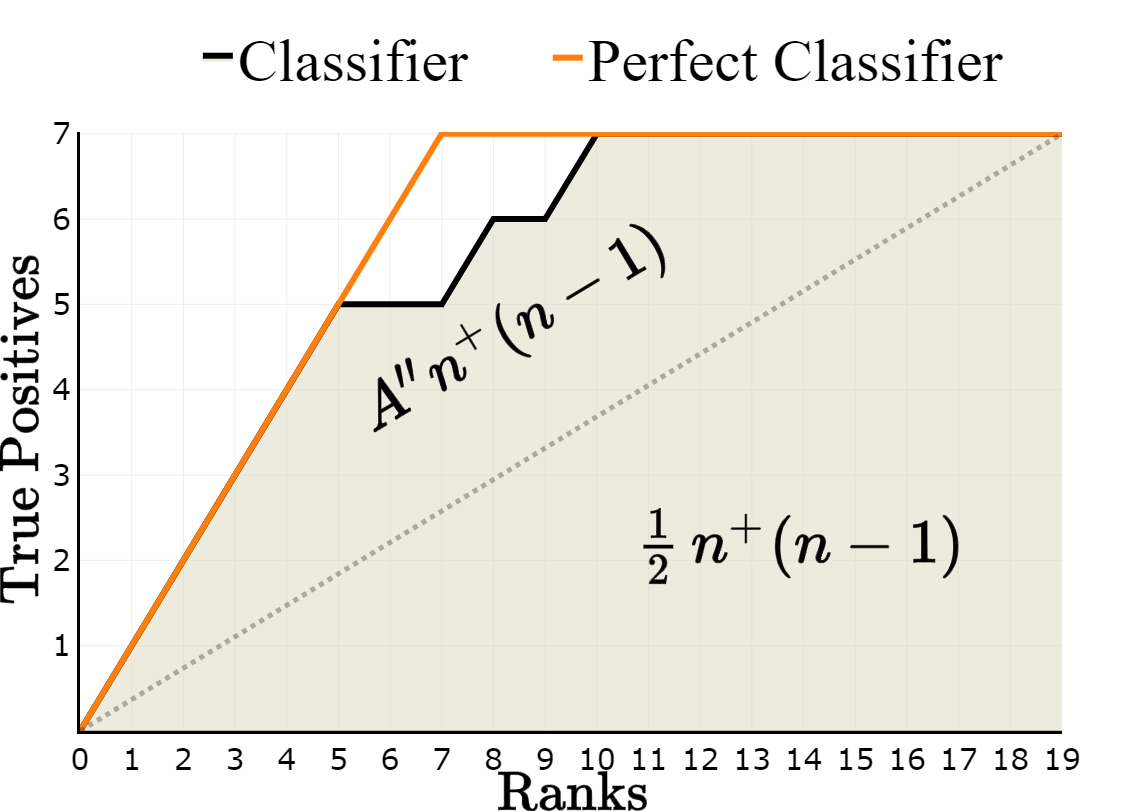}\\[1ex]
    \captionof{figure}{\small{CAP plot after removing a negative instance.}}
    \label{fig:CAP_after_remove_neg}
\end{minipage}
\end{figure*}

\subsection{Datasets description}

\textit{\href{https://archive.ics.uci.edu/ml/datasets/adult}{Adult}} is an extract from the 1994 US Census, with class the binarization of income into $\leq 50K$ and $>50K$. The final training set contains 30,162 instances and 55 features after one-hot encoding. The test set size is 15,060.

\textit{\href{https://www.kaggle.com/wordsforthewise/lending-club}{LendingClub}} regards repaying a loan obtained by an online platform. We used a temporal split to build the final training set (1,364,697 instances) and the test set (445,912 instances). The dataset has 65 features.

The \textit{\href{https://www.kaggle.com/c/GiveMeSomeCredit}{GiveMe}} dataset aims at predicting the financial distress of a borrower within two years. The training and the test set were obtained by stratified random sampling, and they contain 12 features, and 112,500 and 37,500 instances respectively.

\textit{\href{https://archive.ics.uci.edu/ml/datasets/default+of+credit+card+clients}{UCICredit}} regards credit card defaults in Taiwan.  
This dataset from \cite{Dua:2019} 
concerns whether or not a credit card holder will default in the next six months \citep{DBLP:journals/eswa/YehL09a}. Training and test sets were obtained by stratified random sampling. The training set includes 22,500 instances (7,500 for the test set) and 23 features.

\textit{CSDS1}, \textit{CSDS2} and \textit{CSDS3} - from \cite{DBLP:journals/eswa/BarddalLEL20} - regard predicting defaults in repaying a loan: within six months for \textit{CSDS1} (data span over 15 months), within 2 months for \textit{CSDS2} (data span over 25 months), and within three months for \textit{CSDS3} (data span over 16 months). Training set and test set were divided through a timestamp variable.
The training set of \textit{CSDS1} consists of 230,409 instances and 155 features (test set size is 76,939). For \textit{CSDS2}, the training set contains 37,100 instances and 35 features (test set size is 12,533). For \textit{CSDS3}, the training set contains 71,177 instances and 144 features (test set size 23,288).

The \textit{\href{https://www.kaggle.com/competitions/dogs-vs-cats}{CatsVsDogs}} dataset is a collection of cats and dogs images. The task here is to distinguish between the two species. The training and test sets were obtained as described in \cite{DBLP:conf/nips/LiuWLSMU19}. The training set contains 20,000 images, each one of 64x64 pixels. The test set consists of 5,000 images.

Finally we considered the image dataset \textit{\href{https://www.cs.toronto.edu/~kriz/cifar.html}{CIFAR-10-cat}} from \cite{Krizhevsky09learningmultiple}. For each of the 10 class labels, we have 5,000 32x32 images in the training set and 1,000 in the test set. We transformed it into a binary classification task by using the \textit{cat} label as the positive class.

\subsection{Models}

\paragraph{\textsc{AUCross} and \textsc{PlugInAUC}.}
In the main paper we used as base classifier a \textsc{LightGBM} classifier with default parameters. In Tables \ref{tab:class_cov}-\ref{tab:class_pos} we provide also results for a Logistic Regression and a Random Forest Classifier from $\textit{sklearn}$ package with default parameters, a ResNet implementation with default parameters from \textit{rtdl} \citep{DBLP:conf/nips/GorishniyRKB21} and a XGBoost from $\textit{xgboost}$ package with default parameters.

\paragraph{\textsc{PlugIn} and \textsc{SCross}.}
As for \textsc{AUCross}, we considered a \textsc{LightGBM} classifier with default parameters.

\paragraph{\textsc{SelNet}.}
\begin{figure}
    \centering
    \includegraphics[scale=.2]{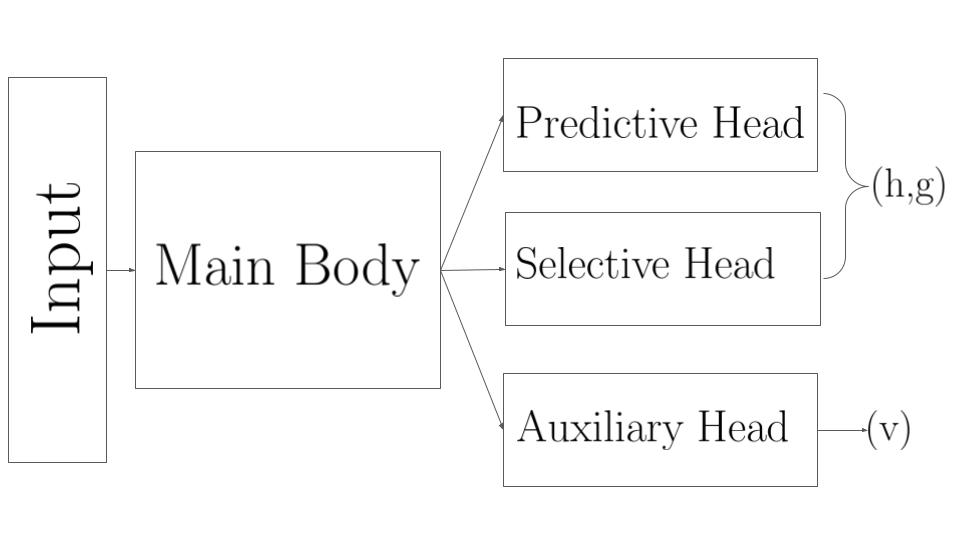}
    \caption{Scheme for Selective Net architecture.}
    \label{fig:SelNetstr}
\end{figure}
Selective Net is a selective model $(h,g)$ that optimizes at the same time both $h(\mathbf{x})$ and $g(\mathbf{x})$.
Its schema is summarized in Figure \ref{fig:SelNetstr}.
The architecture is based on four distinct parts:
the main body, the predictive head, the selective head and the auxiliary head.
The input is initially processed by the main body: it consists of deep layers that are shared by all the three heads. Any type of architecture can be used in this part (e.g., convolutional layers, linear layers, recurrent layers ecc.). The predictive head provides the final prediction $h(\mathbf{x})$; the selective head outputs the selective function $g(\mathbf{x})$; the auxiliary head is used to ensure that the main body part is exposed to all training instances, i.e., it is used to avoid that \textsc{SelNet} overfits on the accepted instances.
We used $\textit{pytorch}$ to model \textsc{SelNet}. For tabular datasets, we built the main body part using ResNet with default parameters provided by $\textit{rtdl}$ \citep{DBLP:conf/nips/GorishniyRKB21}, as authors claim that ResNet is a valid baseline on tabular data.
For images, we used as the main body the VGG16 architecture \citep{DBLP:journals/corr/SimonyanZ14a} as done in the original paper \citep{DBLP:conf/icml/GeifmanE19}.
We then added the classification head, the selection head and the auxiliary head following \textsc{SelNet} paper.
For CSDS1, Lending and GiveMe the prediction head and the auxiliary head were made by a first linear layer with 512 nodes followed by a batch normalization layer and ReLu activation; a second layer with 256 nodes, batch normalization and ReLu activation; a final dense layer with 128 nodes ending with two nodes and softmax activation.
For the other datasets the classification and auxiliary heads were made by a single linear layer with 128 nodes and a final softmax activation.
We built selective heads using a linear layer with 128 nodes, batch normalization, relu activation and another 64-node linear layer ending with a single node and Sigmoid activation.
All the models are available in the code \href{https://www.dropbox.com/sh/zwtskpq5f4tuuh0/AABEWccp0In_KqRaCSiqRGBPa?dl=0}{here}.
We point out that the lack of a clear design methodology of the SelectiveNet structure for a given dataset is a major drawback of \textsc{SelNet} compared to the flexibility our model-agnostic method. 
Models were trained, for an expected coverage $c$, using the same loss function as in \cite{DBLP:conf/icml/GeifmanE19}:
\begin{equation}
    \mathcal{L}=\alpha \frac{\frac{1}{n} \sum_{i=1}^n l(h(\mathbf{x}_i),y_i)g(\mathbf{x}_i)}{ \hat{\phi}(g|S_n)} +\lambda(\max(0, c-\hat{\phi}(g|S_n)))^2 +(1-\alpha)\frac{1}{n}\sum_{i=1}^{n}l(v(\mathbf{x}_i),y_i),
    \label{eq:sel_loss}
\end{equation}
where $h(\mathbf{x}_i)$ is the classification head prediction, $l(h(\mathbf{x}_i), y_i)$ is the cross entropy loss, $g(\mathbf{x}_i)$ is the selection head output over instance $i$, $v(\mathbf{x}_i)$ is the auxiliary head prediction and $\hat{\phi}(g|S_n) = \nicefrac{\sum_{i=1}^n g(\mathbf{x}_i)} {n}$ is the empirical coverage.
Both parameters $\lambda$ and $\alpha$ are set as in \cite{DBLP:conf/icml/GeifmanE19} to $\lambda = 32$ and $\alpha = .5$.
The batch size was $512$ for Lending; $128$ for GiveMe, CSDS1, CSDS2, CSDS3 and image data; $32$ for Adult and UCICredit.
The learning procedure was run for 300 epochs with \ref{eq:sel_loss} as loss and it used as an optimizer Stochastic Gradient Descent setting $\textit{learning-rate=.1, momentum=.9,Nesterov=True}$ and a decay of $.5$ every 25 epochs as in the original paper. The training was performed over 90\% of training set instances while we used the remaining 10\% to calibrate the selective head.

\paragraph{SAT.}
Let us consider the standard binary classification problem where the classifier $h$ can produce a score for instance $i$ belonging to class $0$ or $1$\footnote{The problem is symmetrical in the binary setting.}, i.e. in our main paper $s(\mathbf{x}_i)=s_1(\mathbf{x}_i)$.
\textsc{SAT} \cite{DBLP:conf/nips/Huang0020} introduces an extra class $v$ (representing abstention) during training and replace the confidence function with the score for the additional class $v$. This allows for training a selective classifier in an end-to-end fashion. 
Given a batch of data pairs $\{(\mathbf{x}_i
, y_i)\}$ of size $M$, the model score $s_{a}(\mathbf{x}_i)$ for class $a$, and its exponential moving average $t_i$ for each sample, we optimize the classifier $h$ by minimizing:
\begin{equation}
    \mathcal{L}(h_\theta) = -\frac{1}{M}\sum_{i=1}^M
[t_{i,y_i}\log(s_{y_i}(\mathbf{x}_i)) + (1-t_{i,y_i})\log{s_v(\mathbf{x}_i)}]
\label{eq:sat_loss}
\end{equation}
where $s_{y_i}(\mathbf{x}_i)$ denotes the score attributed by the classifier to the true class of instance $i$. This loss is a composition of two terms: the first one measures the standard cross-entropy loss between prediction and original label $y_i$; the second term acts as the selection function and identifies uncertain samples in the dataset. The value $t_{i,y_i}$ trades-off these two terms: if $t_{i,y_i}$ is very small, the sample is treated as uncertain and the second term enforces the selective classifier to learn to abstain from this sample; if $t_{i,y_i}$ is close to 1, the loss recovers the standard cross entropy minimization and enforces the selective classifier to make perfect predictions.
The code was based on the \textit{pytorch} implementation available \href{https://github.com/LayneH/SAT-selective-cls}{here}.
We employed the same batch sizes as for \textsc{SelNet} and we set up training details as in the original paper of \textsc{SAT}: we used the loss in \ref{eq:sat_loss} for 300 epochs and as an optimizer Stochastic Gradient Descent, setting $\textit{learning-rate=.1, momentum=.9,Nesterov=True}$ and a decay of $.5$ every 25 epochs. The training was performed over 90\% of training set instances while we used the remaining 10\% to calibrate the selection function.

\subsection{Other metrics considered}

We report results for selective accuracy and positive rate in Table \ref{tab:metrics2}. As discussed in the main paper, \textsc{AUCross} and \textsc{PlugInAUC} do not guarantee improvements in terms of accuracy whenever the target coverage decreases. Interestingly, they are able to maintain the positive rate more stable than the compared approaches. Finally, we report training times for all the methods in Table \ref{tab:running_times}. \textsc{PlugInAUC} and \textsc{PlugIn} are clear winners over tabular datases as they can exploit fast classifiers. Both \textsc{AUCross} and \textsc{SCross} pay a factor proportional to the number of folds $K$ used in the cross-fitting part of the algorithm. This extra cost can be potentially mitigated on tabular datasets by parallelizing the cross-fitting procedure. Finally, notice that \textsc{SelNet} is the only approach which require a separate run for each target coverage $c$. 

\begin{table*}[ht!]
\centering
    \caption{\small{Performance metrics (1,000 bootstrap runs over the test set, results as mean $\pm$ stdev).}}
\resizebox{\textwidth}{!}{%
\begin{tabular}{c|c|cccccc|cccccc}
\multicolumn{1}{c}{} & \multicolumn{1}{c}{} & \multicolumn{6}{c}{\textbf{Selective Accuracy}} & \multicolumn{6}{c}{\textbf{Positive Rate}} \\
      & \boldmath{}\textbf{$c$}\unboldmath{} & \textbf{\textsc{AUCross}} & \textbf{\textsc{PlugIn}} & \textbf{\textsc{PlugInAUC}} & \textbf{\textsc{SCross}} & \textbf{\textsc{SelNet}} & \textbf{\textsc{SAT}} & \textbf{\textsc{AUCross}} & \textbf{\textsc{PlugIn}} & \textbf{\textsc{PlugInAUC}} & \textbf{\textsc{SCross}} & \textbf{\textsc{SelNet}} & \textbf{\textsc{SAT}} \\
\midrule
\multirow{6}{*}{\rotatebox[origin=c]{90}{\textbf{Adult}}}  & .99   & .870 $\pm$ .003 & \boldmath{}\textbf{.872 $\pm$ .003}\unboldmath{} & .870 $\pm$ .003 & .871 $\pm$ .003 & .848 $\pm$ .003 & .845 $\pm$ .003 & \boldmath{}\textbf{.246 $\pm$ .004}\unboldmath{} & .245 $\pm$ .004 & \boldmath{}\textbf{.246 $\pm$ .004}\unboldmath{} & .245 $\pm$ .004 & .242 $\pm$ .004 & \boldmath{}\textbf{.246 $\pm$ .004}\unboldmath{} \\
      & .95   & .875 $\pm$ .003 & \boldmath{}\textbf{.888 $\pm$ .003}\unboldmath{} & .874 $\pm$ .003 & \boldmath{}\textbf{.888 $\pm$ .003}\unboldmath{} & .858 $\pm$ .003 & .845 $\pm$ .003 & \boldmath{}\textbf{.246 $\pm$ .004}\unboldmath{} & .234 $\pm$ .004 & .247 $\pm$ .004 & .234 $\pm$ .004 & .235 $\pm$ .004 & .247 $\pm$ .004 \\
      & .90   & .882 $\pm$ .003 & \boldmath{}\textbf{.903 $\pm$ .003}\unboldmath{} & .880 $\pm$ .003 & .902 $\pm$ .003 & .872 $\pm$ .003 & .845 $\pm$ .003 & \boldmath{}\textbf{.246 $\pm$ .004}\unboldmath{} & .221 $\pm$ .004 & .248 $\pm$ .004 & .220 $\pm$ .004 & .225 $\pm$ .004 & .249 $\pm$ .004 \\
      & .85   & .888 $\pm$ .003 & \boldmath{}\textbf{.920 $\pm$ .003}\unboldmath{} & .887 $\pm$ .003 & .919 $\pm$ .003 & .886 $\pm$ .003 & .846 $\pm$ .003 & \boldmath{}\textbf{.247 $\pm$ .004}\unboldmath{} & .203 $\pm$ .004 & .249 $\pm$ .004 & .206 $\pm$ .004 & .115 $\pm$ .003 & .250 $\pm$ .004 \\
      & .80   & .899 $\pm$ .003 & \boldmath{}\textbf{.936 $\pm$ .003}\unboldmath{} & .897 $\pm$ .003 & .934 $\pm$ .003 & .901 $\pm$ .003 & .846 $\pm$ .003 & \boldmath{}\textbf{.245 $\pm$ .004}\unboldmath{} & .185 $\pm$ .004 & .249 $\pm$ .004 & .190 $\pm$ .004 & .180 $\pm$ .004 & .251 $\pm$ .004 \\
      & .75   & .907 $\pm$ .003 & \boldmath{}\textbf{.950 $\pm$ .003}\unboldmath{} & .905 $\pm$ .003 & \boldmath{}\textbf{.950 $\pm$ .003}\unboldmath{} & .908 $\pm$ .003 & .845 $\pm$ .004 & \boldmath{}\textbf{.245 $\pm$ .005}\unboldmath{} & .169 $\pm$ .004 & .250 $\pm$ .005 & .170 $\pm$ .004 & .149 $\pm$ .004 & .252 $\pm$ .004 \\
\midrule
\midrule
\multirow{6}{*}{\rotatebox[origin=c]{90}{\textbf{Lending}}}  & .99   & \boldmath{}\textbf{.899 $\pm$ .001}\unboldmath{} & \boldmath{}\textbf{.899 $\pm$ .001}\unboldmath{} & \boldmath{}\textbf{.899 $\pm$ .001}\unboldmath{} & \boldmath{}\textbf{.899 $\pm$ .001}\unboldmath{} & .870 $\pm$ .001 & .754 $\pm$ .003 & \boldmath{}\textbf{.224 $\pm$ .001}\unboldmath{} & .222 $\pm$ .001 & .224 $\pm$ .001 & \boldmath{}\textbf{.222 $\pm$ .001}\unboldmath{} & .223 $\pm$ .001 & .249 $\pm$ .003 \\
      & .95   & .908 $\pm$ .001 & .909 $\pm$ .001 & .909 $\pm$ .001 & \boldmath{}\textbf{.910 $\pm$ .001}\unboldmath{} & .873 $\pm$ .001 & .754 $\pm$ .003 & \boldmath{}\textbf{.217 $\pm$ .001}\unboldmath{} & .206 $\pm$ .001 & \boldmath{}\textbf{.217 $\pm$ .001}\unboldmath{} & .205 $\pm$ .001 & .215 $\pm$ .001 & .249 $\pm$ .003 \\
      & .90   & .921 $\pm$ .001 & .923 $\pm$ .001 & .921 $\pm$ .001 & \boldmath{}\textbf{.924 $\pm$ .001}\unboldmath{} & .912 $\pm$ .001 & .754 $\pm$ .003 & \boldmath{}\textbf{.207 $\pm$ .001}\unboldmath{} & .184 $\pm$ .001 & \boldmath{}\textbf{.207 $\pm$ .001}\unboldmath{} & .183 $\pm$ .001 & .191 $\pm$ .001 & .249 $\pm$ .003 \\
      & .85   & .934 $\pm$ .001 & .937 $\pm$ .001 & .934 $\pm$ .001 & \boldmath{}\textbf{.938 $\pm$ .001}\unboldmath{} & .885 $\pm$ .001 & .754 $\pm$ .003 & .197 $\pm$ .001 & .162 $\pm$ .001 & .197 $\pm$ .001 & .161 $\pm$ .001 & .134 $\pm$ .001 & \boldmath{}\textbf{.249 $\pm$ .003}\unboldmath{} \\
      & .80   & .948 $\pm$ .001 & \boldmath{}\textbf{.951 $\pm$ .001}\unboldmath{} & .949 $\pm$ .001 & \boldmath{}\textbf{.951 $\pm$ .001}\unboldmath{} & .926 $\pm$ .001 & .754 $\pm$ .003 & .186 $\pm$ .001 & .140 $\pm$ .001 & .186 $\pm$ .001 & .139 $\pm$ .001 & .157 $\pm$ .001 & \boldmath{}\textbf{.249 $\pm$ .003}\unboldmath{} \\
      & .75   & .964 $\pm$ .001 & .963 $\pm$ .001 & \boldmath{}\textbf{.965 $\pm$ .001}\unboldmath{} & .963 $\pm$ .001 & .924 $\pm$ .001 & .754 $\pm$ .003 & .174 $\pm$ .001 & .118 $\pm$ .001 & .174 $\pm$ .001 & .116 $\pm$ .001 & .143 $\pm$ .001 & \boldmath{}\textbf{.249 $\pm$ .003}\unboldmath{} \\
\midrule
\midrule
\multirow{6}{*}{\rotatebox[origin=c]{90}{\textbf{GiveMe}}}  & .99   & .938 $\pm$ .002 & .942 $\pm$ .002 & .938 $\pm$ .002 & .942 $\pm$ .002 & \boldmath{}\textbf{.954 $\pm$ .002}\unboldmath{} & .916 $\pm$ .002 & .068 $\pm$ .002 & .063 $\pm$ .002 & \boldmath{}\textbf{.067 $\pm$ .002}\unboldmath{} & .063 $\pm$ .002 & .047 $\pm$ .002 & .085 $\pm$ .002 \\
      & .95   & .937 $\pm$ .002 & \boldmath{}\textbf{.956 $\pm$ .002}\unboldmath{} & .937 $\pm$ .002 & \boldmath{}\textbf{.956 $\pm$ .002}\unboldmath{} & .954 $\pm$ .002 & .915 $\pm$ .002 & \boldmath{}\textbf{.068 $\pm$ .002}\unboldmath{} & .046 $\pm$ .002 & \boldmath{}\textbf{.068 $\pm$ .002}\unboldmath{} & .047 $\pm$ .002 & .047 $\pm$ .002 & .086 $\pm$ .002 \\
      & .90   & .938 $\pm$ .002 & \boldmath{}\textbf{.967 $\pm$ .001}\unboldmath{} & .937 $\pm$ .002 & \boldmath{}\textbf{.967 $\pm$ .001}\unboldmath{} & .961 $\pm$ .001 & .913 $\pm$ .002 & \boldmath{}\textbf{.067 $\pm$ .002}\unboldmath{} & .034 $\pm$ .001 & \boldmath{}\textbf{.068 $\pm$ .002}\unboldmath{} & .034 $\pm$ .001 & .040 $\pm$ .001 & .088 $\pm$ .002 \\
      & .85   & .939 $\pm$ .002 & \boldmath{}\textbf{.973 $\pm$ .001}\unboldmath{} & .938 $\pm$ .002 & \boldmath{}\textbf{.973 $\pm$ .001}\unboldmath{} & .967 $\pm$ .001 & .913 $\pm$ .002 & \boldmath{}\textbf{.067 $\pm$ .002}\unboldmath{} & .028 $\pm$ .001 & \boldmath{}\textbf{.068 $\pm$ .002}\unboldmath{} & .028 $\pm$ .001 & .034 $\pm$ .001 & .088 $\pm$ .002 \\
      & .80   & .939 $\pm$ .002 & .977 $\pm$ .001 & .938 $\pm$ .002 & \boldmath{}\textbf{.978 $\pm$ .001}\unboldmath{} & .969 $\pm$ .001 & .912 $\pm$ .002 & \boldmath{}\textbf{.067 $\pm$ .002}\unboldmath{} & .024 $\pm$ .001 & \boldmath{}\textbf{.067 $\pm$ .002}\unboldmath{} & .023 $\pm$ .001 & .032 $\pm$ .001 & .089 $\pm$ .002 \\
      & .75   & .939 $\pm$ .002 & .980 $\pm$ .001 & .939 $\pm$ .002 & \boldmath{}\textbf{.981 $\pm$ .001}\unboldmath{} & .971 $\pm$ .001 & .912 $\pm$ .002 & \boldmath{}\textbf{.067 $\pm$ .002}\unboldmath{} & .021 $\pm$ .001 & \boldmath{}\textbf{.068 $\pm$ .002}\unboldmath{} & .020 $\pm$ .001 & .030 $\pm$ .001 & .089 $\pm$ .002 \\
\midrule
\midrule
\multirow{6}{*}{\rotatebox[origin=c]{90}{\textbf{UCICredit}}}  & .99   & .811 $\pm$ .005 & .814 $\pm$ .005 & .812 $\pm$ .005 & .814 $\pm$ .005 & \boldmath{}\textbf{.815 $\pm$ .005}\unboldmath{} & .812 $\pm$ .005 & \boldmath{}\textbf{.222 $\pm$ .005}\unboldmath{} & .220 $\pm$ .005 & \boldmath{}\textbf{.222 $\pm$ .005}\unboldmath{} & .220 $\pm$ .005 & .220 $\pm$ .005 & \boldmath{}\textbf{.222 $\pm$ .005}\unboldmath{} \\
      & .95   & .810 $\pm$ .005 & .826 $\pm$ .005 & .811 $\pm$ .005 & \boldmath{}\textbf{.827 $\pm$ .005}\unboldmath{} & .824 $\pm$ .005 & .812 $\pm$ .005 & .225 $\pm$ .006 & .210 $\pm$ .005 & .224 $\pm$ .005 & .209 $\pm$ .005 & .180 $\pm$ .005 & \boldmath{}\textbf{.222 $\pm$ .005}\unboldmath{} \\
      & .90   & .809 $\pm$ .005 & .838 $\pm$ .005 & .811 $\pm$ .005 & \boldmath{}\textbf{.839 $\pm$ .005}\unboldmath{} & .830 $\pm$ .005 & .813 $\pm$ .005 & .228 $\pm$ .006 & .194 $\pm$ .005 & .226 $\pm$ .006 & .195 $\pm$ .005 & .176 $\pm$ .005 & \boldmath{}\textbf{.221 $\pm$ .005}\unboldmath{} \\
      & .85   & .806 $\pm$ .005 & .849 $\pm$ .005 & .811 $\pm$ .005 & \boldmath{}\textbf{.855 $\pm$ .005}\unboldmath{} & .845 $\pm$ .005 & .813 $\pm$ .005 & .233 $\pm$ .006 & .179 $\pm$ .005 & .229 $\pm$ .006 & .176 $\pm$ .005 & .156 $\pm$ .005 & \boldmath{}\textbf{.221 $\pm$ .005}\unboldmath{} \\
      & .80   & .806 $\pm$ .006 & .863 $\pm$ .005 & .810 $\pm$ .006 & \boldmath{}\textbf{.867 $\pm$ .005}\unboldmath{} & .858 $\pm$ .005 & .813 $\pm$ .005 & .236 $\pm$ .006 & .161 $\pm$ .005 & .232 $\pm$ .006 & .158 $\pm$ .005 & .143 $\pm$ .005 & \boldmath{}\textbf{.221 $\pm$ .005}\unboldmath{} \\
      & .75   & .806 $\pm$ .006 & .872 $\pm$ .005 & .808 $\pm$ .006 & \boldmath{}\textbf{.875 $\pm$ .005}\unboldmath{} & .871 $\pm$ .005 & .813 $\pm$ .005 & .239 $\pm$ .006 & .143 $\pm$ .005 & .237 $\pm$ .006 & .141 $\pm$ .005 & .130 $\pm$ .005 & \boldmath{}\textbf{.221 $\pm$ .005}\unboldmath{} \\
\midrule
\midrule
\multirow{6}{*}{\rotatebox[origin=c]{90}{\textbf{CSDS1}}}  & .99   & .857 $\pm$ .002 & \boldmath{}\textbf{.863 $\pm$ .002}\unboldmath{} & .857 $\pm$ .002 & \boldmath{}\textbf{.863 $\pm$ .002}\unboldmath{} & .861 $\pm$ .002 & .862 $\pm$ .002 & .145 $\pm$ .002 & .139 $\pm$ .002 & \boldmath{}\textbf{.144 $\pm$ .002}\unboldmath{} & .138 $\pm$ .002 & .140 $\pm$ .002 & .139 $\pm$ .002 \\
      & .95   & .856 $\pm$ .002 & \boldmath{}\textbf{.875 $\pm$ .002}\unboldmath{} & .856 $\pm$ .002 & \boldmath{}\textbf{.875 $\pm$ .002}\unboldmath{} & .873 $\pm$ .002 & .871 $\pm$ .002 & \boldmath{}\textbf{.146 $\pm$ .002}\unboldmath{} & .126 $\pm$ .002 & \boldmath{}\textbf{.146 $\pm$ .002}\unboldmath{} & .126 $\pm$ .002 & .128 $\pm$ .002 & .130 $\pm$ .002 \\
      & .90   & .854 $\pm$ .002 & \boldmath{}\textbf{.885 $\pm$ .002}\unboldmath{} & .855 $\pm$ .002 & \boldmath{}\textbf{.885 $\pm$ .002}\unboldmath{} & .882 $\pm$ .002 & .883 $\pm$ .002 & \boldmath{}\textbf{.147 $\pm$ .002}\unboldmath{} & .116 $\pm$ .002 & \boldmath{}\textbf{.147 $\pm$ .002}\unboldmath{} & .116 $\pm$ .002 & .119 $\pm$ .002 & .118 $\pm$ .002 \\
      & .85   & .853 $\pm$ .002 & \boldmath{}\textbf{.892 $\pm$ .002}\unboldmath{} & .853 $\pm$ .002 & \boldmath{}\textbf{.892 $\pm$ .002}\unboldmath{} & .890 $\pm$ .002 & .890 $\pm$ .002 & \boldmath{}\textbf{.149 $\pm$ .002}\unboldmath{} & .109 $\pm$ .002 & \boldmath{}\textbf{.149 $\pm$ .002}\unboldmath{} & .109 $\pm$ .002 & .111 $\pm$ .002 & .111 $\pm$ .002 \\
      & .80   & .852 $\pm$ .002 & .898 $\pm$ .002 & .852 $\pm$ .002 & \boldmath{}\textbf{.899 $\pm$ .002}\unboldmath{} & .897 $\pm$ .002 & .897 $\pm$ .002 & \boldmath{}\textbf{.150 $\pm$ .002}\unboldmath{} & .103 $\pm$ .002 & \boldmath{}\textbf{.150 $\pm$ .002}\unboldmath{} & .102 $\pm$ .002 & .104 $\pm$ .002 & .104 $\pm$ .002 \\
      & .75   & .850 $\pm$ .002 & \boldmath{}\textbf{.904 $\pm$ .002}\unboldmath{} & .850 $\pm$ .002 & \boldmath{}\textbf{.904 $\pm$ .002}\unboldmath{} & .902 $\pm$ .002 & .902 $\pm$ .002 & \boldmath{}\textbf{.152 $\pm$ .002}\unboldmath{} & .097 $\pm$ .002 & \boldmath{}\textbf{.152 $\pm$ .002}\unboldmath{} & .097 $\pm$ .002 & .099 $\pm$ .002 & .099 $\pm$ .002 \\
\midrule
\midrule
\multirow{6}{*}{\rotatebox[origin=c]{90}{\textbf{CSDS2}}}  & .99   & .982 $\pm$ .002 & .982 $\pm$ .002 & .982 $\pm$ .002 & .982 $\pm$ .002 & \boldmath{}\textbf{.983 $\pm$ .002}\unboldmath{} & .982 $\pm$ .002 & \boldmath{}\textbf{.019 $\pm$ .002}\unboldmath{} & \boldmath{}\textbf{.019 $\pm$ .002}\unboldmath{} & \boldmath{}\textbf{.019 $\pm$ .002}\unboldmath{} & \boldmath{}\textbf{.019 $\pm$ .002}\unboldmath{} & .018 $\pm$ .002 & \boldmath{}\textbf{.019 $\pm$ .002}\unboldmath{} \\
      & .95   & .982 $\pm$ .002 & .983 $\pm$ .002 & .982 $\pm$ .002 & .985 $\pm$ .002 & \boldmath{}\textbf{.984 $\pm$ .002}\unboldmath{} & .982 $\pm$ .002 & .019 $\pm$ .002 & \boldmath{}\textbf{.018 $\pm$ .002}\unboldmath{} & .019 $\pm$ .002 & .016 $\pm$ .002 & .017 $\pm$ .002 & .019 $\pm$ .002 \\
      & .90   & .982 $\pm$ .002 & .984 $\pm$ .002 & .982 $\pm$ .002 & .985 $\pm$ .002 & \boldmath{}\textbf{.986 $\pm$ .002}\unboldmath{} & .985 $\pm$ .002 & \boldmath{}\textbf{.019 $\pm$ .002}\unboldmath{} & .017 $\pm$ .002 & \boldmath{}\textbf{.019 $\pm$ .002}\unboldmath{} & .016 $\pm$ .002 & .015 $\pm$ .002 & .016 $\pm$ .002 \\
      & .85   & .982 $\pm$ .002 & .985 $\pm$ .002 & .981 $\pm$ .002 & .986 $\pm$ .002 & \boldmath{}\textbf{.986 $\pm$ .002}\unboldmath{} & .985 $\pm$ .002 & \boldmath{}\textbf{.019 $\pm$ .002}\unboldmath{} & .016 $\pm$ .002 & \boldmath{}\textbf{.019 $\pm$ .002}\unboldmath{} & .015 $\pm$ .002 & .015 $\pm$ .002 & .016 $\pm$ .002 \\
      & .80   & .982 $\pm$ .002 & .986 $\pm$ .002 & .981 $\pm$ .002 & .987 $\pm$ .002 & \boldmath{}\textbf{.987 $\pm$ .002}\unboldmath{} & .986 $\pm$ .002 & \boldmath{}\textbf{.019 $\pm$ .002}\unboldmath{} & .015 $\pm$ .002 & \boldmath{}\textbf{.019 $\pm$ .002}\unboldmath{} & .014 $\pm$ .002 & .014 $\pm$ .002 & .015 $\pm$ .002 \\
      & .75   & .981 $\pm$ .002 & .986 $\pm$ .002 & .981 $\pm$ .002 & .987 $\pm$ .002 & \boldmath{}\textbf{.988 $\pm$ .002}\unboldmath{} & .987 $\pm$ .002 & \boldmath{}\textbf{.019 $\pm$ .002}\unboldmath{} & .015 $\pm$ .002 & .020 $\pm$ .002 & .014 $\pm$ .002 & .013 $\pm$ .002 & .014 $\pm$ .002 \\
\midrule
\midrule
\multirow{6}{*}{\rotatebox[origin=c]{90}{\textbf{CSDS3}}}  & .99   & .814 $\pm$ .003 & \boldmath{}\textbf{.816 $\pm$ .003}\unboldmath{} & .813 $\pm$ .003 & \boldmath{}\textbf{.816 $\pm$ .003}\unboldmath{} & .810 $\pm$ .003 & .809 $\pm$ .003 & \boldmath{}\textbf{.253 $\pm$ .003}\unboldmath{} & .252 $\pm$ .003 & .254 $\pm$ .003 & .251 $\pm$ .003 & .251 $\pm$ .003 & .252 $\pm$ .003 \\
      & .95   & .817 $\pm$ .003 & .826 $\pm$ .003 & .816 $\pm$ .003 & \boldmath{}\textbf{.827 $\pm$ .003}\unboldmath{} & .822 $\pm$ .003 & .820 $\pm$ .003 & \boldmath{}\textbf{.254 $\pm$ .003}\unboldmath{} & .242 $\pm$ .003 & \boldmath{}\textbf{.254 $\pm$ .003}\unboldmath{} & .242 $\pm$ .003 & .244 $\pm$ .003 & .243 $\pm$ .003 \\
      & .90   & .820 $\pm$ .003 & \boldmath{}\textbf{.841 $\pm$ .003}\unboldmath{} & .819 $\pm$ .003 & \boldmath{}\textbf{.841 $\pm$ .003}\unboldmath{} & .830 $\pm$ .003 & .835 $\pm$ .003 & \boldmath{}\textbf{.254 $\pm$ .003}\unboldmath{} & .228 $\pm$ .003 & .255 $\pm$ .003 & .229 $\pm$ .003 & .233 $\pm$ .003 & .232 $\pm$ .003 \\
      & .85   & .825 $\pm$ .003 & \boldmath{}\textbf{.856 $\pm$ .003}\unboldmath{} & .823 $\pm$ .003 & .855 $\pm$ .003 & .844 $\pm$ .003 & .850 $\pm$ .003 & \boldmath{}\textbf{.254 $\pm$ .003}\unboldmath{} & .214 $\pm$ .003 & .255 $\pm$ .003 & .215 $\pm$ .003 & .219 $\pm$ .003 & .219 $\pm$ .003 \\
      & .80   & .829 $\pm$ .003 & \boldmath{}\textbf{.871 $\pm$ .003}\unboldmath{} & .826 $\pm$ .003 & .869 $\pm$ .003 & .860 $\pm$ .003 & .861 $\pm$ .003 & \boldmath{}\textbf{.254 $\pm$ .004}\unboldmath{} & .199 $\pm$ .003 & .256 $\pm$ .004 & .200 $\pm$ .003 & .206 $\pm$ .003 & .207 $\pm$ .003 \\
      & .75   & .835 $\pm$ .003 & \boldmath{}\textbf{.884 $\pm$ .003}\unboldmath{} & .830 $\pm$ .003 & .883 $\pm$ .003 & .871 $\pm$ .003 & .872 $\pm$ .003 & \boldmath{}\textbf{.254 $\pm$ .004}\unboldmath{} & .182 $\pm$ .003 & .257 $\pm$ .004 & .185 $\pm$ .003 & .178 $\pm$ .003 & .188 $\pm$ .003 \\
\midrule
\midrule
\multirow{6}{*}{\rotatebox[origin=c]{90}{\textbf{CatsVsDogs}}}  & .99   & \boldmath{}\textbf{.954 $\pm$ .004}\unboldmath{} & .950 $\pm$ .004 & .944 $\pm$ .004 & .956 $\pm$ .003 & .941 $\pm$ .004 & .950 $\pm$ .004 & \boldmath{}\textbf{.500 $\pm$ .008}\unboldmath{} & \boldmath{}\textbf{.505 $\pm$ .008}\unboldmath{} & .503 $\pm$ .008 & \boldmath{}\textbf{.500 $\pm$ .008}\unboldmath{} & \boldmath{}\textbf{.500 $\pm$ .008}\unboldmath{} & \boldmath{}\textbf{.500 $\pm$ .008}\unboldmath{} \\
      & .95   & \boldmath{}\textbf{.968 $\pm$ .003}\unboldmath{} & .949 $\pm$ .004 & .953 $\pm$ .004 & .969 $\pm$ .003 & .964 $\pm$ .003 & .964 $\pm$ .003 & \boldmath{}\textbf{.498 $\pm$ .008}\unboldmath{} & .514 $\pm$ .008 & .512 $\pm$ .008 & .495 $\pm$ .008 & .494 $\pm$ .008 & .497 $\pm$ .008 \\
      & .90   & .977 $\pm$ .003 & .947 $\pm$ .004 & .971 $\pm$ .003 & .979 $\pm$ .003 & \boldmath{}\textbf{.978 $\pm$ .003}\unboldmath{} & .977 $\pm$ .003 & .494 $\pm$ .008 & .531 $\pm$ .008 & .518 $\pm$ .008 & .490 $\pm$ .008 & .482 $\pm$ .008 & \boldmath{}\textbf{.496 $\pm$ .008}\unboldmath{} \\
      & .85   & \boldmath{}\textbf{.984 $\pm$ .002}\unboldmath{} & .944 $\pm$ .004 & .982 $\pm$ .003 & .985 $\pm$ .002 & .983 $\pm$ .002 & \boldmath{}\textbf{.984 $\pm$ .002}\unboldmath{} & .488 $\pm$ .008 & .546 $\pm$ .008 & .524 $\pm$ .008 & .488 $\pm$ .008 & .485 $\pm$ .008 & \boldmath{}\textbf{.494 $\pm$ .008}\unboldmath{} \\
      & .80   & .987 $\pm$ .002 & .942 $\pm$ .004 & .986 $\pm$ .002 & .989 $\pm$ .002 & \boldmath{}\textbf{.990 $\pm$ .002}\unboldmath{} & .985 $\pm$ .002 & .473 $\pm$ .008 & .556 $\pm$ .008 & .528 $\pm$ .008 & \boldmath{}\textbf{.493 $\pm$ .008}\unboldmath{} & .476 $\pm$ .008 & .482 $\pm$ .008 \\
      & .75   & .988 $\pm$ .002 & .938 $\pm$ .004 & .988 $\pm$ .002 & .990 $\pm$ .002 & \boldmath{}\textbf{.990 $\pm$ .002}\unboldmath{} & .987 $\pm$ .002 & .445 $\pm$ .008 & .561 $\pm$ .009 & .529 $\pm$ .009 & \boldmath{}\textbf{.500 $\pm$ .009}\unboldmath{} & .472 $\pm$ .008 & .464 $\pm$ .008 \\
\midrule
\midrule
\multirow{6}{*}{\rotatebox[origin=c]{90}{\textbf{CIFAR10-CAT}}}  & .99   & .952 $\pm$ .003 & .952 $\pm$ .003 & .953 $\pm$ .003 & .956 $\pm$ .002 & .952 $\pm$ .003 & \boldmath{}\textbf{.958 $\pm$ .002}\unboldmath{} & \boldmath{}\textbf{.101 $\pm$ .004}\unboldmath{} & .092 $\pm$ .003 & .102 $\pm$ .004 & .097 $\pm$ .004 & .098 $\pm$ .004 & .096 $\pm$ .004 \\
      & .95   & .952 $\pm$ .003 & .952 $\pm$ .003 & .951 $\pm$ .003 & .969 $\pm$ .002 & .967 $\pm$ .002 & \boldmath{}\textbf{.969 $\pm$ .002}\unboldmath{} & \boldmath{}\textbf{.100 $\pm$ .004}\unboldmath{} & .072 $\pm$ .003 & .106 $\pm$ .004 & .078 $\pm$ .003 & .084 $\pm$ .003 & .075 $\pm$ .003 \\
      & .90   & .953 $\pm$ .003 & .950 $\pm$ .003 & .950 $\pm$ .003 & \boldmath{}\textbf{.977 $\pm$ .002}\unboldmath{} & .977 $\pm$ .002 & .973 $\pm$ .002 & \boldmath{}\textbf{.100 $\pm$ .004}\unboldmath{} & .074 $\pm$ .003 & .111 $\pm$ .004 & .049 $\pm$ .003 & .064 $\pm$ .003 & .044 $\pm$ .003 \\
      & .85   & .954 $\pm$ .003 & .947 $\pm$ .003 & .948 $\pm$ .003 & \boldmath{}\textbf{.979 $\pm$ .002}\unboldmath{} & .977 $\pm$ .002 & .976 $\pm$ .002 & \boldmath{}\textbf{.099 $\pm$ .004}\unboldmath{} & .077 $\pm$ .003 & .115 $\pm$ .004 & .040 $\pm$ .003 & .024 $\pm$ .002 & .025 $\pm$ .002 \\
      & .80   & .959 $\pm$ .003 & .945 $\pm$ .003 & .946 $\pm$ .003 & \boldmath{}\textbf{.980 $\pm$ .002}\unboldmath{} & .984 $\pm$ .002 & .975 $\pm$ .002 & \boldmath{}\textbf{.095 $\pm$ .004}\unboldmath{} & .080 $\pm$ .004 & .122 $\pm$ .004 & .037 $\pm$ .003 & .017 $\pm$ .002 & .026 $\pm$ .002 \\
      & .75   & .971 $\pm$ .002 & .942 $\pm$ .003 & .943 $\pm$ .003 & .980 $\pm$ .002 & \boldmath{}\textbf{.986 $\pm$ .002}\unboldmath{} & .975 $\pm$ .002 & \boldmath{}\textbf{.084 $\pm$ .003}\unboldmath{} & \boldmath{}\textbf{.084 $\pm$ .004}\unboldmath{} & .129 $\pm$ .004 & .036 $\pm$ .003 & .015 $\pm$ .002 & .026 $\pm$ .002 \\
\midrule
\midrule
      & \#    & 4/54  & 21/54 & 2/54  & 29/54 & 12/54 & 3/54  & 39/54 & 4/54  & 20/54 & 5/54  & 1/54  & 14/54 \\
\end{tabular}%

}
\mbox{}\\[2ex]

    \label{tab:metrics2}
\end{table*}

\begin{table}[ht!]
    \centering
        \caption{Time required for running different methods (in seconds).}
    \resizebox{.6\textwidth}{!}{
\begin{tabular}{c|c|cccccc}
\multicolumn{1}{c}{} & \multicolumn{1}{c}{} & \multicolumn{6}{c}{\textbf{Training Time (seconds)}} \\
      & \boldmath{}\textbf{$c$}\unboldmath{} & \textbf{\textsc{AUCross}} & \textbf{\textsc{Plug-In}} & \textbf{\textsc{Plug-In-AUC}} & \textbf{\textsc{SCross}} & \textbf{\textsc{SelNet}} & \multicolumn{1}{c|}{\textbf{\textsc{SAT}}} \\
\midrule
\multirow{6}{*}{\rotatebox[origin=c]{90}{\textbf{Adult}}}  & .99   & 1.43  & \textbf{0.22} & \textbf{0.22} & 1.15  & 1787.93 & 2319.66 \\
      & .95   & 1.43  & \textbf{0.22} & \textbf{0.22} & 1.15  & 1779.74 & 2319.66 \\
      & .90   & 1.43  & \textbf{0.22} & \textbf{0.22} & 1.15  & 1786.99 & 2319.66 \\
      & .85   & 1.43  & \textbf{0.22} & \textbf{0.22} & 1.15  & 1781.64 & 2319.66 \\
      & .80   & 1.43  & \textbf{0.22} & \textbf{0.22} & 1.15  & 1789.44 & 2319.66 \\
      & .75   & 1.43  & \textbf{0.22} & \textbf{0.22} & 1.15  & 1788.53 & 2319.66 \\
\midrule
\midrule
\multirow{6}{*}{\rotatebox[origin=c]{90}{\textbf{Lending}}}  & .99   & 13.16 & \textbf{2.55} & \textbf{2.55} & 12.10 & 9168.31 & 14182.26 \\
      & .95   & 13.16 & \textbf{2.55} & \textbf{2.55} & 12.10 & 9198.68 & 14182.26 \\
      & .90   & 13.16 & \textbf{2.55} & \textbf{2.55} & 12.10 & 9187.33 & 14182.26 \\
      & .85   & 13.16 & \textbf{2.55} & \textbf{2.55} & 12.10 & 9204.72 & 14182.26 \\
      & .80   & 13.16 & \textbf{2.55} & \textbf{2.55} & 12.10 & 9165.67 & 14182.26 \\
      & .75   & 13.16 & \textbf{2.55} & \textbf{2.55} & 12.10 & 9199.28 & 14182.26 \\
\midrule
\midrule
\multirow{6}{*}{\rotatebox[origin=c]{90}{\textbf{GiveMe}}}  & .99   & 1.60  & \textbf{0.27} & \textbf{0.27} & 1.39  & 2341.00 & 3104.83 \\
      & .95   & 1.60  & \textbf{0.27} & \textbf{0.27} & 1.39  & 2323.07 & 3104.83 \\
      & .90   & 1.60  & \textbf{0.27} & \textbf{0.27} & 1.39  & 2332.69 & 3104.83 \\
      & .85   & 1.60  & \textbf{0.27} & \textbf{0.27} & 1.39  & 2332.97 & 3104.83 \\
      & .80   & 1.60  & \textbf{0.27} & \textbf{0.27} & 1.39  & 2326.58 & 3104.83 \\
      & .75   & 1.60  & \textbf{0.27} & \textbf{0.27} & 1.39  & 2333.06 & 3104.83 \\
\midrule
\midrule
\multirow{6}{*}{\rotatebox[origin=c]{90}{\textbf{UCICredit}}}  & .99   & 1.26  & \textbf{0.20} & \textbf{0.20} & 1.11  & 1547.44 & 2070.54 \\
      & .95   & 1.26  & \textbf{0.20} & \textbf{0.20} & 1.11  & 1549.75 & 2070.54 \\
      & .90   & 1.26  & \textbf{0.20} & \textbf{0.20} & 1.11  & 1527.84 & 2070.54 \\
      & .85   & 1.26  & \textbf{0.20} & \textbf{0.20} & 1.11  & 1547.15 & 2070.54 \\
      & .80   & 1.26  & \textbf{0.20} & \textbf{0.20} & 1.11  & 1548.82 & 2070.54 \\
      & .75   & 1.26  & \textbf{0.20} & \textbf{0.20} & 1.11  & 1541.74 & 2070.54 \\
\midrule
\midrule
\multirow{6}{*}{\rotatebox[origin=c]{90}{\textbf{CSDS1}}}  & .99   & 5.60  & \textbf{1.02} & \textbf{1.02} & 5.31  & 4815.42 & 5524.00 \\
      & .95   & 5.60  & \textbf{1.02} & \textbf{1.02} & 5.31  & 4812.96 & 5524.00 \\
      & .90   & 5.60  & \textbf{1.02} & \textbf{1.02} & 5.31  & 4813.66 & 5524.00 \\
      & .85   & 5.60  & \textbf{1.02} & \textbf{1.02} & 5.31  & 4817.26 & 5524.00 \\
      & .80   & 5.60  & \textbf{1.02} & \textbf{1.02} & 5.31  & 4818.25 & 5524.00 \\
      & .75   & 5.60  & \textbf{1.02} & \textbf{1.02} & 5.31  & 4812.89 & 5524.00 \\
\midrule
\midrule
\multirow{6}{*}{\rotatebox[origin=c]{90}{\textbf{CSDS2}}}  & .99   & 1.23  & \textbf{0.21} & \textbf{0.21} & 1.41  & 609.99 & 357.92 \\
      & .95   & 1.23  & \textbf{0.21} & \textbf{0.21} & 1.41  & 609.32 & 357.92 \\
      & .90   & 1.23  & \textbf{0.21} & \textbf{0.21} & 1.41  & 606.45 & 357.92 \\
      & .85   & 1.23  & \textbf{0.21} & \textbf{0.21} & 1.41  & 606.47 & 357.92 \\
      & .80   & 1.23  & \textbf{0.21} & \textbf{0.21} & 1.41  & 610.34 & 357.92 \\
      & .75   & 1.23  & \textbf{0.21} & \textbf{0.21} & 1.41  & 605.20 & 357.92 \\
\midrule
\midrule
\multirow{6}{*}{\rotatebox[origin=c]{90}{\textbf{CSDS3}}}  & .99   & 5.07  & \textbf{0.80} & \textbf{0.80} & 4.60  & 1188.31 & 748.88 \\
      & .95   & 5.07  & \textbf{0.80} & \textbf{0.80} & 4.60  & 1193.35 & 748.88 \\
      & .90   & 5.07  & \textbf{0.80} & \textbf{0.80} & 4.60  & 1192.50 & 748.88 \\
      & .85   & 5.07  & \textbf{0.80} & \textbf{0.80} & 4.60  & 1190.79 & 748.88 \\
      & .80   & 5.07  & \textbf{0.80} & \textbf{0.80} & 4.60  & 1191.03 & 748.88 \\
      & .75   & 5.07  & \textbf{0.80} & \textbf{0.80} & 4.60  & 1191.56 & 748.88 \\
\midrule
\midrule
\multirow{6}{*}{\rotatebox[origin=c]{90}{\textbf{CatsVsDogs}}}  & .99   & 14668.55 & \textbf{2624.68} & \textbf{2624.68} & 14475.64 & 2700.46 & 2735.28 \\
      & .95   & 14668.55 & \textbf{2624.68} & \textbf{2624.68} & 14475.64 & 2699.13 & 2735.28 \\
      & .90   & 14668.55 & \textbf{2624.68} & \textbf{2624.68} & 14475.64 & 2733.35 & 2735.28 \\
      & .85   & 14668.55 & \textbf{2624.68} & \textbf{2624.68} & 14475.64 & 2769.44 & 2735.28 \\
      & .80   & 14668.55 & \textbf{2624.68} & \textbf{2624.68} & 14475.64 & 2766.80 & 2735.28 \\
      & .75   & 14668.55 & \textbf{2624.68} & \textbf{2624.68} & 14475.64 & 2762.34 & 2735.28 \\
\midrule
\midrule
\multirow{6}{*}{\rotatebox[origin=c]{90}{\textbf{CIFAR10-CAT}}}  & .99   & 17401.61 & \textbf{3026.98} & \textbf{3026.98} & 16924.30 & 3317.14 & 3239.50 \\
      & .95   & 17401.61 & \textbf{3026.98} & \textbf{3026.98} & 16924.30 & 3296.70 & 3239.50 \\
      & .90   & 17401.61 & \textbf{3026.98} & \textbf{3026.98} & 16924.30 & 3337.28 & 3239.50 \\
      & .85   & 17401.61 & \textbf{3026.98} & \textbf{3026.98} & 16924.30 & 3358.18 & 3239.50 \\
      & .80   & 17401.61 & \textbf{3026.98} & \textbf{3026.98} & 16924.30 & 3262.57 & 3239.50 \\
      & .75   & 17401.61 & \textbf{3026.98} & \textbf{3026.98} & 16924.30 & 3294.08 & 3239.50 \\
\midrule
\midrule
      & \#    & 0/54  & 54/54 & 54/54 & 0/54  & 0/54  & 0/54 \\
\end{tabular}%

    }

    \label{tab:running_times}
\end{table}

\subsection{Results for different classifiers}

We report in Tables \ref{tab:class_cov}-\ref{tab:class_pos} the results for \textsc{AUCross} and \textsc{PlugInAUC} using different classifiers over tabular datasets. Regarding coverage, we see the harshest violations for both \textsc{AUCross}-\textsc{RandForest} and \textsc{PlugInAUC}-\textsc{RandForest} over \textit{CSDS2} and for \textsc{AUCross}-\textsc{ResNet} over \textit{GiveMe}. Coverage violations occur also for \textit{Lending} dataset, independently of the classifier.
Regarding the AUC, both \textsc{AUCross} and \textsc{PlugInAUC} succeed in increasing it while target coverage drops, regardless of the considered base classifier. At the same time, we notice that for all the base classifiers lowering coverage does not guarantee selective accuracy to increase, highlighting once more the trade-off between these two metrics.
Finally, we see similar results across all the classifiers for positive rate.

\subsection{Results for different $K$'s}

We report empirical coverage and selective AUC for different choices of the parameter $K$ in Table \ref{tab:K_cross}. The default value $K=5$ shows a slightly better trade-off between empirical coverage and selectice AUC.

\begin{table}[!t]
    \centering
        \caption{Empirical coverage for \textsc{AUCross} and \textsc{PlugInAUC} using different classifiers (1,000 bootstrap runs over the test set, results as mean $\pm$ stdev).}
    \resizebox{\textwidth}{!}{
\begin{tabular}{r|c|llllllll}
\multicolumn{1}{c}{} & \multicolumn{1}{c}{} & \multicolumn{8}{c}{\textbf{Empirical Coverage}} \\
\multicolumn{1}{c}{} & \multicolumn{1}{c}{} & \multicolumn{4}{c|}{\textbf{\textsc{AUCross}}} & \multicolumn{4}{c}{\textbf{\textsc{PlugInAUC}}} \\
      & \boldmath{}\textbf{$c$}\unboldmath{} & \multicolumn{1}{c}{\textbf{\textsc{Logistic}}} & \multicolumn{1}{c}{\textbf{\textsc{RandForest}}} & \multicolumn{1}{c}{\textbf{\textsc{ResNet}}} & \multicolumn{1}{c|}{\textbf{\textsc{XGBoost}}} & \multicolumn{1}{c}{\textbf{\textsc{Logistic}}} & \multicolumn{1}{c}{\textbf{\textsc{RandForest}}} & \multicolumn{1}{c}{\textbf{\textsc{ResNet}}} & \multicolumn{1}{c}{\textbf{\textsc{XGBoost}}} \\
\midrule
\multicolumn{1}{c|}{\multirow{6}{*}{\rotatebox[origin=c]{90}{\textbf{Adult}}} } & .99   & \multicolumn{1}{c}{.991 $\pm$ (.001)} & \multicolumn{1}{c}{.988 $\pm$ (.001)} & \multicolumn{1}{c}{.994 $\pm$ (.001)} & \multicolumn{1}{c|}{\boldmath{}\textbf{.990 $\pm$ (.001)}\unboldmath{}} & \multicolumn{1}{c}{.988 $\pm$ (.001)} & \multicolumn{1}{c}{\boldmath{}\textbf{.989 $\pm$ (.001)}\unboldmath{}} & \multicolumn{1}{c}{.988 $\pm$ (.001)} & \multicolumn{1}{c}{.992 $\pm$ (.001)} \\
      & .95   & \multicolumn{1}{c}{\boldmath{}\textbf{.950 $\pm$ (.002)}\unboldmath{}} & \multicolumn{1}{c}{.945 $\pm$ (.002)} & \multicolumn{1}{c}{.961 $\pm$ (.002)} & \multicolumn{1}{c|}{.951 $\pm$ (.002)} & \multicolumn{1}{c}{.947 $\pm$ (.002)} & \multicolumn{1}{c}{.949 $\pm$ (.002)} & \multicolumn{1}{c}{.946 $\pm$ (.002)} & \multicolumn{1}{c}{\boldmath{}\textbf{.951 $\pm$ (.002)}\unboldmath{}} \\
      & .90   & \multicolumn{1}{c}{\boldmath{}\textbf{.901 $\pm$ (.003)}\unboldmath{}} & \multicolumn{1}{c}{.893 $\pm$ (.003)} & \multicolumn{1}{c}{.912 $\pm$ (.003)} & \multicolumn{1}{c|}{.891 $\pm$ (.003)} & \multicolumn{1}{c}{.896 $\pm$ (.003)} & \multicolumn{1}{c}{.894 $\pm$ (.003)} & \multicolumn{1}{c}{.903 $\pm$ (.003)} & \multicolumn{1}{c}{\boldmath{}\textbf{.899 $\pm$ (.003)}\unboldmath{}} \\
      & .85   & \multicolumn{1}{c}{\boldmath{}\textbf{.850 $\pm$ (.003)}\unboldmath{}} & \multicolumn{1}{c}{.853 $\pm$ (.003)} & \multicolumn{1}{c}{.840 $\pm$ (.004)} & \multicolumn{1}{c|}{.846 $\pm$ (.003)} & \multicolumn{1}{c}{.848 $\pm$ (.003)} & \multicolumn{1}{c}{.836 $\pm$ (.004)} & \multicolumn{1}{c}{.849 $\pm$ (.003)} & \multicolumn{1}{c}{\boldmath{}\textbf{.850 $\pm$ (.003)}\unboldmath{}} \\
      & .80   & \multicolumn{1}{c}{\boldmath{}\textbf{.800 $\pm$ (.004)}\unboldmath{}} & \multicolumn{1}{c}{.804 $\pm$ (.004)} & \multicolumn{1}{c}{.792 $\pm$ (.004)} & \multicolumn{1}{c|}{.793 $\pm$ (.004)} & \multicolumn{1}{c}{.799 $\pm$ (.004)} & \multicolumn{1}{c}{.783 $\pm$ (.004)} & \multicolumn{1}{c}{.801 $\pm$ (.004)} & \multicolumn{1}{c}{\boldmath{}\textbf{.800 $\pm$ (.004)}\unboldmath{}} \\
      & .75   & \multicolumn{1}{c}{.747 $\pm$ (.004)} & \multicolumn{1}{c}{\boldmath{}\textbf{.752 $\pm$ (.004)}\unboldmath{}} & \multicolumn{1}{c}{.735 $\pm$ (.004)} & \multicolumn{1}{c|}{.744 $\pm$ (.004)} & \multicolumn{1}{c}{.744 $\pm$ (.004)} & \multicolumn{1}{c}{.743 $\pm$ (.004)} & \multicolumn{1}{c}{\boldmath{}\textbf{.749 $\pm$ (.004)}\unboldmath{}} & \multicolumn{1}{c}{.748 $\pm$ (.004)} \\
\midrule
\midrule
\multicolumn{1}{c|}{\multirow{6}{*}{\rotatebox[origin=c]{90}{\textbf{Lending}}} } & .99   & \multicolumn{1}{c}{.996 $\pm$ (.001)} & \multicolumn{1}{c}{\boldmath{}\textbf{.992 $\pm$ (.001)}\unboldmath{}} & \multicolumn{1}{c}{.995 $\pm$ (.001)} & \multicolumn{1}{c|}{.996 $\pm$ (.001)} & \multicolumn{1}{c}{.996 $\pm$ (.001)} & \multicolumn{1}{c}{\boldmath{}\textbf{.992 $\pm$ (.001)}\unboldmath{}} & \multicolumn{1}{c}{.995 $\pm$ (.001)} & \multicolumn{1}{c}{.996 $\pm$ (.001)} \\
      & .95   & \multicolumn{1}{c}{.976 $\pm$ (.001)} & \multicolumn{1}{c}{\boldmath{}\textbf{.971 $\pm$ (.001)}\unboldmath{}} & \multicolumn{1}{c}{.976 $\pm$ (.001)} & \multicolumn{1}{c|}{.980 $\pm$ (.001)} & \multicolumn{1}{c}{.976 $\pm$ (.001)} & \multicolumn{1}{c}{\boldmath{}\textbf{.971 $\pm$ (.001)}\unboldmath{}} & \multicolumn{1}{c}{.975 $\pm$ (.001)} & \multicolumn{1}{c}{.980 $\pm$ (.001)} \\
      & .90   & \multicolumn{1}{c}{.951 $\pm$ (.001)} & \multicolumn{1}{c}{\boldmath{}\textbf{.945 $\pm$ (.001)}\unboldmath{}} & \multicolumn{1}{c}{.951 $\pm$ (.001)} & \multicolumn{1}{c|}{.958 $\pm$ (.001)} & \multicolumn{1}{c}{.951 $\pm$ (.001)} & \multicolumn{1}{c}{\boldmath{}\textbf{.945 $\pm$ (.001)}\unboldmath{}} & \multicolumn{1}{c}{.949 $\pm$ (.001)} & \multicolumn{1}{c}{.958 $\pm$ (.001)} \\
      & .85   & \multicolumn{1}{c}{.926 $\pm$ (.001)} & \multicolumn{1}{c}{\boldmath{}\textbf{.918 $\pm$ (.001)}\unboldmath{}} & \multicolumn{1}{c}{.925 $\pm$ (.001)} & \multicolumn{1}{c|}{.936 $\pm$ (.001)} & \multicolumn{1}{c}{.926 $\pm$ (.001)} & \multicolumn{1}{c}{\boldmath{}\textbf{.917 $\pm$ (.001)}\unboldmath{}} & \multicolumn{1}{c}{.922 $\pm$ (.001)} & \multicolumn{1}{c}{.936 $\pm$ (.001)} \\
      & .80   & \multicolumn{1}{c}{.900 $\pm$ (.001)} & \multicolumn{1}{c}{\boldmath{}\textbf{.890 $\pm$ (.001)}\unboldmath{}} & \multicolumn{1}{c}{.898 $\pm$ (.001)} & \multicolumn{1}{c|}{.912 $\pm$ (.001)} & \multicolumn{1}{c}{.899 $\pm$ (.001)} & \multicolumn{1}{c}{\boldmath{}\textbf{.884 $\pm$ (.001)}\unboldmath{}} & \multicolumn{1}{c}{.894 $\pm$ (.001)} & \multicolumn{1}{c}{.912 $\pm$ (.001)} \\
      & .75   & \multicolumn{1}{c}{.871 $\pm$ (.001)} & \multicolumn{1}{c}{\boldmath{}\textbf{.856 $\pm$ (.001)}\unboldmath{}} & \multicolumn{1}{c}{.868 $\pm$ (.001)} & \multicolumn{1}{c|}{.886 $\pm$ (.001)} & \multicolumn{1}{c}{.871 $\pm$ (.001)} & \multicolumn{1}{c}{\boldmath{}\textbf{.855 $\pm$ (.001)}\unboldmath{}} & \multicolumn{1}{c}{.865 $\pm$ (.001)} & \multicolumn{1}{c}{.886 $\pm$ (.001)} \\
\midrule
\midrule
\multicolumn{1}{c|}{\multirow{6}{*}{\rotatebox[origin=c]{90}{\textbf{GiveMe}}} } & .99   & \multicolumn{1}{c}{\boldmath{}\textbf{.990 $\pm$ (.001)}\unboldmath{}} & \multicolumn{1}{c}{.942 $\pm$ (.002)} & \multicolumn{1}{c}{.944 $\pm$ (.002)} & \multicolumn{1}{c|}{\boldmath{}\textbf{.990 $\pm$ (.001)}\unboldmath{}} & \multicolumn{1}{c}{\boldmath{}\textbf{.991 $\pm$ (.001)}\unboldmath{}} & \multicolumn{1}{c}{.944 $\pm$ (.002)} & \multicolumn{1}{c}{.977 $\pm$ (.001)} & \multicolumn{1}{c}{\boldmath{}\textbf{.991 $\pm$ (.001)}\unboldmath{}} \\
      & .95   & \multicolumn{1}{c}{.953 $\pm$ (.002)} & \multicolumn{1}{c}{.942 $\pm$ (.002)} & \multicolumn{1}{c}{.755 $\pm$ (.003)} & \multicolumn{1}{c|}{\boldmath{}\textbf{.951 $\pm$ (.002)}\unboldmath{}} & \multicolumn{1}{c}{\boldmath{}\textbf{.950 $\pm$ (.002)}\unboldmath{}} & \multicolumn{1}{c}{.944 $\pm$ (.002)} & \multicolumn{1}{c}{.930 $\pm$ (.002)} & \multicolumn{1}{c}{\boldmath{}\textbf{.948 $\pm$ (.002)}\unboldmath{}} \\
      & .90   & \multicolumn{1}{c}{.904 $\pm$ (.002)} & \multicolumn{1}{c}{.858 $\pm$ (.002)} & \multicolumn{1}{c}{.640 $\pm$ (.003)} & \multicolumn{1}{c|}{\boldmath{}\textbf{.898 $\pm$ (.002)}\unboldmath{}} & \multicolumn{1}{c}{\boldmath{}\textbf{.898 $\pm$ (.002)}\unboldmath{}} & \multicolumn{1}{c}{.863 $\pm$ (.002)} & \multicolumn{1}{c}{.884 $\pm$ (.002)} & \multicolumn{1}{c}{\boldmath{}\textbf{.898 $\pm$ (.002)}\unboldmath{}} \\
      & .85   & \multicolumn{1}{c}{\boldmath{}\textbf{.855 $\pm$ (.002)}\unboldmath{}} & \multicolumn{1}{c}{.788 $\pm$ (.003)} & \multicolumn{1}{c}{.570 $\pm$ (.003)} & \multicolumn{1}{c|}{.849 $\pm$ (.002)} & \multicolumn{1}{c}{.844 $\pm$ (.002)} & \multicolumn{1}{c}{.794 $\pm$ (.003)} & \multicolumn{1}{c}{.841 $\pm$ (.002)} & \multicolumn{1}{c}{\boldmath{}\textbf{.849 $\pm$ (.002)}\unboldmath{}} \\
      & .80   & \multicolumn{1}{c}{.805 $\pm$ (.003)} & \multicolumn{1}{c}{.762 $\pm$ (.003)} & \multicolumn{1}{c}{.499 $\pm$ (.003)} & \multicolumn{1}{c|}{\boldmath{}\textbf{.799 $\pm$ (.003)}\unboldmath{}} & \multicolumn{1}{c}{.796 $\pm$ (.003)} & \multicolumn{1}{c}{.780 $\pm$ (.003)} & \multicolumn{1}{c}{\boldmath{}\textbf{.797 $\pm$ (.003)}\unboldmath{}} & \multicolumn{1}{c}{\boldmath{}\textbf{.797 $\pm$ (.003)}\unboldmath{}} \\
      & .75   & \multicolumn{1}{c}{.758 $\pm$ (.003)} & \multicolumn{1}{c}{.741 $\pm$ (.003)} & \multicolumn{1}{c}{.431 $\pm$ (.003)} & \multicolumn{1}{c|}{\boldmath{}\textbf{.750 $\pm$ (.003)}\unboldmath{}} & \multicolumn{1}{c}{.746 $\pm$ (.003)} & \multicolumn{1}{c}{.667 $\pm$ (.003)} & \multicolumn{1}{c}{.748 $\pm$ (.003)} & \multicolumn{1}{c}{\boldmath{}\textbf{.750 $\pm$ (.003)}\unboldmath{}} \\
\midrule
\midrule
\multicolumn{1}{c|}{\multirow{6}{*}{\rotatebox[origin=c]{90}{\textbf{UCICredit}}} } & .99   & \multicolumn{1}{c}{.991 $\pm$ (.002)} & \multicolumn{1}{c}{.949 $\pm$ (.003)} & \multicolumn{1}{c}{.987 $\pm$ (.002)} & \multicolumn{1}{c|}{\boldmath{}\textbf{.990 $\pm$ (.002)}\unboldmath{}} & \multicolumn{1}{c}{\boldmath{}\textbf{.990 $\pm$ (.002)}\unboldmath{}} & \multicolumn{1}{c}{.954 $\pm$ (.003)} & \multicolumn{1}{c}{\boldmath{}\textbf{.991 $\pm$ (.002)}\unboldmath{}} & \multicolumn{1}{c}{\boldmath{}\textbf{.990 $\pm$ (.002)}\unboldmath{}} \\
      & .95   & \multicolumn{1}{c}{\boldmath{}\textbf{.953 $\pm$ (.003)}\unboldmath{}} & \multicolumn{1}{c}{.949 $\pm$ (.003)} & \multicolumn{1}{c}{.939 $\pm$ (.003)} & \multicolumn{1}{c|}{.947 $\pm$ (.003)} & \multicolumn{1}{c}{.962 $\pm$ (.003)} & \multicolumn{1}{c}{.904 $\pm$ (.004)} & \multicolumn{1}{c}{.946 $\pm$ (.003)} & \multicolumn{1}{c}{\boldmath{}\textbf{.948 $\pm$ (.003)}\unboldmath{}} \\
      & .90   & \multicolumn{1}{c}{\boldmath{}\textbf{.903 $\pm$ (.004)}\unboldmath{}} & \multicolumn{1}{c}{.878 $\pm$ (.004)} & \multicolumn{1}{c}{.873 $\pm$ (.004)} & \multicolumn{1}{c|}{.894 $\pm$ (.004)} & \multicolumn{1}{c}{.913 $\pm$ (.004)} & \multicolumn{1}{c}{.857 $\pm$ (.005)} & \multicolumn{1}{c}{\boldmath{}\textbf{.901 $\pm$ (.004)}\unboldmath{}} & \multicolumn{1}{c}{.894 $\pm$ (.004)} \\
      & .85   & \multicolumn{1}{c}{\boldmath{}\textbf{.853 $\pm$ (.005)}\unboldmath{}} & \multicolumn{1}{c}{.827 $\pm$ (.005)} & \multicolumn{1}{c}{.814 $\pm$ (.005)} & \multicolumn{1}{c|}{.840 $\pm$ (.005)} & \multicolumn{1}{c}{.864 $\pm$ (.004)} & \multicolumn{1}{c}{.838 $\pm$ (.005)} & \multicolumn{1}{c}{\boldmath{}\textbf{.849 $\pm$ (.005)}\unboldmath{}} & \multicolumn{1}{c}{.848 $\pm$ (.005)} \\
      & .80   & \multicolumn{1}{c}{\boldmath{}\textbf{.808 $\pm$ (.005)}\unboldmath{}} & \multicolumn{1}{c}{.777 $\pm$ (.005)} & \multicolumn{1}{c}{.759 $\pm$ (.005)} & \multicolumn{1}{c|}{.784 $\pm$ (.005)} & \multicolumn{1}{c}{.816 $\pm$ (.005)} & \multicolumn{1}{c}{.775 $\pm$ (.005)} & \multicolumn{1}{c}{\boldmath{}\textbf{.796 $\pm$ (.005)}\unboldmath{}} & \multicolumn{1}{c}{.806 $\pm$ (.005)} \\
      & .75   & \multicolumn{1}{c}{\boldmath{}\textbf{.762 $\pm$ (.005)}\unboldmath{}} & \multicolumn{1}{c}{.712 $\pm$ (.006)} & \multicolumn{1}{c}{.706 $\pm$ (.006)} & \multicolumn{1}{c|}{.729 $\pm$ (.006)} & \multicolumn{1}{c}{.770 $\pm$ (.005)} & \multicolumn{1}{c}{.733 $\pm$ (.006)} & \multicolumn{1}{c}{.744 $\pm$ (.006)} & \multicolumn{1}{c}{\boldmath{}\textbf{.752 $\pm$ (.005)}\unboldmath{}} \\
\midrule
\midrule
\multicolumn{1}{c|}{\multirow{6}{*}{\rotatebox[origin=c]{90}{\textbf{CSDS1}}} } & .99   & \multicolumn{1}{c}{\boldmath{}\textbf{.991 $\pm$ (.001)}\unboldmath{}} & \multicolumn{1}{c}{.913 $\pm$ (.002)} & \multicolumn{1}{c}{\boldmath{}\textbf{.991 $\pm$ (.001)}\unboldmath{}} & \multicolumn{1}{c|}{\boldmath{}\textbf{.991 $\pm$ (.001)}\unboldmath{}} & \multicolumn{1}{c}{.991 $\pm$ (.001)} & \multicolumn{1}{c}{.961 $\pm$ (.001)} & \multicolumn{1}{c}{\boldmath{}\textbf{.990 $\pm$ (.001)}\unboldmath{}} & \multicolumn{1}{c}{.991 $\pm$ (.001)} \\
      & .95   & \multicolumn{1}{c}{\boldmath{}\textbf{.951 $\pm$ (.001)}\unboldmath{}} & \multicolumn{1}{c}{.913 $\pm$ (.002)} & \multicolumn{1}{c}{.953 $\pm$ (.001)} & \multicolumn{1}{c|}{\boldmath{}\textbf{.951 $\pm$ (.001)}\unboldmath{}} & \multicolumn{1}{c}{\boldmath{}\textbf{.950 $\pm$ (.001)}\unboldmath{}} & \multicolumn{1}{c}{.914 $\pm$ (.001)} & \multicolumn{1}{c}{.949 $\pm$ (.001)} & \multicolumn{1}{c}{\boldmath{}\textbf{.950 $\pm$ (.001)}\unboldmath{}} \\
      & .90   & \multicolumn{1}{c}{.902 $\pm$ (.002)} & \multicolumn{1}{c}{.870 $\pm$ (.002)} & \multicolumn{1}{c}{.905 $\pm$ (.002)} & \multicolumn{1}{c|}{\boldmath{}\textbf{.901 $\pm$ (.002)}\unboldmath{}} & \multicolumn{1}{c}{\boldmath{}\textbf{.900 $\pm$ (.002)}\unboldmath{}} & \multicolumn{1}{c}{.872 $\pm$ (.002)} & \multicolumn{1}{c}{.899 $\pm$ (.002)} & \multicolumn{1}{c}{.899 $\pm$ (.002)} \\
      & .85   & \multicolumn{1}{c}{.852 $\pm$ (.002)} & \multicolumn{1}{c}{.822 $\pm$ (.002)} & \multicolumn{1}{c}{.854 $\pm$ (.002)} & \multicolumn{1}{c|}{\boldmath{}\textbf{.851 $\pm$ (.002)}\unboldmath{}} & \multicolumn{1}{c}{.851 $\pm$ (.002)} & \multicolumn{1}{c}{.820 $\pm$ (.002)} & \multicolumn{1}{c}{.851 $\pm$ (.002)} & \multicolumn{1}{c}{\boldmath{}\textbf{.850 $\pm$ (.002)}\unboldmath{}} \\
      & .80   & \multicolumn{1}{c}{.803 $\pm$ (.002)} & \multicolumn{1}{c}{.784 $\pm$ (.002)} & \multicolumn{1}{c}{.797 $\pm$ (.002)} & \multicolumn{1}{c|}{\boldmath{}\textbf{.799 $\pm$ (.002)}\unboldmath{}} & \multicolumn{1}{c}{\boldmath{}\textbf{.800 $\pm$ (.002)}\unboldmath{}} & \multicolumn{1}{c}{.786 $\pm$ (.002)} & \multicolumn{1}{c}{.798 $\pm$ (.002)} & \multicolumn{1}{c}{\boldmath{}\textbf{.800 $\pm$ (.002)}\unboldmath{}} \\
      & .75   & \multicolumn{1}{c}{.752 $\pm$ (.002)} & \multicolumn{1}{c}{.729 $\pm$ (.002)} & \multicolumn{1}{c}{.744 $\pm$ (.002)} & \multicolumn{1}{c|}{\boldmath{}\textbf{.749 $\pm$ (.002)}\unboldmath{}} & \multicolumn{1}{c}{\boldmath{}\textbf{.749 $\pm$ (.002)}\unboldmath{}} & \multicolumn{1}{c}{.698 $\pm$ (.002)} & \multicolumn{1}{c}{.746 $\pm$ (.002)} & \multicolumn{1}{c}{.752 $\pm$ (.002)} \\
\midrule
\midrule
\multicolumn{1}{c|}{\multirow{6}{*}{\rotatebox[origin=c]{90}{\textbf{CSDS2}}} } & .99   & \multicolumn{1}{c}{.993 $\pm$ (.001)} & \multicolumn{1}{c}{.653 $\pm$ (.005)} & \multicolumn{1}{c}{\boldmath{}\textbf{.990 $\pm$ (.001)}\unboldmath{}} & \multicolumn{1}{c|}{\boldmath{}\textbf{.990 $\pm$ (.001)}\unboldmath{}} & \multicolumn{1}{c}{\boldmath{}\textbf{.990 $\pm$ (.001)}\unboldmath{}} & \multicolumn{1}{c}{.829 $\pm$ (.004)} & \multicolumn{1}{c}{.989 $\pm$ (.001)} & \multicolumn{1}{c}{.991 $\pm$ (.001)} \\
      & .95   & \multicolumn{1}{c}{.953 $\pm$ (.002)} & \multicolumn{1}{c}{.653 $\pm$ (.005)} & \multicolumn{1}{c}{.942 $\pm$ (.003)} & \multicolumn{1}{c|}{\boldmath{}\textbf{.949 $\pm$ (.002)}\unboldmath{}} & \multicolumn{1}{c}{\boldmath{}\textbf{.950 $\pm$ (.003)}\unboldmath{}} & \multicolumn{1}{c}{.829 $\pm$ (.004)} & \multicolumn{1}{c}{.947 $\pm$ (.003)} & \multicolumn{1}{c}{.954 $\pm$ (.002)} \\
      & .90   & \multicolumn{1}{c}{\boldmath{}\textbf{.903 $\pm$ (.003)}\unboldmath{}} & \multicolumn{1}{c}{.653 $\pm$ (.005)} & \multicolumn{1}{c}{.881 $\pm$ (.003)} & \multicolumn{1}{c|}{.890 $\pm$ (.003)} & \multicolumn{1}{c}{.894 $\pm$ (.003)} & \multicolumn{1}{c}{.829 $\pm$ (.004)} & \multicolumn{1}{c}{\boldmath{}\textbf{.902 $\pm$ (.003)}\unboldmath{}} & \multicolumn{1}{c}{.907 $\pm$ (.003)} \\
      & .85   & \multicolumn{1}{c}{\boldmath{}\textbf{.854 $\pm$ (.004)}\unboldmath{}} & \multicolumn{1}{c}{.653 $\pm$ (.005)} & \multicolumn{1}{c}{.824 $\pm$ (.004)} & \multicolumn{1}{c|}{.834 $\pm$ (.004)} & \multicolumn{1}{c}{.841 $\pm$ (.004)} & \multicolumn{1}{c}{.650 $\pm$ (.005)} & \multicolumn{1}{c}{.856 $\pm$ (.004)} & \multicolumn{1}{c}{\boldmath{}\textbf{.851 $\pm$ (.004)}\unboldmath{}} \\
      & .80   & \multicolumn{1}{c}{\boldmath{}\textbf{.803 $\pm$ (.004)}\unboldmath{}} & \multicolumn{1}{c}{.653 $\pm$ (.005)} & \multicolumn{1}{c}{.761 $\pm$ (.004)} & \multicolumn{1}{c|}{.781 $\pm$ (.004)} & \multicolumn{1}{c}{\boldmath{}\textbf{.798 $\pm$ (.004)}\unboldmath{}} & \multicolumn{1}{c}{.650 $\pm$ (.005)} & \multicolumn{1}{c}{.806 $\pm$ (.004)} & \multicolumn{1}{c}{.806 $\pm$ (.004)} \\
      & .75   & \multicolumn{1}{c}{\boldmath{}\textbf{.756 $\pm$ (.004)}\unboldmath{}} & \multicolumn{1}{c}{.653 $\pm$ (.005)} & \multicolumn{1}{c}{.703 $\pm$ (.005)} & \multicolumn{1}{c|}{.732 $\pm$ (.004)} & \multicolumn{1}{c}{\boldmath{}\textbf{.749 $\pm$ (.004)}\unboldmath{}} & \multicolumn{1}{c}{.516 $\pm$ (.005)} & \multicolumn{1}{c}{.756 $\pm$ (.004)} & \multicolumn{1}{c}{.754 $\pm$ (.004)} \\
\midrule
\midrule
\multicolumn{1}{c|}{\multirow{6}{*}{\rotatebox[origin=c]{90}{\textbf{CSDS3}}} } & .99   & \multicolumn{1}{c}{\boldmath{}\textbf{.990 $\pm$ (.001)}\unboldmath{}} & \multicolumn{1}{c}{.980 $\pm$ (.001)} & \multicolumn{1}{c}{.992 $\pm$ (.001)} & \multicolumn{1}{c|}{.989 $\pm$ (.001)} & \multicolumn{1}{c}{.992 $\pm$ (.001)} & \multicolumn{1}{c}{.989 $\pm$ (.001)} & \multicolumn{1}{c}{\boldmath{}\textbf{.991 $\pm$ (.001)}\unboldmath{}} & \multicolumn{1}{c}{\boldmath{}\textbf{.991 $\pm$ (.001)}\unboldmath{}} \\
      & .95   & \multicolumn{1}{c}{.948 $\pm$ (.002)} & \multicolumn{1}{c}{.933 $\pm$ (.002)} & \multicolumn{1}{c}{.954 $\pm$ (.002)} & \multicolumn{1}{c|}{\boldmath{}\textbf{.949 $\pm$ (.002)}\unboldmath{}} & \multicolumn{1}{c}{.955 $\pm$ (.002)} & \multicolumn{1}{c}{.941 $\pm$ (.002)} & \multicolumn{1}{c}{\boldmath{}\textbf{.952 $\pm$ (.002)}\unboldmath{}} & \multicolumn{1}{c}{.955 $\pm$ (.002)} \\
      & .90   & \multicolumn{1}{c}{\boldmath{}\textbf{.899 $\pm$ (.002)}\unboldmath{}} & \multicolumn{1}{c}{.887 $\pm$ (.003)} & \multicolumn{1}{c}{.907 $\pm$ (.002)} & \multicolumn{1}{c|}{.903 $\pm$ (.002)} & \multicolumn{1}{c}{.905 $\pm$ (.002)} & \multicolumn{1}{c}{.896 $\pm$ (.003)} & \multicolumn{1}{c}{\boldmath{}\textbf{.901 $\pm$ (.003)}\unboldmath{}} & \multicolumn{1}{c}{.906 $\pm$ (.002)} \\
      & .85   & \multicolumn{1}{c}{\boldmath{}\textbf{.851 $\pm$ (.003)}\unboldmath{}} & \multicolumn{1}{c}{.831 $\pm$ (.003)} & \multicolumn{1}{c}{.860 $\pm$ (.003)} & \multicolumn{1}{c|}{.854 $\pm$ (.003)} & \multicolumn{1}{c}{.857 $\pm$ (.003)} & \multicolumn{1}{c}{.838 $\pm$ (.003)} & \multicolumn{1}{c}{\boldmath{}\textbf{.850 $\pm$ (.003)}\unboldmath{}} & \multicolumn{1}{c}{.860 $\pm$ (.003)} \\
      & .80   & \multicolumn{1}{c}{\boldmath{}\textbf{.799 $\pm$ (.003)}\unboldmath{}} & \multicolumn{1}{c}{.795 $\pm$ (.003)} & \multicolumn{1}{c}{.789 $\pm$ (.003)} & \multicolumn{1}{c|}{.814 $\pm$ (.003)} & \multicolumn{1}{c}{.799 $\pm$ (.003)} & \multicolumn{1}{c}{\boldmath{}\textbf{.803 $\pm$ (.003)}\unboldmath{}} & \multicolumn{1}{c}{\boldmath{}\textbf{.803 $\pm$ (.003)}\unboldmath{}} & \multicolumn{1}{c}{.814 $\pm$ (.003)} \\
      & .75   & \multicolumn{1}{c}{\boldmath{}\textbf{.744 $\pm$ (.003)}\unboldmath{}} & \multicolumn{1}{c}{.738 $\pm$ (.003)} & \multicolumn{1}{c}{\boldmath{}\textbf{.744 $\pm$ (.003)}\unboldmath{}} & \multicolumn{1}{c|}{.760 $\pm$ (.003)} & \multicolumn{1}{c}{.746 $\pm$ (.003)} & \multicolumn{1}{c}{\boldmath{}\textbf{.748 $\pm$ (.003)}\unboldmath{}} & \multicolumn{1}{c}{.753 $\pm$ (.003)} & \multicolumn{1}{c}{.762 $\pm$ (.003)} \\
\midrule
\midrule
      & \#    & \multicolumn{1}{c}{22/42} & \multicolumn{1}{c}{7/42} & \multicolumn{1}{c}{3/42} & \multicolumn{1}{c|}{16/42} & \multicolumn{1}{c}{12/42} & \multicolumn{1}{c}{9/42} & \multicolumn{1}{c}{13/42} & \multicolumn{1}{c}{18/42} \\
\midrule
      & $V$   & .015  $\pm$ .034 & .064  $\pm$ .087 & .064  $\pm$ .108 & \multicolumn{1}{c|}{.019  $\pm$ .038} & .017  $\pm$ .034 & .056  $\pm$ .069 & .016  $\pm$ .032 & .017  $\pm$ .038 \\
\end{tabular}%

}

    \label{tab:class_cov}
\end{table}

\begin{table}[t]
    \centering
        \caption{Selective AUC for \textsc{AUCross} and \textsc{PlugInAUC} using different classifiers (1,000 bootstrap runs over the test set, results as mean $\pm$ stdev).}
    \resizebox{\textwidth}{!}{
\begin{tabular}{c|c|cccc|cccc}
\multicolumn{1}{r}{} & \multicolumn{1}{r}{} & \multicolumn{8}{c}{\textbf{Selective AUC}} \\
\multicolumn{1}{r}{} & \multicolumn{1}{r}{} & \multicolumn{4}{c}{\textbf{\textsc{AUCross}}} & \multicolumn{4}{c}{\textbf{\textsc{PlugInAUC}}} \\
      & \multicolumn{1}{c|}{\boldmath{}\textbf{$c$}\unboldmath{}} & \textbf{\textsc{Logistic}} & \textbf{\textsc{ RandForest}} & \textbf{\textsc{ResNet}} & \textbf{\textsc{XGBoost}} & \textbf{\textsc{Logistic}} & \textbf{\textsc{ RandForest}} & \textbf{\textsc{ResNet}} & \textbf{\textsc{XGBoost}} \\
\midrule
\multicolumn{1}{c|}{\multirow{6}{*}{\rotatebox[origin=c]{90}{\textbf{Adult}}} } & \multicolumn{1}{c|}{.99} & .903 $\pm$ .003 & .888 $\pm$ .003 & .901 $\pm$ .003 & \boldmath{}\textbf{.928 $\pm$ .003}\unboldmath{} & .904 $\pm$ .003 & .886 $\pm$ .003 & .906 $\pm$ .003 & \boldmath{}\textbf{.927 $\pm$ .003}\unboldmath{} \\
      & \multicolumn{1}{c|}{.95} & .910 $\pm$ .003 & .895 $\pm$ .003 & .907 $\pm$ .003 & \boldmath{}\textbf{.934 $\pm$ .003}\unboldmath{} & .911 $\pm$ .003 & .892 $\pm$ .003 & .912 $\pm$ .003 & \boldmath{}\textbf{.934 $\pm$ .003}\unboldmath{} \\
      & \multicolumn{1}{c|}{.90} & .918 $\pm$ .003 & .903 $\pm$ .003 & .914 $\pm$ .003 & \boldmath{}\textbf{.944 $\pm$ .002}\unboldmath{} & .918 $\pm$ .003 & .900 $\pm$ .003 & .919 $\pm$ .003 & \boldmath{}\textbf{.942 $\pm$ .003}\unboldmath{} \\
      & \multicolumn{1}{c|}{.85} & .926 $\pm$ .003 & .910 $\pm$ .003 & .925 $\pm$ .003 & \boldmath{}\textbf{.950 $\pm$ .002}\unboldmath{} & .925 $\pm$ .003 & .908 $\pm$ .003 & .927 $\pm$ .003 & \boldmath{}\textbf{.949 $\pm$ .002}\unboldmath{} \\
      & \multicolumn{1}{c|}{.80} & .934 $\pm$ .003 & .918 $\pm$ .003 & .933 $\pm$ .003 & \boldmath{}\textbf{.959 $\pm$ .002}\unboldmath{} & .932 $\pm$ .003 & .916 $\pm$ .003 & .934 $\pm$ .003 & \boldmath{}\textbf{.956 $\pm$ .002}\unboldmath{} \\
      & \multicolumn{1}{c|}{.75} & .941 $\pm$ .003 & .925 $\pm$ .003 & .942 $\pm$ .003 & \boldmath{}\textbf{.965 $\pm$ .002}\unboldmath{} & .939 $\pm$ .003 & .922 $\pm$ .003 & .942 $\pm$ .003 & \boldmath{}\textbf{.962 $\pm$ .002}\unboldmath{} \\
\midrule
\midrule
\multicolumn{1}{c|}{\multirow{6}{*}{\rotatebox[origin=c]{90}{\textbf{Lending}}} } & \multicolumn{1}{c|}{.99} & .956 $\pm$ .001 & .977 $\pm$ .001 & .981 $\pm$ .001 & \boldmath{}\textbf{.987 $\pm$ .001}\unboldmath{} & .956 $\pm$ .001 & .976 $\pm$ .001 & .975 $\pm$ .001 & \boldmath{}\textbf{.987 $\pm$ .001}\unboldmath{} \\
      & \multicolumn{1}{c|}{.95} & .958 $\pm$ .001 & .979 $\pm$ .001 & .983 $\pm$ .001 & \boldmath{}\textbf{.988 $\pm$ .001}\unboldmath{} & .958 $\pm$ .001 & .979 $\pm$ .001 & .978 $\pm$ .001 & \boldmath{}\textbf{.989 $\pm$ .001}\unboldmath{} \\
      & \multicolumn{1}{c|}{.90} & .962 $\pm$ .001 & .982 $\pm$ .001 & .986 $\pm$ .001 & \boldmath{}\textbf{.990 $\pm$ .001}\unboldmath{} & .962 $\pm$ .001 & .982 $\pm$ .001 & .981 $\pm$ .001 & \boldmath{}\textbf{.990 $\pm$ .001}\unboldmath{} \\
      & \multicolumn{1}{c|}{.85} & .965 $\pm$ .001 & .984 $\pm$ .001 & .988 $\pm$ .001 & \boldmath{}\textbf{.992 $\pm$ .001}\unboldmath{} & .965 $\pm$ .001 & .984 $\pm$ .001 & .985 $\pm$ .001 & \boldmath{}\textbf{.992 $\pm$ .001}\unboldmath{} \\
      & \multicolumn{1}{c|}{.80} & .968 $\pm$ .001 & .987 $\pm$ .001 & .991 $\pm$ .001 & \boldmath{}\textbf{.993 $\pm$ .001}\unboldmath{} & .968 $\pm$ .001 & .987 $\pm$ .001 & .988 $\pm$ .001 & \boldmath{}\textbf{.993 $\pm$ .001}\unboldmath{} \\
      & \multicolumn{1}{c|}{.75} & .971 $\pm$ .001 & .989 $\pm$ .001 & .993 $\pm$ .001 & \boldmath{}\textbf{.994 $\pm$ .001}\unboldmath{} & .971 $\pm$ .001 & .989 $\pm$ .001 & .991 $\pm$ .001 & \boldmath{}\textbf{.994 $\pm$ .001}\unboldmath{} \\
\midrule
\midrule
\multicolumn{1}{c|}{\multirow{6}{*}{\rotatebox[origin=c]{90}{\textbf{GiveMe}}} } & \multicolumn{1}{c|}{.99} & .699 $\pm$ .006 & .836 $\pm$ .005 & .670 $\pm$ .006 & \boldmath{}\textbf{.865 $\pm$ .004}\unboldmath{} & .698 $\pm$ .006 & .840 $\pm$ .005 & .834 $\pm$ .005 & \boldmath{}\textbf{.862 $\pm$ .004}\unboldmath{} \\
      & \multicolumn{1}{c|}{.95} & .703 $\pm$ .006 & .836 $\pm$ .005 & .688 $\pm$ .007 & \boldmath{}\textbf{.871 $\pm$ .004}\unboldmath{} & .702 $\pm$ .006 & .840 $\pm$ .005 & .840 $\pm$ .005 & \boldmath{}\textbf{.869 $\pm$ .004}\unboldmath{} \\
      & \multicolumn{1}{c|}{.90} & .708 $\pm$ .006 & .846 $\pm$ .005 & .705 $\pm$ .008 & \boldmath{}\textbf{.878 $\pm$ .004}\unboldmath{} & .708 $\pm$ .006 & .850 $\pm$ .005 & .847 $\pm$ .005 & \boldmath{}\textbf{.876 $\pm$ .004}\unboldmath{} \\
      & \multicolumn{1}{c|}{.85} & .713 $\pm$ .006 & .854 $\pm$ .005 & .720 $\pm$ .008 & \boldmath{}\textbf{.885 $\pm$ .004}\unboldmath{} & .714 $\pm$ .006 & .857 $\pm$ .005 & .853 $\pm$ .005 & \boldmath{}\textbf{.883 $\pm$ .004}\unboldmath{} \\
      & \multicolumn{1}{c|}{.80} & .720 $\pm$ .006 & .858 $\pm$ .005 & .732 $\pm$ .008 & \boldmath{}\textbf{.892 $\pm$ .004}\unboldmath{} & .720 $\pm$ .006 & .859 $\pm$ .005 & .859 $\pm$ .005 & \boldmath{}\textbf{.891 $\pm$ .004}\unboldmath{} \\
      & \multicolumn{1}{c|}{.75} & .726 $\pm$ .006 & .861 $\pm$ .005 & .749 $\pm$ .009 & \boldmath{}\textbf{.899 $\pm$ .004}\unboldmath{} & .727 $\pm$ .007 & .870 $\pm$ .005 & .866 $\pm$ .005 & \boldmath{}\textbf{.897 $\pm$ .004}\unboldmath{} \\
\midrule
\midrule
\multicolumn{1}{c|}{\multirow{6}{*}{\rotatebox[origin=c]{90}{\textbf{UCICredit}}} } & \multicolumn{1}{c|}{.99} & .701 $\pm$ .009 & .760 $\pm$ .008 & \boldmath{}\textbf{.768 $\pm$ .007}\unboldmath{} & .767 $\pm$ .008 & .702 $\pm$ .009 & .757 $\pm$ .008 & \boldmath{}\textbf{.768 $\pm$ .007}\unboldmath{} & .760 $\pm$ .008 \\
      & \multicolumn{1}{c|}{.95} & .704 $\pm$ .009 & .760 $\pm$ .008 & \boldmath{}\textbf{.774 $\pm$ .007}\unboldmath{} & .774 $\pm$ .008 & .704 $\pm$ .009 & .764 $\pm$ .008 & \boldmath{}\textbf{.774 $\pm$ .007}\unboldmath{} & .765 $\pm$ .008 \\
      & \multicolumn{1}{c|}{.90} & .709 $\pm$ .009 & .769 $\pm$ .008 & \boldmath{}\textbf{.783 $\pm$ .007}\unboldmath{} & .780 $\pm$ .008 & .708 $\pm$ .009 & .770 $\pm$ .008 & \boldmath{}\textbf{.779 $\pm$ .007}\unboldmath{} & .772 $\pm$ .008 \\
      & \multicolumn{1}{c|}{.85} & .714 $\pm$ .009 & .776 $\pm$ .008 & \boldmath{}\textbf{.792 $\pm$ .007}\unboldmath{} & .787 $\pm$ .008 & .713 $\pm$ .009 & .772 $\pm$ .008 & \boldmath{}\textbf{.785 $\pm$ .007}\unboldmath{} & .778 $\pm$ .008 \\
      & \multicolumn{1}{c|}{.80} & .719 $\pm$ .009 & .784 $\pm$ .008 & \boldmath{}\textbf{.799 $\pm$ .007}\unboldmath{} & .794 $\pm$ .008 & .718 $\pm$ .009 & .780 $\pm$ .008 & \boldmath{}\textbf{.793 $\pm$ .008}\unboldmath{} & .783 $\pm$ .008 \\
      & \multicolumn{1}{c|}{.75} & .724 $\pm$ .009 & .794 $\pm$ .008 & \boldmath{}\textbf{.807 $\pm$ .007}\unboldmath{} & .801 $\pm$ .008 & .723 $\pm$ .009 & .786 $\pm$ .008 & \boldmath{}\textbf{.802 $\pm$ .008}\unboldmath{} & .790 $\pm$ .008 \\
\midrule
\midrule
\multicolumn{1}{c|}{\multirow{6}{*}{\rotatebox[origin=c]{90}{\textbf{CSDS1}}} } & \multicolumn{1}{c|}{.99} & .676 $\pm$ .003 & .636 $\pm$ .004 & .679 $\pm$ .003 & \boldmath{}\textbf{.680 $\pm$ .003}\unboldmath{} & .676 $\pm$ .003 & .634 $\pm$ .003 & .679 $\pm$ .003 & \boldmath{}\textbf{.680 $\pm$ .003}\unboldmath{} \\
      & \multicolumn{1}{c|}{.95} & .679 $\pm$ .003 & .636 $\pm$ .004 & .682 $\pm$ .003 & \boldmath{}\textbf{.683 $\pm$ .003}\unboldmath{} & .679 $\pm$ .003 & .636 $\pm$ .004 & \boldmath{}\textbf{.683 $\pm$ .003}\unboldmath{} & \boldmath{}\textbf{.683 $\pm$ .003}\unboldmath{} \\
      & \multicolumn{1}{c|}{.90} & .682 $\pm$ .003 & .639 $\pm$ .004 & .685 $\pm$ .003 & \boldmath{}\textbf{.686 $\pm$ .003}\unboldmath{} & .682 $\pm$ .003 & .639 $\pm$ .004 & .686 $\pm$ .003 & \boldmath{}\textbf{.687 $\pm$ .003}\unboldmath{} \\
      & \multicolumn{1}{c|}{.85} & .685 $\pm$ .003 & .640 $\pm$ .004 & \boldmath{}\textbf{.690 $\pm$ .004}\unboldmath{} & \boldmath{}\textbf{.690 $\pm$ .003}\unboldmath{} & .685 $\pm$ .004 & .641 $\pm$ .004 & .690 $\pm$ .004 & \boldmath{}\textbf{.691 $\pm$ .003}\unboldmath{} \\
      & \multicolumn{1}{c|}{.80} & .689 $\pm$ .004 & .643 $\pm$ .004 & \boldmath{}\textbf{.694 $\pm$ .004}\unboldmath{} & \boldmath{}\textbf{.694 $\pm$ .004}\unboldmath{} & .689 $\pm$ .004 & .644 $\pm$ .004 & .693 $\pm$ .004 & \boldmath{}\textbf{.695 $\pm$ .003}\unboldmath{} \\
      & \multicolumn{1}{c|}{.75} & .693 $\pm$ .004 & .644 $\pm$ .004 & \boldmath{}\textbf{.698 $\pm$ .004}\unboldmath{} & .697 $\pm$ .004 & .693 $\pm$ .004 & .647 $\pm$ .004 & .697 $\pm$ .004 & \boldmath{}\textbf{.698 $\pm$ .004}\unboldmath{} \\
\midrule
\midrule
\multicolumn{1}{c|}{\multirow{6}{*}{\rotatebox[origin=c]{90}{\textbf{CSDS2}}} } & \multicolumn{1}{c|}{.99} & \boldmath{}\textbf{.615 $\pm$ .019}\unboldmath{} & .604 $\pm$ .023 & .615 $\pm$ .019 & .575 $\pm$ .021 & .616 $\pm$ .019 & .587 $\pm$ .021 & \boldmath{}\textbf{.624 $\pm$ .020}\unboldmath{} & .577 $\pm$ .019 \\
      & \multicolumn{1}{c|}{.95} & .615 $\pm$ .019 & .604 $\pm$ .023 & \boldmath{}\textbf{.622 $\pm$ .020}\unboldmath{} & .574 $\pm$ .021 & .621 $\pm$ .020 & .587 $\pm$ .021 & \boldmath{}\textbf{.626 $\pm$ .020}\unboldmath{} & .578 $\pm$ .019 \\
      & \multicolumn{1}{c|}{.90} & .619 $\pm$ .020 & .604 $\pm$ .023 & \boldmath{}\textbf{.623 $\pm$ .020}\unboldmath{} & .574 $\pm$ .021 & .627 $\pm$ .020 & .587 $\pm$ .021 & \boldmath{}\textbf{.628 $\pm$ .020}\unboldmath{} & .580 $\pm$ .019 \\
      & \multicolumn{1}{c|}{.85} & .622 $\pm$ .020 & .604 $\pm$ .023 & \boldmath{}\textbf{.628 $\pm$ .020}\unboldmath{} & .576 $\pm$ .022 & .631 $\pm$ .021 & .580 $\pm$ .023 & \boldmath{}\textbf{.633 $\pm$ .021}\unboldmath{} & .583 $\pm$ .020 \\
      & \multicolumn{1}{c|}{.80} & \boldmath{}\textbf{.626 $\pm$ .021}\unboldmath{} & .604 $\pm$ .023 & .619 $\pm$ .020 & .571 $\pm$ .022 & .636 $\pm$ .021 & .580 $\pm$ .023 & \boldmath{}\textbf{.639 $\pm$ .022}\unboldmath{} & .585 $\pm$ .020 \\
      & \multicolumn{1}{c|}{.75} & \boldmath{}\textbf{.633 $\pm$ .022}\unboldmath{} & .604 $\pm$ .023 & .616 $\pm$ .020 & .568 $\pm$ .022 & \boldmath{}\textbf{.644 $\pm$ .022}\unboldmath{} & .591 $\pm$ .025 & .639 $\pm$ .022 & .588 $\pm$ .021 \\
\midrule
\midrule
\multicolumn{1}{c|}{\multirow{6}{*}{\rotatebox[origin=c]{90}{\textbf{CSDS3}}} } & \multicolumn{1}{c|}{.99} & .840 $\pm$ .003 & .838 $\pm$ .003 & .840 $\pm$ .003 & \boldmath{}\textbf{.845 $\pm$ .003}\unboldmath{} & .840 $\pm$ .003 & .835 $\pm$ .003 & \boldmath{}\textbf{.845 $\pm$ .003}\unboldmath{} & .843 $\pm$ .003 \\
      & \multicolumn{1}{c|}{.95} & .846 $\pm$ .003 & .845 $\pm$ .003 & .845 $\pm$ .003 & \boldmath{}\textbf{.851 $\pm$ .003}\unboldmath{} & .844 $\pm$ .003 & .841 $\pm$ .003 & \boldmath{}\textbf{.850 $\pm$ .003}\unboldmath{} & .848 $\pm$ .003 \\
      & \multicolumn{1}{c|}{.90} & .852 $\pm$ .003 & .851 $\pm$ .003 & .851 $\pm$ .003 & \boldmath{}\textbf{.858 $\pm$ .003}\unboldmath{} & .851 $\pm$ .003 & .847 $\pm$ .003 & \boldmath{}\textbf{.857 $\pm$ .003}\unboldmath{} & .856 $\pm$ .003 \\
      & \multicolumn{1}{c|}{.85} & .859 $\pm$ .003 & .859 $\pm$ .003 & .858 $\pm$ .003 & \boldmath{}\textbf{.865 $\pm$ .003}\unboldmath{} & .856 $\pm$ .003 & .855 $\pm$ .004 & \boldmath{}\textbf{.864 $\pm$ .003}\unboldmath{} & .863 $\pm$ .003 \\
      & \multicolumn{1}{c|}{.80} & .866 $\pm$ .003 & .865 $\pm$ .004 & .869 $\pm$ .003 & \boldmath{}\textbf{.871 $\pm$ .003}\unboldmath{} & .865 $\pm$ .003 & .861 $\pm$ .004 & \boldmath{}\textbf{.870 $\pm$ .003}\unboldmath{} & .869 $\pm$ .003 \\
      & \multicolumn{1}{c|}{.75} & .873 $\pm$ .003 & .873 $\pm$ .004 & .875 $\pm$ .003 & \boldmath{}\textbf{.879 $\pm$ .003}\unboldmath{} & .872 $\pm$ .003 & .868 $\pm$ .004 & \boldmath{}\textbf{.878 $\pm$ .003}\unboldmath{} & .877 $\pm$ .003 \\
\midrule
\midrule
      & \#    & 3/42  & 0/42  & 12/42 & 29/42 & 1/42  & 0/42  & 18/42 & 24/42 \\
\end{tabular}%

}
    \label{tab:class_auc}
\end{table}

\begin{table}[t]
    \centering
        \caption{Selective accuracy for \textsc{AUCross} and \textsc{PlugInAUC} using different classifiers (1,000 bootstrap runs over the test set, results as mean $\pm$ stdev).}
    \resizebox{\textwidth}{!}{
\begin{tabular}{c|c|cccc|cccc}
\multicolumn{1}{r}{} & \multicolumn{1}{r}{} & \multicolumn{8}{c}{\textbf{Selective Accuracy}} \\
\multicolumn{1}{r}{} & \multicolumn{1}{r}{} & \multicolumn{4}{c}{\textbf{\textsc{AUCross}}} & \multicolumn{4}{c}{\textbf{\textsc{PlugInAUC}}} \\
      & \multicolumn{1}{c|}{\boldmath{}\textbf{$c$}\unboldmath{}} & \textbf{\textsc{Logistic}} & \textbf{\textsc{ RandForest}} & \textbf{\textsc{ResNet}} & \textbf{\textsc{XGBoost}} & \textbf{\textsc{Logistic}} & \textbf{\textsc{ RandForest}} & \textbf{\textsc{ResNet}} & \textbf{\textsc{XGBoost}} \\
\midrule
\multicolumn{1}{c|}{\multirow{6}{*}{\rotatebox[origin=c]{90}{\textbf{Adult}}} } & \multicolumn{1}{c|}{.99} & .846 $\pm$ .003 & .838 $\pm$ .004 & .844 $\pm$ .003 & \boldmath{}\textbf{.870 $\pm$ .003}\unboldmath{} & .847 $\pm$ .003 & .839 $\pm$ .004 & .844 $\pm$ .003 & \boldmath{}\textbf{.869 $\pm$ .003}\unboldmath{} \\
      & \multicolumn{1}{c|}{.95} & .851 $\pm$ .003 & .842 $\pm$ .004 & .847 $\pm$ .004 & \boldmath{}\textbf{.874 $\pm$ .003}\unboldmath{} & .851 $\pm$ .003 & .842 $\pm$ .004 & .847 $\pm$ .003 & \boldmath{}\textbf{.874 $\pm$ .003}\unboldmath{} \\
      & \multicolumn{1}{c|}{.90} & .856 $\pm$ .003 & .847 $\pm$ .004 & .853 $\pm$ .004 & \boldmath{}\textbf{.883 $\pm$ .003}\unboldmath{} & .855 $\pm$ .003 & .847 $\pm$ .004 & .853 $\pm$ .003 & \boldmath{}\textbf{.880 $\pm$ .003}\unboldmath{} \\
      & \multicolumn{1}{c|}{.85} & .863 $\pm$ .004 & .853 $\pm$ .004 & .861 $\pm$ .004 & \boldmath{}\textbf{.890 $\pm$ .003}\unboldmath{} & .861 $\pm$ .004 & .851 $\pm$ .004 & .860 $\pm$ .004 & \boldmath{}\textbf{.888 $\pm$ .003}\unboldmath{} \\
      & \multicolumn{1}{c|}{.80} & .870 $\pm$ .004 & .860 $\pm$ .004 & .870 $\pm$ .004 & \boldmath{}\textbf{.901 $\pm$ .003}\unboldmath{} & .868 $\pm$ .004 & .858 $\pm$ .004 & .866 $\pm$ .004 & \boldmath{}\textbf{.896 $\pm$ .003}\unboldmath{} \\
      & \multicolumn{1}{c|}{.75} & .879 $\pm$ .004 & .868 $\pm$ .004 & .883 $\pm$ .003 & \boldmath{}\textbf{.911 $\pm$ .003}\unboldmath{} & .875 $\pm$ .004 & .864 $\pm$ .004 & .874 $\pm$ .004 & \boldmath{}\textbf{.905 $\pm$ .003}\unboldmath{} \\
\midrule
\midrule
\multicolumn{1}{c|}{\multirow{6}{*}{\rotatebox[origin=c]{90}{\textbf{Lending}}} } & \multicolumn{1}{c|}{.99} & .848 $\pm$ .001 & .878 $\pm$ .001 & .883 $\pm$ .001 & \boldmath{}\textbf{.903 $\pm$ .001}\unboldmath{} & .848 $\pm$ .001 & .877 $\pm$ .001 & .874 $\pm$ .001 & \boldmath{}\textbf{.904 $\pm$ .001}\unboldmath{} \\
      & \multicolumn{1}{c|}{.95} & .857 $\pm$ .001 & .890 $\pm$ .001 & .894 $\pm$ .001 & \boldmath{}\textbf{.913 $\pm$ .001}\unboldmath{} & .857 $\pm$ .001 & .888 $\pm$ .001 & .884 $\pm$ .001 & \boldmath{}\textbf{.914 $\pm$ .001}\unboldmath{} \\
      & \multicolumn{1}{c|}{.90} & .869 $\pm$ .001 & .903 $\pm$ .001 & .909 $\pm$ .001 & \boldmath{}\textbf{.927 $\pm$ .001}\unboldmath{} & .869 $\pm$ .001 & .902 $\pm$ .001 & .896 $\pm$ .001 & \boldmath{}\textbf{.928 $\pm$ .001}\unboldmath{} \\
      & \multicolumn{1}{c|}{.85} & .882 $\pm$ .001 & .918 $\pm$ .001 & .924 $\pm$ .001 & \boldmath{}\textbf{.941 $\pm$ .001}\unboldmath{} & .882 $\pm$ .001 & .917 $\pm$ .001 & .910 $\pm$ .001 & \boldmath{}\textbf{.942 $\pm$ .001}\unboldmath{} \\
      & \multicolumn{1}{c|}{.80} & .896 $\pm$ .001 & .933 $\pm$ .001 & .940 $\pm$ .001 & \boldmath{}\textbf{.957 $\pm$ .001}\unboldmath{} & .895 $\pm$ .001 & .936 $\pm$ .001 & .925 $\pm$ .001 & \boldmath{}\textbf{.957 $\pm$ .001}\unboldmath{} \\
      & \multicolumn{1}{c|}{.75} & .911 $\pm$ .001 & .952 $\pm$ .001 & .957 $\pm$ .001 & \boldmath{}\textbf{.975 $\pm$ .001}\unboldmath{} & .911 $\pm$ .001 & .951 $\pm$ .001 & .941 $\pm$ .001 & \boldmath{}\textbf{.974 $\pm$ .001}\unboldmath{} \\
\midrule
\midrule
\multicolumn{1}{c|}{\multirow{6}{*}{\rotatebox[origin=c]{90}{\textbf{GiveMe}}} } & \multicolumn{1}{c|}{.99} & .934 $\pm$ .002 & .935 $\pm$ .002 & .933 $\pm$ .002 & \boldmath{}\textbf{.937 $\pm$ .002}\unboldmath{} & .934 $\pm$ .002 & .935 $\pm$ .002 & .933 $\pm$ .002 & \boldmath{}\textbf{.936 $\pm$ .002}\unboldmath{} \\
      & \multicolumn{1}{c|}{.95} & .934 $\pm$ .002 & .935 $\pm$ .002 & .929 $\pm$ .002 & \boldmath{}\textbf{.937 $\pm$ .002}\unboldmath{} & .934 $\pm$ .002 & .935 $\pm$ .002 & .932 $\pm$ .002 & \boldmath{}\textbf{.936 $\pm$ .002}\unboldmath{} \\
      & \multicolumn{1}{c|}{.90} & .934 $\pm$ .002 & .934 $\pm$ .002 & .928 $\pm$ .002 & \boldmath{}\textbf{.937 $\pm$ .002}\unboldmath{} & .934 $\pm$ .002 & .934 $\pm$ .002 & .931 $\pm$ .002 & \boldmath{}\textbf{.936 $\pm$ .002}\unboldmath{} \\
      & \multicolumn{1}{c|}{.85} & .933 $\pm$ .002 & .932 $\pm$ .002 & .928 $\pm$ .002 & \boldmath{}\textbf{.937 $\pm$ .002}\unboldmath{} & .934 $\pm$ .002 & .932 $\pm$ .002 & .930 $\pm$ .002 & \boldmath{}\textbf{.936 $\pm$ .002}\unboldmath{} \\
      & \multicolumn{1}{c|}{.80} & .934 $\pm$ .002 & .933 $\pm$ .002 & .926 $\pm$ .002 & \boldmath{}\textbf{.938 $\pm$ .002}\unboldmath{} & .935 $\pm$ .002 & .932 $\pm$ .002 & .929 $\pm$ .002 & \boldmath{}\textbf{.937 $\pm$ .002}\unboldmath{} \\
      & \multicolumn{1}{c|}{.75} & .934 $\pm$ .002 & .933 $\pm$ .002 & .923 $\pm$ .003 & \boldmath{}\textbf{.939 $\pm$ .002}\unboldmath{} & .935 $\pm$ .002 & .929 $\pm$ .002 & .929 $\pm$ .002 & \boldmath{}\textbf{.937 $\pm$ .002}\unboldmath{} \\
\midrule
\midrule
\multicolumn{1}{c|}{\multirow{6}{*}{\rotatebox[origin=c]{90}{\textbf{UCICredit}}} } & \multicolumn{1}{c|}{.99} & .784 $\pm$ .005 & .806 $\pm$ .005 & \boldmath{}\textbf{.814 $\pm$ .005}\unboldmath{} & .805 $\pm$ .005 & .785 $\pm$ .005 & .805 $\pm$ .005 & \boldmath{}\textbf{.814 $\pm$ .005}\unboldmath{} & .805 $\pm$ .005 \\
      & \multicolumn{1}{c|}{.95} & .780 $\pm$ .005 & .806 $\pm$ .005 & \boldmath{}\textbf{.812 $\pm$ .005}\unboldmath{} & .805 $\pm$ .005 & .783 $\pm$ .005 & .804 $\pm$ .005 & \boldmath{}\textbf{.812 $\pm$ .005}\unboldmath{} & .804 $\pm$ .005 \\
      & \multicolumn{1}{c|}{.90} & .777 $\pm$ .006 & .803 $\pm$ .005 & \boldmath{}\textbf{.809 $\pm$ .005}\unboldmath{} & .803 $\pm$ .005 & .777 $\pm$ .006 & .801 $\pm$ .006 & \boldmath{}\textbf{.810 $\pm$ .005}\unboldmath{} & .803 $\pm$ .005 \\
      & \multicolumn{1}{c|}{.85} & .773 $\pm$ .006 & .801 $\pm$ .006 & \boldmath{}\textbf{.808 $\pm$ .006}\unboldmath{} & .800 $\pm$ .006 & .774 $\pm$ .006 & .799 $\pm$ .006 & \boldmath{}\textbf{.806 $\pm$ .006}\unboldmath{} & .802 $\pm$ .005 \\
      & \multicolumn{1}{c|}{.80} & .769 $\pm$ .006 & .800 $\pm$ .006 & \boldmath{}\textbf{.806 $\pm$ .006}\unboldmath{} & .796 $\pm$ .006 & .770 $\pm$ .006 & .795 $\pm$ .006 & \boldmath{}\textbf{.804 $\pm$ .006}\unboldmath{} & .800 $\pm$ .006 \\
      & \multicolumn{1}{c|}{.75} & .765 $\pm$ .006 & .799 $\pm$ .006 & \boldmath{}\textbf{.802 $\pm$ .006}\unboldmath{} & .794 $\pm$ .006 & .765 $\pm$ .006 & .793 $\pm$ .006 & \boldmath{}\textbf{.806 $\pm$ .006}\unboldmath{} & .798 $\pm$ .006 \\
\midrule
\midrule
\multicolumn{1}{c|}{\multirow{6}{*}{\rotatebox[origin=c]{90}{\textbf{CSDS1}}} } & \multicolumn{1}{c|}{.99} & \boldmath{}\textbf{.857 $\pm$ .002}\unboldmath{} & .846 $\pm$ .002 & \boldmath{}\textbf{.857 $\pm$ .002}\unboldmath{} & .856 $\pm$ .002 & \boldmath{}\textbf{.857 $\pm$ .002}\unboldmath{} & .847 $\pm$ .002 & \boldmath{}\textbf{.857 $\pm$ .002}\unboldmath{} & .856 $\pm$ .002 \\
      & \multicolumn{1}{c|}{.95} & \boldmath{}\textbf{.856 $\pm$ .002}\unboldmath{} & .846 $\pm$ .002 & \boldmath{}\textbf{.856 $\pm$ .002}\unboldmath{} & .855 $\pm$ .002 & \boldmath{}\textbf{.856 $\pm$ .002}\unboldmath{} & .845 $\pm$ .002 & \boldmath{}\textbf{.856 $\pm$ .002}\unboldmath{} & .854 $\pm$ .002 \\
      & \multicolumn{1}{c|}{.90} & \boldmath{}\textbf{.854 $\pm$ .002}\unboldmath{} & .844 $\pm$ .002 & \boldmath{}\textbf{.854 $\pm$ .002}\unboldmath{} & .854 $\pm$ .002 & \boldmath{}\textbf{.855 $\pm$ .002}\unboldmath{} & .844 $\pm$ .002 & \boldmath{}\textbf{.854 $\pm$ .002}\unboldmath{} & .853 $\pm$ .002 \\
      & \multicolumn{1}{c|}{.85} & \boldmath{}\textbf{.852 $\pm$ .002}\unboldmath{} & .842 $\pm$ .002 & \boldmath{}\textbf{.853 $\pm$ .002}\unboldmath{} & .852 $\pm$ .002 & \boldmath{}\textbf{.852 $\pm$ .002}\unboldmath{} & .841 $\pm$ .002 & \boldmath{}\textbf{.853 $\pm$ .002}\unboldmath{} & .851 $\pm$ .002 \\
      & \multicolumn{1}{c|}{.80} & \boldmath{}\textbf{.851 $\pm$ .002}\unboldmath{} & .840 $\pm$ .002 & \boldmath{}\textbf{.851 $\pm$ .002}\unboldmath{} & .850 $\pm$ .002 & \boldmath{}\textbf{.851 $\pm$ .002}\unboldmath{} & .840 $\pm$ .002 & \boldmath{}\textbf{.851 $\pm$ .002}\unboldmath{} & .850 $\pm$ .002 \\
      & \multicolumn{1}{c|}{.75} & \boldmath{}\textbf{.850 $\pm$ .002}\unboldmath{} & .836 $\pm$ .002 & .849 $\pm$ .002 & .848 $\pm$ .002 & \boldmath{}\textbf{.850 $\pm$ .002}\unboldmath{} & .834 $\pm$ .002 & .849 $\pm$ .002 & .848 $\pm$ .002 \\
\midrule
\midrule
\multicolumn{1}{c|}{\multirow{6}{*}{\rotatebox[origin=c]{90}{\textbf{CSDS2}}} } & \multicolumn{1}{c|}{.99} & \boldmath{}\textbf{.982 $\pm$ .002}\unboldmath{} & .980 $\pm$ .002 & \boldmath{}\textbf{.982 $\pm$ .002}\unboldmath{} & \boldmath{}\textbf{.982 $\pm$ .002}\unboldmath{} & \boldmath{}\textbf{.982 $\pm$ .002}\unboldmath{} & .981 $\pm$ .002 & \boldmath{}\textbf{.982 $\pm$ .002}\unboldmath{} & \boldmath{}\textbf{.982 $\pm$ .002}\unboldmath{} \\
      & \multicolumn{1}{c|}{.95} & \boldmath{}\textbf{.982 $\pm$ .002}\unboldmath{} & .980 $\pm$ .002 & \boldmath{}\textbf{.982 $\pm$ .002}\unboldmath{} & \boldmath{}\textbf{.982 $\pm$ .002}\unboldmath{} & \boldmath{}\textbf{.982 $\pm$ .002}\unboldmath{} & .981 $\pm$ .002 & \boldmath{}\textbf{.982 $\pm$ .002}\unboldmath{} & \boldmath{}\textbf{.982 $\pm$ .002}\unboldmath{} \\
      & \multicolumn{1}{c|}{.90} & \boldmath{}\textbf{.982 $\pm$ .002}\unboldmath{} & .980 $\pm$ .002 & \boldmath{}\textbf{.982 $\pm$ .002}\unboldmath{} & .981 $\pm$ .002 & \boldmath{}\textbf{.982 $\pm$ .002}\unboldmath{} & .981 $\pm$ .002 & \boldmath{}\textbf{.982 $\pm$ .002}\unboldmath{} & \boldmath{}\textbf{.982 $\pm$ .002}\unboldmath{} \\
      & \multicolumn{1}{c|}{.85} & \boldmath{}\textbf{.982 $\pm$ .002}\unboldmath{} & .980 $\pm$ .002 & .981 $\pm$ .002 & .981 $\pm$ .002 & \boldmath{}\textbf{.982 $\pm$ .002}\unboldmath{} & .980 $\pm$ .002 & \boldmath{}\textbf{.982 $\pm$ .002}\unboldmath{} & \boldmath{}\textbf{.982 $\pm$ .002}\unboldmath{} \\
      & \multicolumn{1}{c|}{.80} & \boldmath{}\textbf{.982 $\pm$ .002}\unboldmath{} & .980 $\pm$ .002 & .980 $\pm$ .002 & .980 $\pm$ .002 & \boldmath{}\textbf{.982 $\pm$ .002}\unboldmath{} & .980 $\pm$ .002 & \boldmath{}\textbf{.982 $\pm$ .002}\unboldmath{} & \boldmath{}\textbf{.982 $\pm$ .002}\unboldmath{} \\
      & \multicolumn{1}{c|}{.75} & \boldmath{}\textbf{.982 $\pm$ .002}\unboldmath{} & .980 $\pm$ .002 & .980 $\pm$ .002 & .980 $\pm$ .002 & \boldmath{}\textbf{.982 $\pm$ .002}\unboldmath{} & .979 $\pm$ .002 & \boldmath{}\textbf{.982 $\pm$ .002}\unboldmath{} & \boldmath{}\textbf{.982 $\pm$ .002}\unboldmath{} \\
\midrule
\midrule
\multicolumn{1}{c|}{\multirow{6}{*}{\rotatebox[origin=c]{90}{\textbf{CSDS3}}} } & \multicolumn{1}{c|}{.99} & .808 $\pm$ .003 & .808 $\pm$ .003 & .806 $\pm$ .003 & \boldmath{}\textbf{.811 $\pm$ .003}\unboldmath{} & \boldmath{}\textbf{.809 $\pm$ .003}\unboldmath{} & .804 $\pm$ .003 & .808 $\pm$ .003 & .808 $\pm$ .003 \\
      & \multicolumn{1}{c|}{.95} & .810 $\pm$ .003 & .810 $\pm$ .003 & .807 $\pm$ .003 & \boldmath{}\textbf{.813 $\pm$ .003}\unboldmath{} & .809 $\pm$ .003 & .805 $\pm$ .003 & .809 $\pm$ .003 & \boldmath{}\textbf{.810 $\pm$ .003}\unboldmath{} \\
      & \multicolumn{1}{c|}{.90} & .811 $\pm$ .003 & .813 $\pm$ .003 & .808 $\pm$ .003 & \boldmath{}\textbf{.815 $\pm$ .003}\unboldmath{} & .809 $\pm$ .003 & .807 $\pm$ .003 & .811 $\pm$ .003 & \boldmath{}\textbf{.814 $\pm$ .003}\unboldmath{} \\
      & \multicolumn{1}{c|}{.85} & .812 $\pm$ .003 & .817 $\pm$ .003 & .810 $\pm$ .003 & \boldmath{}\textbf{.818 $\pm$ .003}\unboldmath{} & .810 $\pm$ .003 & .811 $\pm$ .003 & .813 $\pm$ .003 & \boldmath{}\textbf{.817 $\pm$ .003}\unboldmath{} \\
      & \multicolumn{1}{c|}{.80} & .814 $\pm$ .003 & .820 $\pm$ .003 & .817 $\pm$ .003 & \boldmath{}\textbf{.822 $\pm$ .003}\unboldmath{} & .812 $\pm$ .003 & .815 $\pm$ .003 & .816 $\pm$ .003 & \boldmath{}\textbf{.821 $\pm$ .003}\unboldmath{} \\
      & \multicolumn{1}{c|}{.75} & .817 $\pm$ .003 & .826 $\pm$ .003 & .820 $\pm$ .003 & \boldmath{}\textbf{.828 $\pm$ .003}\unboldmath{} & .815 $\pm$ .003 & .819 $\pm$ .003 & .821 $\pm$ .003 & \boldmath{}\textbf{.825 $\pm$ .003}\unboldmath{} \\
\midrule
\midrule
      & \#    & 12/42 & 0/42  & 14/42 & 26/42 & 13/42 & 0/42  & 17/42 & 29/42 \\
\end{tabular}%

    }

    \label{tab:class_acc}
\end{table}

\begin{table}[t]
    \centering
        \caption{Positive rate for \textsc{AUCross} and \textsc{PlugInAUC} using different classifiers (1,000 bootstrap runs over the test set, results as mean $\pm$ stdev).}
    \resizebox{\textwidth}{!}{
\begin{tabular}{c|c|cccc|cccc}
\multicolumn{1}{c}{} & \multicolumn{1}{c}{} & \multicolumn{8}{c}{\textbf{Positive Rate}} \\
\multicolumn{1}{c}{} & \multicolumn{1}{c}{} & \multicolumn{4}{c}{\textbf{\textsc{AUCross}}} & \multicolumn{4}{c}{\textbf{\textsc{PlugInAUC}}} \\
      & \boldmath{}\textbf{$c$}\unboldmath{} & \textbf{\textsc{Logistic}} & \textbf{\textsc{ RandForest}} & \textbf{\textsc{ResNet}} & \textbf{\textsc{XGBoost}} & \textbf{\textsc{Logistic}} & \textbf{\textsc{ RandForest}} & \textbf{\textsc{ResNet}} & \textbf{\textsc{XGBoost}} \\
\midrule
\multirow{6}{*}{\rotatebox[origin=c]{90}{\textbf{Adult}}}  & .99   & \boldmath{}\textbf{.246 $\pm$ .004}\unboldmath{} & \boldmath{}\textbf{.246 $\pm$ .004}\unboldmath{} & \boldmath{}\textbf{.246 $\pm$ .004}\unboldmath{} & \boldmath{}\textbf{.246 $\pm$ .004}\unboldmath{} & \boldmath{}\textbf{.246 $\pm$ .004}\unboldmath{} & \boldmath{}\textbf{.246 $\pm$ .004}\unboldmath{} & \boldmath{}\textbf{.246 $\pm$ .004}\unboldmath{} & \boldmath{}\textbf{.246 $\pm$ .004}\unboldmath{} \\
      & .95   & .245 $\pm$ .004 & \boldmath{}\textbf{.246 $\pm$ .004}\unboldmath{} & \boldmath{}\textbf{.246 $\pm$ .004}\unboldmath{} & \boldmath{}\textbf{.246 $\pm$ .004}\unboldmath{} & .246 $\pm$ .004 & \boldmath{}\textbf{.247 $\pm$ .004}\unboldmath{} & \boldmath{}\textbf{.247 $\pm$ .004}\unboldmath{} & .246 $\pm$ .004 \\
      & .90   & .245 $\pm$ .004 & \boldmath{}\textbf{.246 $\pm$ .004}\unboldmath{} & .245 $\pm$ .004 & \boldmath{}\textbf{.246 $\pm$ .004}\unboldmath{} & .247 $\pm$ .004 & \boldmath{}\textbf{.248 $\pm$ .004}\unboldmath{} & .246 $\pm$ .004 & .248 $\pm$ .004 \\
      & .85   & .245 $\pm$ .004 & .245 $\pm$ .004 & .246 $\pm$ .004 & \boldmath{}\textbf{.246 $\pm$ .004}\unboldmath{} & .248 $\pm$ .004 & \boldmath{}\textbf{.250 $\pm$ .004}\unboldmath{} & .246 $\pm$ .004 & .246 $\pm$ .004 \\
      & .80   & \boldmath{}\textbf{.244 $\pm$ .004}\unboldmath{} & \boldmath{}\textbf{.244 $\pm$ .004}\unboldmath{} & .242 $\pm$ .004 & \boldmath{}\textbf{.244 $\pm$ .004}\unboldmath{} & .247 $\pm$ .004 & \boldmath{}\textbf{.250 $\pm$ .004}\unboldmath{} & .246 $\pm$ .004 & .247 $\pm$ .004 \\
      & .75   & \boldmath{}\textbf{.243 $\pm$ .004}\unboldmath{} & \boldmath{}\textbf{.243 $\pm$ .005}\unboldmath{} & .238 $\pm$ .005 & \boldmath{}\textbf{.243 $\pm$ .005}\unboldmath{} & .248 $\pm$ .004 & \boldmath{}\textbf{.250 $\pm$ .005}\unboldmath{} & .245 $\pm$ .005 & .248 $\pm$ .005 \\
\midrule
\midrule
\multirow{6}{*}{\rotatebox[origin=c]{90}{\textbf{Lending}}}  & .99   & \boldmath{}\textbf{.224 $\pm$ .001}\unboldmath{} & .222 $\pm$ .001 & .223 $\pm$ .001 & \boldmath{}\textbf{.224 $\pm$ .001}\unboldmath{} & \boldmath{}\textbf{.224 $\pm$ .001}\unboldmath{} & .222 $\pm$ .001 & \boldmath{}\textbf{.224 $\pm$ .001}\unboldmath{} & \boldmath{}\textbf{.224 $\pm$ .001}\unboldmath{} \\
      & .95   & \boldmath{}\textbf{.216 $\pm$ .001}\unboldmath{} & .213 $\pm$ .001 & .214 $\pm$ .001 & \boldmath{}\textbf{.216 $\pm$ .001}\unboldmath{} & \boldmath{}\textbf{.217 $\pm$ .001}\unboldmath{} & .213 $\pm$ .001 & .216 $\pm$ .001 & .216 $\pm$ .001 \\
      & .90   & \boldmath{}\textbf{.206 $\pm$ .001}\unboldmath{} & .202 $\pm$ .001 & .202 $\pm$ .001 & .205 $\pm$ .001 & \boldmath{}\textbf{.207 $\pm$ .001}\unboldmath{} & .202 $\pm$ .001 & \boldmath{}\textbf{.207 $\pm$ .001}\unboldmath{} & .205 $\pm$ .001 \\
      & .85   & \boldmath{}\textbf{.196 $\pm$ .001}\unboldmath{} & .190 $\pm$ .001 & .190 $\pm$ .001 & .193 $\pm$ .001 & \boldmath{}\textbf{.196 $\pm$ .001}\unboldmath{} & .190 $\pm$ .001 & \boldmath{}\textbf{.196 $\pm$ .001}\unboldmath{} & .195 $\pm$ .001 \\
      & .80   & \boldmath{}\textbf{.184 $\pm$ .001}\unboldmath{} & .179 $\pm$ .001 & .177 $\pm$ .001 & .181 $\pm$ .001 & \boldmath{}\textbf{.185 $\pm$ .001}\unboldmath{} & .175 $\pm$ .001 & .184 $\pm$ .001 & .182 $\pm$ .001 \\
      & .75   & \boldmath{}\textbf{.171 $\pm$ .001}\unboldmath{} & .165 $\pm$ .001 & .164 $\pm$ .001 & .168 $\pm$ .001 & \boldmath{}\textbf{.172 $\pm$ .001}\unboldmath{} & .163 $\pm$ .001 & .171 $\pm$ .001 & .170 $\pm$ .001 \\
\midrule
\midrule
\multirow{6}{*}{\rotatebox[origin=c]{90}{\textbf{GiveMe}}}  & .99   & .067 $\pm$ .002 & \boldmath{}\textbf{.068 $\pm$ .002}\unboldmath{} & \boldmath{}\textbf{.068 $\pm$ .002}\unboldmath{} & .067 $\pm$ .002 & \boldmath{}\textbf{.068 $\pm$ .002}\unboldmath{} & \boldmath{}\textbf{.068 $\pm$ .002}\unboldmath{} & \boldmath{}\textbf{.068 $\pm$ .002}\unboldmath{} & \boldmath{}\textbf{.068 $\pm$ .002}\unboldmath{} \\
      & .95   & .067 $\pm$ .002 & .068 $\pm$ .002 & \boldmath{}\textbf{.072 $\pm$ .002}\unboldmath{} & .068 $\pm$ .002 & .067 $\pm$ .002 & .068 $\pm$ .002 & \boldmath{}\textbf{.069 $\pm$ .002}\unboldmath{} & .068 $\pm$ .002 \\
      & .90   & .068 $\pm$ .002 & .070 $\pm$ .002 & \boldmath{}\textbf{.072 $\pm$ .002}\unboldmath{} & .068 $\pm$ .002 & .067 $\pm$ .002 & .069 $\pm$ .002 & \boldmath{}\textbf{.070 $\pm$ .002}\unboldmath{} & .068 $\pm$ .002 \\
      & .85   & .068 $\pm$ .002 & .072 $\pm$ .002 & \boldmath{}\textbf{.073 $\pm$ .002}\unboldmath{} & .068 $\pm$ .002 & .067 $\pm$ .002 & \boldmath{}\textbf{.071 $\pm$ .002}\unboldmath{} & \boldmath{}\textbf{.071 $\pm$ .002}\unboldmath{} & .068 $\pm$ .002 \\
      & .80   & .067 $\pm$ .002 & .071 $\pm$ .002 & \boldmath{}\textbf{.075 $\pm$ .002}\unboldmath{} & .068 $\pm$ .002 & .067 $\pm$ .002 & .071 $\pm$ .002 & \boldmath{}\textbf{.072 $\pm$ .002}\unboldmath{} & .068 $\pm$ .002 \\
      & .75   & .067 $\pm$ .002 & .070 $\pm$ .002 & \boldmath{}\textbf{.077 $\pm$ .003}\unboldmath{} & .067 $\pm$ .002 & .066 $\pm$ .002 & \boldmath{}\textbf{.075 $\pm$ .002}\unboldmath{} & .073 $\pm$ .002 & .067 $\pm$ .002 \\
\midrule
\midrule
\multirow{6}{*}{\rotatebox[origin=c]{90}{\textbf{UCICredit}}}  & .99   & .222 $\pm$ .005 & \boldmath{}\textbf{.226 $\pm$ .006}\unboldmath{} & .223 $\pm$ .005 & .222 $\pm$ .005 & .222 $\pm$ .005 & \boldmath{}\textbf{.224 $\pm$ .005}\unboldmath{} & .222 $\pm$ .005 & .222 $\pm$ .005 \\
      & .95   & \boldmath{}\textbf{.227 $\pm$ .006}\unboldmath{} & .226 $\pm$ .006 & .226 $\pm$ .006 & .223 $\pm$ .005 & .225 $\pm$ .005 & \boldmath{}\textbf{.227 $\pm$ .006}\unboldmath{} & .226 $\pm$ .006 & .224 $\pm$ .005 \\
      & .90   & .231 $\pm$ .006 & \boldmath{}\textbf{.232 $\pm$ .006}\unboldmath{} & \boldmath{}\textbf{.232 $\pm$ .006}\unboldmath{} & .227 $\pm$ .006 & .230 $\pm$ .006 & \boldmath{}\textbf{.232 $\pm$ .006}\unboldmath{} & .229 $\pm$ .006 & .227 $\pm$ .006 \\
      & .85   & .235 $\pm$ .006 & .235 $\pm$ .006 & \boldmath{}\textbf{.237 $\pm$ .006}\unboldmath{} & .232 $\pm$ .006 & .234 $\pm$ .006 & .234 $\pm$ .006 & \boldmath{}\textbf{.236 $\pm$ .006}\unboldmath{} & .230 $\pm$ .006 \\
      & .80   & .239 $\pm$ .006 & .239 $\pm$ .006 & \boldmath{}\textbf{.242 $\pm$ .006}\unboldmath{} & .238 $\pm$ .006 & .238 $\pm$ .006 & \boldmath{}\textbf{.241 $\pm$ .006}\unboldmath{} & .240 $\pm$ .006 & .233 $\pm$ .006 \\
      & .75   & .243 $\pm$ .006 & .243 $\pm$ .007 & \boldmath{}\textbf{.249 $\pm$ .007}\unboldmath{} & .242 $\pm$ .006 & .244 $\pm$ .006 & \boldmath{}\textbf{.245 $\pm$ .006}\unboldmath{} & .241 $\pm$ .006 & .238 $\pm$ .006 \\
\midrule
\midrule
\multirow{6}{*}{\rotatebox[origin=c]{90}{\textbf{CSDS1}}}  & .99   & .144 $\pm$ .002 & \boldmath{}\textbf{.147 $\pm$ .002}\unboldmath{} & .145 $\pm$ .002 & .145 $\pm$ .002 & .144 $\pm$ .002 & \boldmath{}\textbf{.146 $\pm$ .002}\unboldmath{} & .144 $\pm$ .002 & .144 $\pm$ .002 \\
      & .95   & .145 $\pm$ .002 & \boldmath{}\textbf{.147 $\pm$ .002}\unboldmath{} & .146 $\pm$ .002 & .146 $\pm$ .002 & .145 $\pm$ .002 & \boldmath{}\textbf{.147 $\pm$ .002}\unboldmath{} & .145 $\pm$ .002 & .146 $\pm$ .002 \\
      & .90   & .147 $\pm$ .002 & \boldmath{}\textbf{.148 $\pm$ .002}\unboldmath{} & .147 $\pm$ .002 & .147 $\pm$ .002 & .147 $\pm$ .002 & \boldmath{}\textbf{.148 $\pm$ .002}\unboldmath{} & .147 $\pm$ .002 & .147 $\pm$ .002 \\
      & .85   & .149 $\pm$ .002 & \boldmath{}\textbf{.150 $\pm$ .002}\unboldmath{} & .148 $\pm$ .002 & .149 $\pm$ .002 & .149 $\pm$ .002 & \boldmath{}\textbf{.150 $\pm$ .002}\unboldmath{} & .148 $\pm$ .002 & .149 $\pm$ .002 \\
      & .80   & .150 $\pm$ .002 & \boldmath{}\textbf{.151 $\pm$ .002}\unboldmath{} & .150 $\pm$ .002 & \boldmath{}\textbf{.151 $\pm$ .002}\unboldmath{} & .150 $\pm$ .002 & \boldmath{}\textbf{.151 $\pm$ .002}\unboldmath{} & .150 $\pm$ .002 & .150 $\pm$ .002 \\
      & .75   & .152 $\pm$ .002 & \boldmath{}\textbf{.154 $\pm$ .002}\unboldmath{} & .152 $\pm$ .002 & .153 $\pm$ .002 & .151 $\pm$ .002 & \boldmath{}\textbf{.155 $\pm$ .002}\unboldmath{} & .152 $\pm$ .002 & .153 $\pm$ .002 \\
\midrule
\midrule
\multirow{6}{*}{\rotatebox[origin=c]{90}{\textbf{CSDS2}}}  & .99   & .019 $\pm$ .002 & \boldmath{}\textbf{.021 $\pm$ .002}\unboldmath{} & .019 $\pm$ .002 & .019 $\pm$ .002 & .019 $\pm$ .002 & \boldmath{}\textbf{.020 $\pm$ .002}\unboldmath{} & .019 $\pm$ .002 & .019 $\pm$ .002 \\
      & .95   & .019 $\pm$ .002 & \boldmath{}\textbf{.021 $\pm$ .002}\unboldmath{} & .019 $\pm$ .002 & .019 $\pm$ .002 & .019 $\pm$ .002 & \boldmath{}\textbf{.020 $\pm$ .002}\unboldmath{} & .019 $\pm$ .002 & .019 $\pm$ .002 \\
      & .90   & .019 $\pm$ .002 & \boldmath{}\textbf{.021 $\pm$ .002}\unboldmath{} & .019 $\pm$ .002 & .019 $\pm$ .002 & .019 $\pm$ .002 & \boldmath{}\textbf{.020 $\pm$ .002}\unboldmath{} & .019 $\pm$ .002 & .019 $\pm$ .002 \\
      & .85   & .019 $\pm$ .002 & \boldmath{}\textbf{.021 $\pm$ .002}\unboldmath{} & .020 $\pm$ .002 & .019 $\pm$ .002 & .019 $\pm$ .002 & \boldmath{}\textbf{.021 $\pm$ .002}\unboldmath{} & .019 $\pm$ .002 & .018 $\pm$ .002 \\
      & .80   & .019 $\pm$ .002 & \boldmath{}\textbf{.021 $\pm$ .002}\unboldmath{} & \boldmath{}\textbf{.021 $\pm$ .002}\unboldmath{} & .020 $\pm$ .002 & .019 $\pm$ .002 & \boldmath{}\textbf{.021 $\pm$ .002}\unboldmath{} & .019 $\pm$ .002 & .019 $\pm$ .002 \\
      & .75   & .019 $\pm$ .002 & \boldmath{}\textbf{.021 $\pm$ .002}\unboldmath{} & \boldmath{}\textbf{.021 $\pm$ .002}\unboldmath{} & \boldmath{}\textbf{.021 $\pm$ .002}\unboldmath{} & .019 $\pm$ .002 & \boldmath{}\textbf{.022 $\pm$ .002}\unboldmath{} & .019 $\pm$ .002 & .019 $\pm$ .002 \\
\midrule
\midrule
\multirow{6}{*}{\rotatebox[origin=c]{90}{\textbf{CSDS3}}}  & .99   & \boldmath{}\textbf{.254 $\pm$ .003}\unboldmath{} & \boldmath{}\textbf{.254 $\pm$ .003}\unboldmath{} & \boldmath{}\textbf{.254 $\pm$ .003}\unboldmath{} & \boldmath{}\textbf{.254 $\pm$ .003}\unboldmath{} & \boldmath{}\textbf{.254 $\pm$ .003}\unboldmath{} & \boldmath{}\textbf{.254 $\pm$ .003}\unboldmath{} & \boldmath{}\textbf{.254 $\pm$ .003}\unboldmath{} & \boldmath{}\textbf{.254 $\pm$ .003}\unboldmath{} \\
      & .95   & \boldmath{}\textbf{.255 $\pm$ .003}\unboldmath{} & \boldmath{}\textbf{.255 $\pm$ .003}\unboldmath{} & \boldmath{}\textbf{.255 $\pm$ .003}\unboldmath{} & .254 $\pm$ .003 & \boldmath{}\textbf{.256 $\pm$ .003}\unboldmath{} & .255 $\pm$ .003 & \boldmath{}\textbf{.256 $\pm$ .003}\unboldmath{} & .255 $\pm$ .003 \\
      & .90   & \boldmath{}\textbf{.258 $\pm$ .003}\unboldmath{} & .256 $\pm$ .003 & .257 $\pm$ .003 & .255 $\pm$ .003 & \boldmath{}\textbf{.259 $\pm$ .003}\unboldmath{} & .256 $\pm$ .003 & .258 $\pm$ .003 & .254 $\pm$ .003 \\
      & .85   & \boldmath{}\textbf{.261 $\pm$ .003}\unboldmath{} & .257 $\pm$ .004 & .259 $\pm$ .003 & .256 $\pm$ .003 & \boldmath{}\textbf{.263 $\pm$ .003}\unboldmath{} & .257 $\pm$ .004 & .260 $\pm$ .003 & .255 $\pm$ .004 \\
      & .80   & \boldmath{}\textbf{.263 $\pm$ .004}\unboldmath{} & .257 $\pm$ .004 & .258 $\pm$ .004 & .256 $\pm$ .004 & \boldmath{}\textbf{.265 $\pm$ .004}\unboldmath{} & .256 $\pm$ .004 & .261 $\pm$ .004 & .255 $\pm$ .004 \\
      & .75   & \boldmath{}\textbf{.266 $\pm$ .004}\unboldmath{} & .257 $\pm$ .004 & .259 $\pm$ .004 & .255 $\pm$ .004 & \boldmath{}\textbf{.268 $\pm$ .004}\unboldmath{} & .257 $\pm$ .004 & .261 $\pm$ .004 & .256 $\pm$ .004 \\
\midrule
\midrule
      & \#    & 16/42 & 22/42 & 16/42 & 11/42 & 14/42 & 27/42 & 13/42 & 4/42 \\
\end{tabular}%

    }

    \label{tab:class_pos}
\end{table}

\begin{table}[t]
    \centering
        \caption{Performance metrics for \textsc{AUCross} using different number of folds $K$ over tabular datasets (1,000 bootstrap runs over the test set, results as mean $\pm$ stdev).}
    \resizebox{\textwidth}{!}{
\begin{tabular}{c|c|ccccc|ccccc}
\multicolumn{1}{c}{} & \multicolumn{1}{c}{} & \multicolumn{5}{c}{\textbf{Empirical Coverage}} & \multicolumn{5}{c}{\textbf{Selective AUC}} \\
      & \boldmath{}\textbf{$c$}\unboldmath{} & \boldmath{}\textbf{$K=2$}\unboldmath{} & \boldmath{}\textbf{$K=3$}\unboldmath{} & \boldmath{}\textbf{$K=5$}\unboldmath{} & \boldmath{}\textbf{$K=7$}\unboldmath{} & \boldmath{}\textbf{$K=10$}\unboldmath{} & \boldmath{}\textbf{$K=2$}\unboldmath{} & \boldmath{}\textbf{$K=3$}\unboldmath{} & \boldmath{}\textbf{$K=5$}\unboldmath{} & \boldmath{}\textbf{$K=7$}\unboldmath{} & \boldmath{}\textbf{$K=10$}\unboldmath{} \\
\midrule
\multirow{6}{*}{\rotatebox[origin=c]{90}{\textbf{Adult}}}  & .99   & \boldmath{}\textbf{.990 $\pm$ .001}\unboldmath{} & .991 $\pm$ .001 & .989 $\pm$ .001 & .991 $\pm$ .001 & .991 $\pm$ .001 & \boldmath{}\textbf{.929 $\pm$ .003}\unboldmath{} & .928 $\pm$ .003 & \boldmath{}\textbf{.929 $\pm$ .003}\unboldmath{} & .928 $\pm$ .003 & .928 $\pm$ .003 \\
      & .95   & .952 $\pm$ .002 & .951 $\pm$ .002 & \boldmath{}\textbf{.950 $\pm$ .002}\unboldmath{} & .949 $\pm$ .002 & .947 $\pm$ .002 & .935 $\pm$ .003 & .935 $\pm$ .003 & .935 $\pm$ .003 & .935 $\pm$ .003 & \boldmath{}\textbf{.936 $\pm$ .003}\unboldmath{} \\
      & .90   & .896 $\pm$ .003 & .897 $\pm$ .003 & .896 $\pm$ .003 & .896 $\pm$ .003 & \boldmath{}\textbf{.900 $\pm$ .003}\unboldmath{} & \boldmath{}\textbf{.943 $\pm$ .002}\unboldmath{} & \boldmath{}\textbf{.943 $\pm$ .002}\unboldmath{} & \boldmath{}\textbf{.943 $\pm$ .002}\unboldmath{} & \boldmath{}\textbf{.943 $\pm$ .002}\unboldmath{} & \boldmath{}\textbf{.943 $\pm$ .002}\unboldmath{} \\
      & .85   & .848 $\pm$ .003 & \boldmath{}\textbf{.850 $\pm$ .003}\unboldmath{} & .849 $\pm$ .003 & .848 $\pm$ .003 & .849 $\pm$ .003 & \boldmath{}\textbf{.950 $\pm$ .002}\unboldmath{} & \boldmath{}\textbf{.950 $\pm$ .002}\unboldmath{} & \boldmath{}\textbf{.950 $\pm$ .002}\unboldmath{} & \boldmath{}\textbf{.950 $\pm$ .002}\unboldmath{} & \boldmath{}\textbf{.950 $\pm$ .002}\unboldmath{} \\
      & .80   & .796 $\pm$ .004 & .797 $\pm$ .004 & .797 $\pm$ .004 & .799 $\pm$ .004 & \boldmath{}\textbf{.800 $\pm$ .004}\unboldmath{} & \boldmath{}\textbf{.958 $\pm$ .002}\unboldmath{} & \boldmath{}\textbf{.958 $\pm$ .002}\unboldmath{} & \boldmath{}\textbf{.958 $\pm$ .002}\unboldmath{} & \boldmath{}\textbf{.957 $\pm$ .002}\unboldmath{} & \boldmath{}\textbf{.957 $\pm$ .002}\unboldmath{} \\
      & .75   & \boldmath{}\textbf{.750 $\pm$ .004}\unboldmath{} & .752 $\pm$ .004 & \boldmath{}\textbf{.750 $\pm$ .004}\unboldmath{} & .753 $\pm$ .004 & .751 $\pm$ .004 & .963 $\pm$ .002 & .963 $\pm$ .002 & .963 $\pm$ .002 & \boldmath{}\textbf{.964 $\pm$ .002}\unboldmath{} & \boldmath{}\textbf{.964 $\pm$ .002}\unboldmath{} \\
\midrule
\midrule
\multirow{6}{*}{\rotatebox[origin=c]{90}{\textbf{Lending}}}  & .99   & \boldmath{}\textbf{.996 $\pm$ .001}\unboldmath{} & \boldmath{}\textbf{.996 $\pm$ .001}\unboldmath{} & \boldmath{}\textbf{.996 $\pm$ .001}\unboldmath{} & \boldmath{}\textbf{.996 $\pm$ .001}\unboldmath{} & \boldmath{}\textbf{.996 $\pm$ .001}\unboldmath{} & .983 $\pm$ .001 & .983 $\pm$ .001 & .983 $\pm$ .001 & .983 $\pm$ .001 & .983 $\pm$ .001 \\
      & .95   & \boldmath{}\textbf{.980 $\pm$ .001}\unboldmath{} & \boldmath{}\textbf{.980 $\pm$ .001}\unboldmath{} & \boldmath{}\textbf{.980 $\pm$ .001}\unboldmath{} & \boldmath{}\textbf{.980 $\pm$ .001}\unboldmath{} & \boldmath{}\textbf{.980 $\pm$ .001}\unboldmath{} & .985 $\pm$ .001 & .985 $\pm$ .001 & .985 $\pm$ .001 & .985 $\pm$ .001 & .985 $\pm$ .001 \\
      & .90   & .958 $\pm$ .001 & \boldmath{}\textbf{.959 $\pm$ .001}\unboldmath{} & \boldmath{}\textbf{.959 $\pm$ .001}\unboldmath{} & \boldmath{}\textbf{.959 $\pm$ .001}\unboldmath{} & \boldmath{}\textbf{.959 $\pm$ .001}\unboldmath{} & .987 $\pm$ .001 & .987 $\pm$ .001 & .987 $\pm$ .001 & .987 $\pm$ .001 & .987 $\pm$ .001 \\
      & .85   & .935 $\pm$ .001 & \boldmath{}\textbf{.936 $\pm$ .001}\unboldmath{} & \boldmath{}\textbf{.936 $\pm$ .001}\unboldmath{} & \boldmath{}\textbf{.936 $\pm$ .001}\unboldmath{} & \boldmath{}\textbf{.936 $\pm$ .001}\unboldmath{} & .989 $\pm$ .001 & .989 $\pm$ .001 & .989 $\pm$ .001 & .989 $\pm$ .001 & .989 $\pm$ .001 \\
      & .80   & \boldmath{}\textbf{.912 $\pm$ .001}\unboldmath{} & \boldmath{}\textbf{.912 $\pm$ .001}\unboldmath{} & \boldmath{}\textbf{.912 $\pm$ .001}\unboldmath{} & \boldmath{}\textbf{.912 $\pm$ .001}\unboldmath{} & \boldmath{}\textbf{.912 $\pm$ .001}\unboldmath{} & .990 $\pm$ .001 & .990 $\pm$ .001 & .990 $\pm$ .001 & .990 $\pm$ .001 & .990 $\pm$ .001 \\
      & .75   & \boldmath{}\textbf{.886 $\pm$ .001}\unboldmath{} & \boldmath{}\textbf{.886 $\pm$ .001}\unboldmath{} & \boldmath{}\textbf{.886 $\pm$ .001}\unboldmath{} & \boldmath{}\textbf{.886 $\pm$ .001}\unboldmath{} & \boldmath{}\textbf{.886 $\pm$ .001}\unboldmath{} & .992 $\pm$ .001 & .992 $\pm$ .001 & .992 $\pm$ .001 & .992 $\pm$ .001 & .992 $\pm$ .001 \\
\midrule
\midrule
\multirow{6}{*}{\rotatebox[origin=c]{90}{\textbf{GiveMe}}}  & .99   & \boldmath{}\textbf{.990 $\pm$ .001}\unboldmath{} & \boldmath{}\textbf{.990 $\pm$ .001}\unboldmath{} & \boldmath{}\textbf{.990 $\pm$ .001}\unboldmath{} & \boldmath{}\textbf{.990 $\pm$ .001}\unboldmath{} & .991 $\pm$ .001 & \boldmath{}\textbf{.868 $\pm$ .004}\unboldmath{} & \boldmath{}\textbf{.868 $\pm$ .004}\unboldmath{} & \boldmath{}\textbf{.868 $\pm$ .004}\unboldmath{} & \boldmath{}\textbf{.868 $\pm$ .004}\unboldmath{} & .867 $\pm$ .004 \\
      & .95   & .951 $\pm$ .002 & \boldmath{}\textbf{.950 $\pm$ .002}\unboldmath{} & \boldmath{}\textbf{.950 $\pm$ .002}\unboldmath{} & \boldmath{}\textbf{.950 $\pm$ .002}\unboldmath{} & \boldmath{}\textbf{.950 $\pm$ .002}\unboldmath{} & \boldmath{}\textbf{.874 $\pm$ .004}\unboldmath{} & \boldmath{}\textbf{.874 $\pm$ .004}\unboldmath{} & \boldmath{}\textbf{.874 $\pm$ .004}\unboldmath{} & \boldmath{}\textbf{.874 $\pm$ .004}\unboldmath{} & \boldmath{}\textbf{.874 $\pm$ .004}\unboldmath{} \\
      & .90   & \boldmath{}\textbf{.900 $\pm$ .002}\unboldmath{} & .899 $\pm$ .002 & .898 $\pm$ .002 & .899 $\pm$ .002 & .898 $\pm$ .002 & .882 $\pm$ .004 & .882 $\pm$ .004 & \boldmath{}\textbf{.883 $\pm$ .004}\unboldmath{} & \boldmath{}\textbf{.883 $\pm$ .004}\unboldmath{} & \boldmath{}\textbf{.883 $\pm$ .004}\unboldmath{} \\
      & .85   & \boldmath{}\textbf{.850 $\pm$ .002}\unboldmath{} & .849 $\pm$ .002 & .848 $\pm$ .002 & .849 $\pm$ .002 & .849 $\pm$ .002 & \boldmath{}\textbf{.890 $\pm$ .004}\unboldmath{} & \boldmath{}\textbf{.890 $\pm$ .004}\unboldmath{} & \boldmath{}\textbf{.890 $\pm$ .004}\unboldmath{} & \boldmath{}\textbf{.890 $\pm$ .004}\unboldmath{} & \boldmath{}\textbf{.890 $\pm$ .004}\unboldmath{} \\
      & .80   & .794 $\pm$ .003 & .795 $\pm$ .003 & .795 $\pm$ .003 & \boldmath{}\textbf{.796 $\pm$ .003}\unboldmath{} & \boldmath{}\textbf{.796 $\pm$ .003}\unboldmath{} & \boldmath{}\textbf{.898 $\pm$ .004}\unboldmath{} & .897 $\pm$ .004 & \boldmath{}\textbf{.898 $\pm$ .004}\unboldmath{} & \boldmath{}\textbf{.898 $\pm$ .004}\unboldmath{} & .897 $\pm$ .004 \\
      & .75   & .747 $\pm$ .003 & .749 $\pm$ .003 & .748 $\pm$ .003 & .749 $\pm$ .003 & \boldmath{}\textbf{.750 $\pm$ .003}\unboldmath{} & \boldmath{}\textbf{.904 $\pm$ .004}\unboldmath{} & .903 $\pm$ .004 & \boldmath{}\textbf{.904 $\pm$ .004}\unboldmath{} & \boldmath{}\textbf{.904 $\pm$ .004}\unboldmath{} & .903 $\pm$ .004 \\
\midrule
\midrule
\multirow{6}{*}{\rotatebox[origin=c]{90}{\textbf{UCICredit}}}  & .99   & .989 $\pm$ .002 & .989 $\pm$ .002 & .988 $\pm$ .002 & .988 $\pm$ .002 & \boldmath{}\textbf{.990 $\pm$ .002}\unboldmath{} & .772 $\pm$ .008 & .772 $\pm$ .007 & .772 $\pm$ .008 & .772 $\pm$ .008 & .772 $\pm$ .008 \\
      & .95   & .944 $\pm$ .003 & \boldmath{}\textbf{.946 $\pm$ .003}\unboldmath{} & .944 $\pm$ .003 & .945 $\pm$ .003 & \boldmath{}\textbf{.946 $\pm$ .003}\unboldmath{} & .778 $\pm$ .008 & .778 $\pm$ .008 & .778 $\pm$ .008 & .778 $\pm$ .008 & .778 $\pm$ .008 \\
      & .90   & .886 $\pm$ .004 & .894 $\pm$ .004 & .896 $\pm$ .004 & .895 $\pm$ .004 & \boldmath{}\textbf{.900 $\pm$ .004}\unboldmath{} & \boldmath{}\textbf{.786 $\pm$ .008}\unboldmath{} & .785 $\pm$ .008 & .785 $\pm$ .008 & .784 $\pm$ .008 & .784 $\pm$ .008 \\
      & .85   & .830 $\pm$ .005 & .838 $\pm$ .005 & .840 $\pm$ .005 & .840 $\pm$ .005 & \boldmath{}\textbf{.843 $\pm$ .005}\unboldmath{} & \boldmath{}\textbf{.794 $\pm$ .008}\unboldmath{} & .793 $\pm$ .007 & .792 $\pm$ .008 & .792 $\pm$ .008 & .792 $\pm$ .008 \\
      & .80   & .775 $\pm$ .005 & .789 $\pm$ .005 & .789 $\pm$ .005 & .788 $\pm$ .005 & \boldmath{}\textbf{.795 $\pm$ .005}\unboldmath{} & \boldmath{}\textbf{.802 $\pm$ .008}\unboldmath{} & .800 $\pm$ .008 & .800 $\pm$ .008 & .800 $\pm$ .008 & .800 $\pm$ .008 \\
      & .75   & .720 $\pm$ .006 & .737 $\pm$ .006 & .737 $\pm$ .006 & .735 $\pm$ .006 & \boldmath{}\textbf{.745 $\pm$ .006}\unboldmath{} & \boldmath{}\textbf{.811 $\pm$ .008}\unboldmath{} & .808 $\pm$ .008 & .809 $\pm$ .008 & .809 $\pm$ .008 & .807 $\pm$ .008 \\
\midrule
\midrule
\multirow{6}{*}{\rotatebox[origin=c]{90}{\textbf{CSDS1}}}  & .99   & .991 $\pm$ .001 & \boldmath{}\textbf{.990 $\pm$ .001}\unboldmath{} & \boldmath{}\textbf{.990 $\pm$ .001}\unboldmath{} & \boldmath{}\textbf{.990 $\pm$ .001}\unboldmath{} & \boldmath{}\textbf{.990 $\pm$ .001}\unboldmath{} & .686 $\pm$ .003 & .686 $\pm$ .003 & .686 $\pm$ .003 & .686 $\pm$ .003 & .686 $\pm$ .003 \\
      & .95   & \boldmath{}\textbf{.950 $\pm$ .001}\unboldmath{} & \boldmath{}\textbf{.950 $\pm$ .001}\unboldmath{} & \boldmath{}\textbf{.950 $\pm$ .001}\unboldmath{} & .949 $\pm$ .001 & \boldmath{}\textbf{.950 $\pm$ .001}\unboldmath{} & .689 $\pm$ .003 & .689 $\pm$ .003 & .689 $\pm$ .003 & .689 $\pm$ .003 & .689 $\pm$ .003 \\
      & .90   & .898 $\pm$ .002 & \boldmath{}\textbf{.899 $\pm$ .002}\unboldmath{} & .898 $\pm$ .002 & .898 $\pm$ .002 & .898 $\pm$ .002 & .693 $\pm$ .003 & .693 $\pm$ .003 & .693 $\pm$ .003 & .693 $\pm$ .003 & .693 $\pm$ .003 \\
      & .85   & .848 $\pm$ .002 & \boldmath{}\textbf{.849 $\pm$ .002}\unboldmath{} & \boldmath{}\textbf{.849 $\pm$ .002}\unboldmath{} & .848 $\pm$ .002 & .848 $\pm$ .002 & .697 $\pm$ .003 & .697 $\pm$ .003 & .697 $\pm$ .003 & .697 $\pm$ .003 & .697 $\pm$ .003 \\
      & .80   & .796 $\pm$ .002 & \boldmath{}\textbf{.799 $\pm$ .002}\unboldmath{} & \boldmath{}\textbf{.799 $\pm$ .002}\unboldmath{} & .797 $\pm$ .002 & .797 $\pm$ .002 & .702 $\pm$ .004 & .702 $\pm$ .004 & .702 $\pm$ .004 & .702 $\pm$ .004 & .702 $\pm$ .004 \\
      & .75   & .746 $\pm$ .002 & \boldmath{}\textbf{.748 $\pm$ .002}\unboldmath{} & \boldmath{}\textbf{.748 $\pm$ .002}\unboldmath{} & .747 $\pm$ .002 & .747 $\pm$ .002 & .706 $\pm$ .004 & .706 $\pm$ .004 & .706 $\pm$ .004 & .706 $\pm$ .004 & .706 $\pm$ .004 \\
\midrule
\midrule
\multirow{6}{*}{\rotatebox[origin=c]{90}{\textbf{CSDS2}}}  & .99   & \boldmath{}\textbf{.990 $\pm$ .001}\unboldmath{} & .992 $\pm$ .001 & \boldmath{}\textbf{.990 $\pm$ .001}\unboldmath{} & \boldmath{}\textbf{.990 $\pm$ .001}\unboldmath{} & .992 $\pm$ .001 & .615 $\pm$ .020 & \boldmath{}\textbf{.618 $\pm$ .020}\unboldmath{} & .616 $\pm$ .020 & \boldmath{}\textbf{.618 $\pm$ .020}\unboldmath{} & .617 $\pm$ .020 \\
      & .95   & .952 $\pm$ .002 & \boldmath{}\textbf{.951 $\pm$ .002}\unboldmath{} & .952 $\pm$ .002 & .956 $\pm$ .002 & .956 $\pm$ .002 & \boldmath{}\textbf{.622 $\pm$ .020}\unboldmath{} & .619 $\pm$ .020 & .619 $\pm$ .020 & .619 $\pm$ .020 & .621 $\pm$ .020 \\
      & .90   & \boldmath{}\textbf{.904 $\pm$ .003}\unboldmath{} & \boldmath{}\textbf{.904 $\pm$ .003}\unboldmath{} & \boldmath{}\textbf{.904 $\pm$ .003}\unboldmath{} & .909 $\pm$ .003 & .911 $\pm$ .003 & .620 $\pm$ .020 & \boldmath{}\textbf{.626 $\pm$ .020}\unboldmath{} & .620 $\pm$ .020 & .625 $\pm$ .020 & .624 $\pm$ .020 \\
      & .85   & \boldmath{}\textbf{.854 $\pm$ .004}\unboldmath{} & \boldmath{}\textbf{.854 $\pm$ .004}\unboldmath{} & .860 $\pm$ .004 & .861 $\pm$ .003 & .866 $\pm$ .003 & .622 $\pm$ .020 & .629 $\pm$ .021 & \boldmath{}\textbf{.631 $\pm$ .021}\unboldmath{} & .626 $\pm$ .021 & .626 $\pm$ .021 \\
      & .80   & .808 $\pm$ .004 & \boldmath{}\textbf{.804 $\pm$ .004}\unboldmath{} & .811 $\pm$ .004 & .814 $\pm$ .004 & .819 $\pm$ .004 & .628 $\pm$ .021 & \boldmath{}\textbf{.636 $\pm$ .021}\unboldmath{} & .631 $\pm$ .021 & .632 $\pm$ .021 & .630 $\pm$ .021 \\
      & .75   & .760 $\pm$ .004 & \boldmath{}\textbf{.754 $\pm$ .004}\unboldmath{} & .760 $\pm$ .004 & .765 $\pm$ .004 & .770 $\pm$ .004 & .628 $\pm$ .021 & \boldmath{}\textbf{.637 $\pm$ .021}\unboldmath{} & .635 $\pm$ .021 & .634 $\pm$ .021 & .630 $\pm$ .021 \\
\midrule
\midrule
\multirow{6}{*}{\rotatebox[origin=c]{90}{\textbf{CSDS3}}}  & .99   & \boldmath{}\textbf{.991 $\pm$ .001}\unboldmath{} & .992 $\pm$ .001 & \boldmath{}\textbf{.991 $\pm$ .001}\unboldmath{} & \boldmath{}\textbf{.991 $\pm$ .001}\unboldmath{} & \boldmath{}\textbf{.991 $\pm$ .001}\unboldmath{} & .851 $\pm$ .003 & .851 $\pm$ .003 & .851 $\pm$ .003 & .851 $\pm$ .003 & .851 $\pm$ .003 \\
      & .95   & .948 $\pm$ .002 & \boldmath{}\textbf{.950 $\pm$ .002}\unboldmath{} & .948 $\pm$ .002 & \boldmath{}\textbf{.950 $\pm$ .002}\unboldmath{} & .951 $\pm$ .002 & \boldmath{}\textbf{.858 $\pm$ .003}\unboldmath{} & .857 $\pm$ .003 & \boldmath{}\textbf{.858 $\pm$ .003}\unboldmath{} & .857 $\pm$ .003 & .857 $\pm$ .003 \\
      & .90   & .902 $\pm$ .002 & \boldmath{}\textbf{.901 $\pm$ .002}\unboldmath{} & \boldmath{}\textbf{.901 $\pm$ .002}\unboldmath{} & .902 $\pm$ .002 & \boldmath{}\textbf{.901 $\pm$ .002}\unboldmath{} & .865 $\pm$ .003 & .865 $\pm$ .003 & .865 $\pm$ .003 & .865 $\pm$ .003 & .865 $\pm$ .003 \\
      & .85   & \boldmath{}\textbf{.850 $\pm$ .003}\unboldmath{} & .853 $\pm$ .003 & .849 $\pm$ .003 & \boldmath{}\textbf{.850 $\pm$ .003}\unboldmath{} & \boldmath{}\textbf{.850 $\pm$ .003}\unboldmath{} & .873 $\pm$ .003 & .873 $\pm$ .003 & .873 $\pm$ .003 & .873 $\pm$ .003 & .873 $\pm$ .003 \\
      & .80   & .809 $\pm$ .003 & .810 $\pm$ .003 & \boldmath{}\textbf{.805 $\pm$ .003}\unboldmath{} & \boldmath{}\textbf{.805 $\pm$ .003}\unboldmath{} & \boldmath{}\textbf{.805 $\pm$ .003}\unboldmath{} & .879 $\pm$ .003 & .879 $\pm$ .003 & \boldmath{}\textbf{.880 $\pm$ .003}\unboldmath{} & \boldmath{}\textbf{.880 $\pm$ .003}\unboldmath{} & \boldmath{}\textbf{.880 $\pm$ .003}\unboldmath{} \\
      & .75   & .760 $\pm$ .003 & .759 $\pm$ .003 & .759 $\pm$ .003 & \boldmath{}\textbf{.758 $\pm$ .003}\unboldmath{} & \boldmath{}\textbf{.758 $\pm$ .003}\unboldmath{} & .886 $\pm$ .003 & \boldmath{}\textbf{.887 $\pm$ .003}\unboldmath{} & \boldmath{}\textbf{.887 $\pm$ .003}\unboldmath{} & \boldmath{}\textbf{.887 $\pm$ .003}\unboldmath{} & \boldmath{}\textbf{.887 $\pm$ .003}\unboldmath{} \\
\midrule
\midrule
      & \#    & 15/42 & 23/42 & {20/42} & {16/42} & 24/42 & 29/42 & 28/42 & 31/42 & 30/42 & 27/42 \\
\midrule
      & $V$   & .019  $\pm$ .038 & .017  $\pm$ .038 & .018  $\pm$ .038 & .018  $\pm$ .038 & .018  $\pm$ .038 &       &       &       &       &  \\
\end{tabular}%

    }

    \label{tab:K_cross}
\end{table}


\end{document}